\documentclass[12pt]{article}

\usepackage{amssymb,amsfonts,amsmath}
\usepackage{amsthm}
\usepackage{mathrsfs}

\usepackage{subcaption}
\usepackage{graphicx}
\usepackage{multirow}

\usepackage{overpic}
\usepackage{upgreek}

\usepackage{geometry}
\usepackage{xcolor}
\usepackage{multirow}
\usepackage{palatino}
\usepackage{tikz}
\usetikzlibrary{arrows,shapes,chains} 
\usepackage{authblk}
\usepackage{natbib}
\usepackage{setspace}

\usepackage{algorithm}  
\usepackage{algorithmicx}

\usepackage[noend]{algpseudocode}

\usepackage{cases}

\newtheorem{theorem}{Theorem}

\newtheorem{lemma}{Lemma}

\newtheorem{corollary}{Corollary}

\title{Online Resource Allocation: Bandits feedback and Advice on Time-varying Demands}

\author[1]{Lixing Lyu}
\author[2]{Wang Chi Cheung}
\affil[1]{Institute of Operations Research and Analytics, National University of Singapore, \textit{lyulixing@u.nus.edu}}
\affil[2]{Department of Industrial Systems Engineering and Management, National University of Singapore, \textit{isecwc@nus.edu.sg}}

\allowdisplaybreaks[3]

\begin{document}

\bibliographystyle{plainnat}
	
\maketitle
\newcommand*\abs[1]{\lvert#1\rvert}

\begin{abstract}
We consider a general online resource allocation model with bandit feedback and time-varying demands. While online resource allocation has been well studied in the literature, most existing works make the strong assumption that the demand arrival process is stationary. In practical applications, such as online advertisement and revenue management, however, this process may be exogenous and non-stationary, like the constantly changing internet traffic. Motivated by the recent Online Algorithms with Advice framework [Mitazenmacher and Vassilvitskii, \emph{Commun. ACM} 2022], we explore how online advice can inform policy design. We establish an impossibility result that any algorithm perform poorly in terms of regret without any advice in our setting. In contrast, we design an robust online algorithm that leverages the online predictions on the total demand volumes. Empowered with online advice, our proposed algorithm is shown to have both theoretical performance and promising numerical results compared with other algorithms in literature. We also provide two explicit examples for the time-varying demand scenarios and derive corresponding theoretical performance guarantees. Finally, we adapt our model to a network revenue management problem, and numerically demonstrate that our algorithm can still performs competitively compared to existing baselines.

\end{abstract}

\section{Introduction}

Online resource allocation is a fundamental framework in Operations Research, encompassing models where a decision maker (DM) endowed with resource constraint serves sequentially comming customers, each of which consumes resources and generates a reward. At each time step, the DM takes an action, aiming to maximize cumulative reward subject to resource constraints. This area has been studied extensively over many decades, with several applications such as network revenue management (\cite{talluri1998analysis}, \cite{liu2008choice}, \cite{zhang2009approximate}) and online advertising (\cite{mehta2007adwords,mehta2013online}). 

While there is a substantial body of literature in this area, most existing works rely on the assumption that the demand arrival process is stationary (\cite{besbes2012blind,ferreira2018online}). 
In practice, however, this assumption could be too strong. Consider the example of online advertising. The internet traffic volume is constantly changing, leading to a non-stationary model. While the problem model exhibits non-stationarity, the model is structured in the sense that 
the non-stationariety in internet traffic is often exogenous, while the attractiveness of an online advertisement can be considered stationary.
How could we effectively harness this problem structure to design policies?


Motivated by the aforementioned considerations, we address an online resource allocation problem in which the outcome, represented an array of reward and resource consumptions, is scaled by a non-stationary quantity that signifies changing demand volumes.
Specifically, the DM is endowed with a set of available actions and an array of resource budgets to be allocated in a fixed and finite horizon $T$.
The outcome of each action at time $t \in [T]$ involves an adversarial and a stochastic component. 
The adversarial part, $q_t$, represents the non-stationary demand volume, and the stochastic term refers to the outcome per unit demand, which is sampled from an unknown distribution, characterizing the intrinsic quality of each action. The DM only receives information about the adversarial $q_t$ and the latent outcome distribution via the bandit feedback after each allocation.
The DM needs to choose an action at each time step with the objective of maximizing the total cumulative rewards and subject to the fixed resource budget constraint. Replenishment is not allowed during the planning horizon.
Given the unknown distribution of outcome and bandit feedback setting, the DM faces a common \emph{exploration-exploitation} tradeoff, and the limited resource increases the urgency of \emph{exploration}. Moreover, the non-stationarity of demands brings about additional uncertainty in the outcome of each allocation decision each round and the future total demand volume, further complicating the online optimization challenge.

The uncertainties arising from both latent outcome distribution and non-stationary time-varying demand pose significant challenges in designing robust online policies that perform well. \cite{badanidiyuru2013bandits} is the most relevant work to ours, except for the stationary demand, refered to as \textit{Stochastic Bandits with Knapsacks} (Stochastic BwK) problem. This model is a special case of ours by setting $q_t = 1$ for all $t \in [T]$. \cite{badanidiyuru2013bandits,agrawal2014bandits} provide optimal online algorithms for stochastic BwK with regret sub-linear in $T$ when the resource budget is in the same level of $T$. The regret of an algorithm is the difference between the optimum and the expected cumulative reward earned by the algorithm, and a sub-linear-in-$T$ regret implies the convergence to optimality as $T$ grows. An alternative setting, \textit{adversarial BwK}, is introduced by \cite{immorlica2019adversarial}. Each action's outcome distribution can change arbitrarily over time. Contrary to the stochastic BwK, it is impossible to achieve a regret sub-linear in $T$, even when the outcome distribution is changed only once during the horizon, as shown in \citep{liu2022non}. Although our model is less adversarial compared to the adversarial BwK, as the outcome distribution changes due to non-stationary demand, our setting still presents significant challenges. Indeed, it is still impossible to achieve a sub-linear-in-$T$ regret in our setting, if the DM does not receive any online advice, as shown in our forthcoming Section \ref{sec:lower_bound}.

Motivated by the Online Algorithm with Advice framework \citep{mitzenmacher2022algorithms}, we are interested in incorporating predictions on non-stationary demand to enhance the tractability of our problem. These predictions as a form of advice can be used to model consulting services, which is broadly used in applications in supply chain management, revenue management and risk management. For example, in the context of revenue management, consultants may provide forecast about market size trend, which can help firms or government organizations to optimize their pricing strategies. Interesting readers can refer to \cite{ribeiro2001quality,mughan2004management,makovetskiy2021consulting}. Intuitively, the performance of a policy that leverages the advice crucially depends on the accuracy of the advice. Our paper aims to develop an algorithm with theoretical performance guarantees, by robustly incorporating an online advice irrespective to its quality. To achieve this, we assume that the DM only has access to the advice, but not its accuracy of. In other words, our algorithm takes the advice as a black box, and performs well when the advice is informative. By leveraging the growing capabilities from machine learning methods, our algorithm can harness the increasing accuracy of predictions.

\subsection{Main Contributions}

Here we briefly summarize our contributions.

\textbf{Model. } We propose an online resource allocation model with bandit feedback and time-varying demand, incorporating a prediction oracle. Such an incorporation is novel compared to existing literature. We identify the value of total demand volume $Q = \sum_{t=1}^T q_t$ as a crucial (but latent) parameter that affects the online decision. The prediction oracle provides a prediction to this value at each time step. The oracle corresponds to how the DM (who can be a firm or the government) constantly acquires updated information about the total demand 
through some consulting service. In practice, various existing time series or machine learning tools can be used to construct such prediction oracle. 

\textbf{Impossibility Results. } We establish two impossibility results in terms of regret for our proposed model: Firstly, without any advice, we prove that any online algorithm performs poorly, suffering a linear-in-$T$ regret. This result highlights the challenging nature of the problem when no advice is available. Second, with the access of a prediction oracle, we establish the regret lower bound that depends on the accuracy of the prediction. When the predictions align perfectly with the groundtruth, the regret lower bound will reduce to that of the well-studied stochastic BwK problem. This result demonstrates that the quality of the prediction can significantly impact the achievable regret in our model, indicating the importance of accurate demand forecasting.

\textbf{Algorithms and Analysis. } We design a robust online algorithm, OA-UCB, that utilizes the predictions judiciously. OA-UCB is innovative in its incorporation of the prediction and demand volume into the estimation of opportunity costs of the resources, in relation to the predicted demand volumes. We derive a regret upper bound on OA-UCB that depends on the accuracy of the predictions, even though the algorithm does not have knowledge about the accuracy of each prediction. OA-UCB is shown to achieve near optimal regret bounds when the accuracy of the predictions improves across time. This highlights the effectiveness of OA-UCB in leveraging the predictions to make informed online decisions and achieve competitive regret performance.

\textbf{Two Specific Examples. } We propose two specific examples of the time-varying demand: Linearly increasing model and AR(1) model. For each example, we propose a prediction oracle and derive explicit regret upper bound, respectively. Both of these bounds are shown to be near optimal.
 
\textbf{Numerical Validations. } We propose two parts of numerical experiment to validate the advancement of our proposed OA-UCB algorithm. First, we compare the performance our OA-UCB with other existing algorithms in our setting, demonstrating that the OA-UCB can outperform others with the help of online advice.
Moreover, we investigate the impact of different prediction oracles by comparing the performance of OA-UCB with various prediction methods. This analysis shows the benefit of receiving accurate predictions in our model.
Second, we adapt our model to a network revenue management problem and develop a variant of our algorithm, dubbed as OA-UCB-DP. We perform numerical experiments to evaluate the performance of OA-UCB-DP in comparison to other dynamic pricing algorithms in the literature. In this experiment, our algorithm can still outperform others.


\subsection{Literature Review}

Our work relates to the following areas of literature:

\textbf{Bandits with Knapsacks (BwK): } The Bandits with Knapsacks (BwK) problem has been extensively studied. \cite{badanidiyuru2013bandits} first introduced the stochastic BwK problem, which bears applications in dynamic pricing (\cite{besbes2009dynamic,besbes2012blind}) and ad allocation (\cite{mehta2007adwords}). The BwK problem is generalized by  \cite{agrawal2014bandits} to incorporate convex constraints and concave reweards.
Several variants are studied, such as the settings of contextual bandits (\cite{agrawal2016efficient,badanidiyuru2014resourceful}), combinatorial semi-bandits (\cite{sankararaman2018combinatorial}). 

Non-stationary BwK problems, where the outcome distribution of each arm is changing over time, are studied recently. \cite{immorlica2019adversarial} achieves a $O(\log T)$ competitive ratio against the best fixed distribution benchmark in an adversarial setting. \cite{rangi2018unifying} consider both stochastic and adversarial BwK problems in the single resource case. \cite{liu2022non} design a sliding window learning algorithm with sub-linear-in-$T$ regret, assuming the amount of non-stationarity is upper bounded and known. A sub-linear-in-$T$ regret on non-stationary BwK is only possible in restrictive settings. For example, as shown in  \cite{immorlica2019adversarial,liu2022non} and our forthcoming Lemma \ref{lem-lowerbound-oblivious}, for any online algorithm, there exists a non-stationary BwK instance where the outcome distribution only changes once during the horizon, for which the algorithm incurs a linear-in-$T$ regret. 

\textbf{Non-stationary Bandits: } Non-stationary stochastic bandits with no resource constraints are studied in \citep{besbes2014stochastic,cheung2019learning,zhu2020demands}, who provide sub-linear-in-$T$ regret bounds in less restrictive non-stationary settings than \cite{liu2022non}, while the amount of non-stationariety, quantified as the variational budget or the number of change points, has to be sub-linear in $T$. Our work goes in an orthogonal direction. Instead of studying settings with limited non-stationariety, we seek an improved regret bound when the decision maker is endowed with information additional (in the form of prediction oracle) to the online observations.

\textbf{Online Resource Allocation with horizon uncertainty: } 
Our work is also related to a recent stream of work on resource allocation with horizon uncertainty. \cite{BaiEJRTW22,AouadMa22} consider a stochastic resource allcation setting under full model certainty. In their model, the total demand volume is a random variable, whose probability distribution is known to the decision maker, but the \textit{realization} of the total demand volume is not known. \cite{BalseiroKR22} consider an online resource allocation setting with model uncertainty on the horizon, which is closely related to our model uncertainty setting on the total demand volume $Q$. A cruciall difference between the uncertainty settings in \cite{BalseiroKR22} and ours is that, the former focuses on the case when the DM is provided with \emph{static} advices, while our work complementarily consider the case of \emph{dynamic} advices. More precisely, \cite{BalseiroKR22} consider a model where $q_t\in \{0, 1\}$. They first consider a model when the DM knows $Q\in [Q_{lower}, Q_{upper}]$ but not the actual value of $Q$ at the beginning, and then consider a case when the DM is additionally endowed with a static prediction $\hat{Q}$ at the beginning. In both cases, the performance guarantees are quantified as \emph{competitive ratios} that depends on the static advices. By contrast, our study quantifies the benefit of receiving dyanmically updated advices $\hat{Q}_t$, and pinpoint conditions on $\{\hat{Q}_t\}^T_{t=1}$
that leads to a sublinear-in-$T$ regret. 

\textbf{Online algorithm with machine learning advice: } Our work is related to an active stream of research works on online algorithm design with machine learned advice (\cite{Workshop19,MitzenmancherV22}). While traditional online algorithm research focuses on worst case performance guarantee in full model uncertainty setting, this stream of works focuses on enhancing the performance guarantee when the decision maker (DM) is provided with a machine learned advice at the start of the online dynamics. A variety of results are derived in different settings, such as online caching (\cite{lykouris2021competitive}), rent-or-buy (\cite{PurohitSK18}), scheduling (\cite{Mitzenmacher19,LattanziLMV20}), online set cover problems (\cite{BamasAS20,AlmanzaCLPR21}), online matching (\cite{AntoniadisGKK20}). Our research seeks to take a further step, by investigating the case when the DM receives pregressively updated predictions across the horizon, instead of being given a fixed prediction at the beginning. 

\textbf{Online optimization with predictions: } Lastly, our prediction model is also related to a line of works online optimization with predictions, which concerns improving the performance guarantee with the help of predictions. These predictions are provided to the DM at the beginning of each round sequentially. A variety of full feedback settings are studied in \cite{ rakhlin2013online,rakhlin2013optimization,steinhardt2014adaptivity,jadbabaie2015online}, and the contextual bandit setting is studied in \cite{wei2020taking}. We remark that the abovementioned works do not involve resource constraints, and they are fundamental different from ours, as shown in the forthcoming Lemma \ref{lem-lowerbound-oblivious}.

\subsection{Notation}

For a positive integer $d$, we denote $[d] = \{1, \ldots, d\}$; We denote $\Delta_d = \{\boldsymbol{x} \in [0,1]^d: \sum_{i=1}^d x_i = 1\}$ as the probabilistic simplex with dimension $d$. We adopt the $O(\cdot), o(\cdot), \Omega(\cdot)$ notation in \cite{CormenLRS09}.

\section{Model}

Now we describe our model. Our model instance can be represented by the tuple
\begin{equation*}
    ({\cal A}, d, B, T, \{q_t\}^T_{t=1}, \{P_{a,t}\}_{a\in \mathcal{A}, t \in [T]}).
\end{equation*}
We denote ${\cal A}$ as the set of $K$ actions. There are $d$ types of resources, and the DM is endowed with $B_i  \ge 0 $ units of resource $i$ for each $i\in [d]$. The planning horizon consists of $T$ discrete time steps. Following the convention in \citep{badanidiyuru2013bandits}, we assume for all $i\in [d]$ that $B_i = B = b T$, where $b$ is the normalized budget. At time $t$, there are $q_t$ units of demands arriving at the DM's platform. For example, $q_t$ can be the number of customers visiting an online shop at time step $t$, and a time step can be a fifteen minute interval. We assume $q_t\in [\underline{q},\overline{q}]$, where $0 < \underline{q} < \overline{q}$, and is unknown to the DM. The sequence $\{q_t\}^T_{t=1}$ is an arbitrary element of $[\underline{q},\overline{q}]^T$ fixed by the nature before the online dynamics. The arbitrariness represents the exogenous nature of the demands.

When the DM chooses action $a \in \mathcal{A}$ at time $t$, s/he receives a vector $(R_t;C_{t,1},\ldots,C_{t,d}) \in [0,1]^{d+1}$ of random outcomes, sampled independently from latent distribution $P_{a,t}$. The quantity $R_t$ is the reward earned per demand unit, and $C_{t,i}$ is the amount of type $i$ resource consumed per demand unit. The latent distribution $P_{a,t}$ for arm $a$ is varies over time, sharing a common support of $[0,1]^{d+1}$ and mean. We denote $(r(a);c(a,1),\ldots,c(a,d)) = \mathbb{E}[P_{a,t}]$, and $\boldsymbol{r} = (r(a))_{a\in \mathcal{A}}$, $\boldsymbol{c} = (c(a, i))_{a\in \mathcal{A}, i\in [d]}$. To ensure feasiblity, we assume there is a \textit{null action} $a_0 \in {\cal A}$, which yields no reward and no resource consumption whenever it is pulled. Equivalently, we allow the DM to do nothing (not take any action) in a time step. This is a common assumption in BwK literature. In the context of dynamic pricing, for example, the null action corresponds to a "no-sale" price (like pricing a hotel room stay at USD 999,999) to indicate a product's unavailability. For technical convenience, we also introduce an additional \textit{null resource} such that each action does not yield any comsumption on this resource. We dub it as the $(d+1)^{\text{th}}$ resource. Clearly, it satisfies $C_{t,d+1}=0$ with certainty and $c(a,d+1)=0$ for all $a \in \mathcal{A}$.


At each time $t$, the DM is provided with a \emph{prediction oracle} ${\cal F}_t$. The oracle is a function ${\cal F}_t: [\underline{q}, \overline{q}]^{t-1} \rightarrow [\underline{q}T, \overline{q}T]$ that provides an prediction $\hat{Q}_t = {\cal F}_t(q_1, \ldots, q_{t-1})$ on $Q = \sum_{t=1}^{T} q_t$ with the past observations $\{q_s\}^{t-1}_{s=1}$. At time $t$, the DM knows ${\cal A}, B, T, d, \{q_s\}^{t-1}_{s=1}$, and has the access to ${\cal F}_t$ in a sequential manner. In contrast, the DM does not know $\{P_{a,t}\}_{a\in {\cal A},t\in [T]}, \{q_s\}^T_{s=t}$, and lower and upper bound $\underline{q}$, $\overline{q}$. 

\textbf{Dynamics.} At each time $t$, three events happen. 
Firstly, the DM receives a prediction $\hat{Q}_t = \mathcal{F}_t(q_1,\ldots ,q_{t-1})$ on $Q$. Secondly, based on $\hat{Q}_t$ and the history observation, the DM chooses action $A_t \in \mathcal{A}$. Thirdly, the DM observes the feedback consisting of (i) demand volume $q_t$, (ii) reward earned $q_t R_t$, (iii) resources consumed $ \{q_t C_{t, i}\}_{i\in [d]}$. Recall that $(R_t, C_{t, 1}, \ldots, C_{t, d})\sim P_{A_t, t}$. Then, the DM proceeds to time $t+1$. If some resource is depleted, \emph{i.e.} $\exists j \in [d]$ such that $\sum_{s=1}^t q_s C_{s,j} > B_j$, then the null action $a_0$ is to be choosed in the remaining horizon $t+1, \ldots, T$. We denote the stopping time here as $\tau$. The DM aims to maximize the total reward $  \sum_{t=1}^{\tau - 1} q_t R_t$, subject to the resource constraints and model uncertainty.

\textbf{On $q_t$. }Our feedback model on $q_t$ is more informative than \cite{LykourisVR20}, where none of $q_1, \ldots, q_T$ is observed during the horizon. In contrast, ours is less informative than \cite{traca2021regulating}, where $q_1, \ldots, q_T$ are all observed at time 1. Our assumption of observing $q_t$ at the end of time $t$ is mild in online retail settings. For example, the number of visitors to a website within a time interval can be extracted from the electronic records when the interval ends.

While the nature sets $\{q_t\}^T_{t=1}$ to be fixed but arbitrary, the sequence is set without knowing the DM's online algorithm and prediciton oracle ${\cal F} = \{{\cal F}_t\}^T_{t=1}$. Our model is milder than the \emph{oblivious adversary} model, where the nature sets a latent quantity (in this case $\{q_t\}^T_{t=1}$) \emph{with} the knowledge of the DM's algorithm before the online dynamics. Our milder model allows the possibility of $\hat{Q}_t = \mathcal{F}_t(q_1,\ldots ,q_{t-1})$ being a sufficiently accurate (to be quantified in our main results) estimate to $Q$ for each $t$, for example when $\{q_t\}^T_{t=1}$ is govenred by a latent time series model. In contrary, an oblivious adversary can set $Q$ to be far away from the predictions $\hat{Q}_1, \ldots, \hat{Q}_T$ in response to the information on ${\cal F}$. 

\textbf{On $\mathcal{F} = \{\mathcal{F}_t\}$.} Our prediction oracle is a general Black-Box model. We do not impose any structural or parameteric assumption on ${\cal F}$ or $\{q_t\}^T_{t=1}$. It is instructive to understand ${\cal F}$ as a side information provided to the DM by an external source, like some consulting service. In the dynamic pricing example, $\hat{Q}_t$ could be a forecast on the customer base population provided by an external marketing research firm. A prime candidate of ${\cal F}$ is the cornucopia of time series prediction models proposed in decades of research works on time series \citep{ShumwayS2017,HyndmanA2021,LimZ21}. These prediction models allow \emph{one step prediction}, where for any $t$, the predictor ${\cal P}$ inputs $\{q_s\}^{t-1}_{s=1}$ and outputs an estimate $\hat{q}_t$ on $q_t$. The prediction $\hat{Q}_t$ can be constructed by (1) iteratively applying ${\cal P}$ on $\{q_s\}^{t-1}_{s=1} \cup \{\hat{q}\}^{t+\rho-1}_{t}$ to output $\hat{q}_{t+\rho}$, for $\rho\in \{0, \ldots, T-t\}$, (2) summing over $q_1, \ldots, q_{t-1}, \hat{q}_t, \ldots, \hat{q}_T$ and return $\hat{Q}_t$. We provide two specific examples in Section 5.






\textbf{Regret.} To measure the performance of an online algorithm, we define the regret of it as
\begin{equation}
    \begin{aligned}
    \text{Regret}_T = \text{OPT} -  \sum_{t=1}^{\tau - 1} q_t R_t,
    \end{aligned}
    \label{eq-regret}
\end{equation}
where $\text{OPT}$ denotes the expected cumulative reward of the offline optimal dynamic policy given all latent distribution $\{P_{a,t}\}_{a \in \mathcal{A}, t \in [T]}$ and demand sequence $\{q_t\}_{t=1}^T$. More specifically, our model can be in fact phrased as a Markov decision process (MDP) problem with $T$ time steps, where a state is represented by the vector of remaining inventory on each resource, and a state transition corresponds to consumption of resources. The offline optimal dynamic policy is defined as the optimal policy (the policy that accrues the highest expected total reward) in the MDP, and in particular a DM who has full model certainty on all parameters would be able to compute such a policy. The scalar quantity $\text{OPT}>0$ is the expected total reward under the offline optimal dynamic policy. Thus, we are comparing our online algorithm, which suffers from model uncertainty, with the optimal algorithm with full information.

For analytical tractabililty in our regret upper bound, we consider an alternative benchmark 
\begin{equation}
    \begin{aligned}
    \text{OPT}_{\text{LP}} = \max_{\boldsymbol{u} \in \Delta_{K}} \quad & \left( \sum_{t=1}^T q_t \right) \boldsymbol{r}^\top \boldsymbol{u} \qquad \\
    \text{s.t.} \quad & \left( \sum_{t=1}^T q_t \right) \boldsymbol{c}^{\top} \boldsymbol{u} \le B \boldsymbol{1}_d,
    \end{aligned}
    \label{eq-linear-relaxation2}
\end{equation}
The benchmark (\ref{eq-linear-relaxation2}) is justified by the following Lemma, which is proved in Appendix A: 
\begin{lemma}
$
    \text{OPT}_{\text{LP}} \ge \text{OPT}.
$
\label{lem-lpupperbound}
\end{lemma}

Based on this lemma, any regret upper bound derived using $\text{OPT}_{\text{LP}}$ as a benchmark will also serve as a regret upper bound for the original benchmark $\text{OPT}$.

\section{Impossibility Results: Regret Lower Bounds}\label{sec:lower_bound}
In this section, we provide impossibility results of our model in the form of regret lower bounds. 
Firstly, we show that a linear-in-$T$ regret is inevitable in the absence of the prediction oracle ${\cal F}$.
\begin{lemma}
Consider a fixed but arbitrary online algorithm that knows
\begin{equation*}
    \{P_{a,t}\}_{a\in {\cal A},t \in [T]},\quad \{(q_s, q_s R_s, q_sC_{s, 1},\ldots q_s C_{s, d})\}^{t-1}_{s=1},\quad q_t,
\end{equation*}
but does not have any access to a prediction oracle when the action $A_t$ is to be chosen at each time $t$. There exists an instance such that the online algorithm suffers $\text{Regret}_T = \Omega(T).$
\label{lem-lowerbound-oblivious}
\end{lemma}
Lemma \ref{lem-lowerbound-oblivious} is proved in Appendix \ref{sec:pf-lem-lowerbound-oblivious}. Lemma \ref{lem-lowerbound-oblivious} shows that even when all model information on time steps $1, \ldots, t$ are revealed when $A_t$ is to be chosen, the DM still suffers $\text{Regret}_T = \Omega(T).$ Thus, Our model is fundamentally different from non-stationary bandits without resource constraints such as \cite{besbes2015non}, and online optimization with predictions problems, such as \cite{rakhlin2013online}. In these settings, we can achieve $\text{Regret}_T = 0$ if all model information on time steps $1, \ldots, t$ are available at the time point of choosing $A_t$ or the action at time $t$. Indeed, given all model information at time $t$, the DM achieves the optimum by choosing an arm or an action that maximizes the reward function of time $t$ for every $t\in [T]$. 

In view of Lemma \ref{lem-lowerbound-oblivious}, we seek to understand if the DM can avoid $\text{Regret}_T = \Omega(T)$ when providing with an accurate prediction on $Q$. Certainly, if the DM only recieves an uninformative prediction, such as a worst case prediction $\hat{Q}_t = 0$, at each time step, $\text{Regret}_T = \Omega(T)$ still cannot be avoided. In contrast, if the DM received an \emph{accurate} prediction at a time step, we demonstrate our first step for deriving a better regret bound, in the form of a more benign regret lower bound compared to Lemma \ref{lem-lowerbound-oblivious}. We formalize the notion of being \emph{accurate} by the following two concepts.

For $T_0 \in [T-1]$ and $\epsilon_{T_0+1} \ge 0$, an instance $\{q_t\}_{t=1}^T$ is said to be \emph{$(T_0+1, \epsilon_{T_0+1})$-well estimated by oracle ${\cal F}$}, if the prediction $\hat{Q}_{T_0+1} = {\cal F}_{T_0+1}(q_1, \ldots, q_{T_0})$ returned by the oracle at time $T_0+1$ satisfies $|Q - \hat{Q}_{T_0+1}| \in \left [ \epsilon_{T_0+1}, 2 \epsilon_{T_0+1} \right ]$. This notion measures the power of prediction oracle $\mathcal{F}$. We say that $\epsilon_{T_0+1}$ is \emph{$(T_0+1, \{q_t\}_{t=1}^{T_0})$-well response} by oracle $\mathcal{F}$ if $\epsilon_{T_0+1}$ satisfies $\epsilon_{T_0+1} \leq \min\{\hat{Q}_{T_0+1} - \sum^{T_0}_{s=1} q_s - \underline{q}(T-T_0), \overline{q} (T-T_0)- (\hat{Q}_{T_0+1} - \sum^{T_0}_{t=1}q_t), \hat{Q}_{T_0+1}/2\}$, where $\hat{Q}_{T_0+1} = {\cal F}_{T_0+1}(q_1, \ldots, q_{T_0})$. This concept imposes requirements on the quality of prediction by introducing a non-trivial upper bound on $\epsilon_{T_0+1}$ for the "well-estimate" notion. This can help us eliminate trivial and uninformative predictions such as $\hat{Q}_t = 0$ or $\bar{q}T$.
\begin{theorem}
Consider our model setting, and consider a fixed but arbitrary online algorithm and prediciton oracle ${\cal F} = \{{\cal F}_t\}^T_{t=1}$. For any $T_0\in [T-1]$ and any $\epsilon_{T_0+1} > 0$ that is $(T_0+1, \{q_t\}_{t=1}^{T_0})$-well response, there exists a $(T_0+1, \epsilon_{T_0+1})$-well estimated instance $I = \{q_t\}^{T_0}_{t=1} \cup \{q_t\}^{T}_{t=T_0 + 1}$ such that
\begin{equation}\label{eq:regret_lb_overall}
    \text{Regret}_T= \Omega \left( \max\left\{ \frac{1}{Q} \sum^{T_0}_{t=1} q_t \epsilon_{T_0+1}~, ~\Lambda \right\} \right),
\end{equation}
where $Q = \sum^T_{t=1}q_t,$ and %
$$
\Lambda = 
\min \left \{ \text{OPT},\text{OPT} \sqrt{\frac{\overline{q}K}{B}} + \sqrt{ \overline{q} K \text{OPT}} \right \}.
$$
\label{thm-lowerbound-estimation}
\end{theorem}
Theorem \ref{thm-lowerbound-estimation} is proved in Appendix \ref{sec:pf-thm-lowerbound-estimation}. 
In (\ref{eq:regret_lb_overall}), the regret lower bound $\Lambda$ is due to the uncertainty on $\{P_{a,t}\}_{a\in \mathcal{A}, t \in [T]}$, and $\Lambda$ is derived directly from \cite{badanidiyuru2013bandits}. The regret lower bound  $\frac{1}{Q} \sum^{T_0}_{t=1} q_t \epsilon_{T_0+1}$ is due to the oracle's prediction error on $\hat{Q}_{T_0+1}$. Theorem \ref{thm-lowerbound-estimation} demonstrates a more benign regret lower bound than $\Omega(T)$, under the condition that the prediction on $Q$ is sufficiently accurate (as formalized as $(T_0+1, \epsilon_{T_0+1})$-well estimated). 

More specifically, let us consider the following \emph{accurate prediction condition} at time $T_0$ by oracle $\mathcal{F}$: $\epsilon_{T_0+1}$ is $(T_0+1, \{q_t\}_{t=1}^{T_0})$-well response by oracle $\mathcal{F}$ and
\begin{equation}\label{eq:alpha}
\frac{\epsilon_{T_0+1}}{Q} = O(T_0^{-\alpha}) \quad \text{for some $\alpha >0$}.
\end{equation}
The condition implies that, for the prediction $\hat{Q}_{T_0+1}$ made using $T_0$ data points $q_1, \ldots, q_{T_0}$, it holds that $|1- (\hat{Q}_{T_0+1}/Q)| = O(T_0^{-\alpha})$. For example, when $\{q_t\}^T_{t=1}$ are i.i.d. generated, the accurate prediction condition holds with $\alpha=1/2$.
\begin{corollary}\label{cor:benign}
    Consider the setting of Theorem \ref{thm-lowerbound-estimation}. Suppose the accurate prediction condition (\ref{eq:alpha}) holds at $T_0$, then the refined regret lower bound $\text{Regret}_T = \Omega(\max\{\underline{q}T^{1-\alpha}_0, \Lambda\})$ holds. 
\end{corollary}
Altogether, under the accurate prediction condition, the corollary presents a strictly smaller regret lower bound than that in Lemma \ref{lem-lowerbound-oblivious}, which has no prediction oracle available. In complement, we design and analyze an online algorithm in the next section that leverages the benefits of predictions, and in particular nearly matches the regret lower bound in Corollary \ref{cor:benign} under the accurate prediction condition. Thus, a $o(T)$-regret is possible in a non-stationary environment given accurate predictions as prescribed above, even though the amount of non-stationarity in the underlying model is not bounded in general.

\section{Algorithm and Analysis}
We propose the Online-Advice-UCB (OA-UCB) algorithm, displayed in Algorithm \ref{alg-OAU}, for solving our model. 
The algorithm design involves constructing confidence bounds to address the model uncertainty on $\boldsymbol{r}, \boldsymbol{c}$, as discussed in Section \ref{sec:conf_bd}. 
In Section \ref{sec:main_alg}, we provide detail description on OA-UCB, which utilizes Online Convex Optimization (OCO) tools to 
balance the trade-off among rewards and resources. Crucially, at each time $t$, we incorporate the prediction $\hat{Q}_t$ to scale the opportunity costs of the resources. Furthermore, both $q_t$ and $\hat{Q}_t$ are judiciously integrated into the OCO tools to factor the demand volumes into the consideration of the abovemention trade-off. In Section \ref{sec:main_thm}, we provide a regret upper bound to OA-UCB, and demonstrate its near-optimality when the accurate prediction condition (\ref{eq:alpha}) holds and when capacity is large. In Section \ref{sec:sketch} we provide a sketch proof of the regret upper bound, where the complete proof is in Appendix \ref{sec:pf-thm-upperbound-1}.



\subsection{Confidence Bounds}\label{sec:conf_bd}
We consider the following confidence radius function:
\begin{equation}
    \text{rad}(v,N,\delta) = \sqrt{\frac{2 v \log \left(\frac{1}{\delta} \right) }{N}} + \frac{4 \log\left(\frac{1}{\delta} \right)}{N}.
    \label{eq-def-conradius}
\end{equation}
The function (\ref{eq-def-conradius}) satisfies the following property:
\begin{lemma}[\cite{babaioff2015dynamic,agrawal2014bandits}]\label{lem:conf}
Let random variables $\{V_i\}^N_{i=1}$ be independently distributed with support in $[0,1]$. Denote $\hat{V} = \frac{1}{N}\sum_{i=1}^N V_i$, then with probability $\ge 1 - 3 \delta$, we have
\begin{equation*}
    \left |\hat{V} - \mathbb{E} [\hat{V} ] \right | \le \text{rad} (\hat{V},N,\delta ) \le 3 \text{rad} (\mathbb{E} [\hat{V} ],N,\delta ).
\end{equation*}
\label{lem-radius}
\end{lemma}
We prove Lemma \ref{lem:conf} in Appendix \ref{sec:pf_lem-radius} by following the line of argument in \cite{babaioff2015dynamic} for the purpose of extracting the values of the coefficients in (\ref{eq-def-conradius}), which are implicit in \cite{babaioff2015dynamic, agrawal2014bandits}. Based on the observation $\{R_s,\{C_{s,i}\}_{i\in [d]}\}_{s \in [t-1]}$, we compute the sample means 
\begin{equation*}
    \hat{R}_t(a) =  \frac{1}{N_{t-1}^+(a)} \sum_{s=1}^{t-1} R_s \textbf{1}_{\{ A_s = a\}},\quad \forall a \in \mathcal{A},
\end{equation*}
\begin{equation*}
     \hat{C}_t(a,i) =  \frac{1}{N_{t-1}^+(a)} \sum_{s=1}^{t-1} C_{s,i} \textbf{1}_{\{ A_s = a\}},~ \forall a \in \mathcal{A},\ i \in [d],
\end{equation*}
where $N_{t-1}^+(a) = \max  \{\sum_{s=1}^{t-1} \textbf{1}_{\{ A_s = a\}},1  \}$. In line with the principle of Optimism in Face of Uncertatinty, we construct upper confidence bounds (UCBs) for the rewards and lower confidence bounds (LCBs) for resource consumption ammounts. For each $a\in \mathcal{A}$, we set
\begin{equation}
    \text{UCB}_{r,t}(a) = \min\left \{\hat{R}_{t}(a) + \text{rad}(\hat{R}_{t}(a), N_{t-1}^+(a), \delta),1 \right \}.
    \label{eq-ucb-r}
\end{equation}
For each $a\in \mathcal{A}, i\in [d+1]$, we set 
\begin{equation}
   \text{LCB}_{c,t}(a,i) = 
   \begin{cases}\max\left \{\hat{C}_t(a,i) - \text{rad}(\hat{C}_t(a,i), N_{t-1}^+(a),\delta),0 \right \} & i \in [d] \\
   0 & i = d+1.
   \end{cases}
    \label{eq-lcb-c}
\end{equation}
The design of the UCBs and LCBs are justified by 
Lemma \ref{lem:conf} and the model assumption that $r(a), c(a, i)\in [0, 1]$ for all $a\in {\cal A}, i\in [d+1]$:

\begin{lemma}
With probability $\ge 1 - 3 KTd \delta$, we have
\begin{equation*}
    \text{UCB}_{r,t}(a) \ge r(a),\ \ \text{LCB}_{c,t}(a,i) \le c(a,i)
\end{equation*}
for all $a\in \mathcal{A}, i\in [d+1]$.
\label{lem-confidencebound}
\end{lemma}
Lemma \ref{lem-confidencebound} is proved in Appendix \ref{sec:pf_lem-confidencebound}.

\subsection{Details on OA-UCB}\label{sec:main_alg}
OA-UCB is presented in Algorithm \ref{alg-OAU}. At each time step $t$, the algorithm first computes a composite reward term (Line 4-5)
\begin{equation}\label{eq:trade-off}
    \text{UCB}_{r,t}(a) - \frac{\hat{Q}_t}{B} \cdot  \boldsymbol{\mu}_t^\top \textbf{LCB}_{c,t}(a),
\end{equation}
where $\text{UCB}_{r,t}(a), \hat{Q}_t$ and $\textbf{LCB}_{c,t}(a)$ are the surrogates for the latent $r(a), Q, \textbf{c}(a)$ respectively. The term $\frac{\hat{Q}_t}{B} \cdot  \boldsymbol{\mu}_t^\top \textbf{LCB}_{c,t}(a)$ can be interpreted as the opportunity cost of the resources. The scalarization $\boldsymbol{\mu}_t \in \Delta_{d+1}$ weighs the relative importance of the resources. The factor $\hat{Q}_t / B$ reflects that the opportunity cost increases with $\hat{Q}_t$, since with a higher total demand volume, the DM is more likely to exhaust some of the resources during the horizon, and similar reasoning holds for $B$. Altogether, (\ref{eq:trade-off}) balances the trade-off between the reward of an action and the opportunity cost of that action's resource consumption. We choose an action that maximizes (\ref{eq:trade-off}) at time $t$ (Line 5).

After receiving the feedback, We update the scalarization  $\boldsymbol{\mu}_t$ (Line 10-11) via the AdaHedge (\cite{orabona2019modern,orabona2015scale,orabona2018scale}). AdaHedge is a powerful OCO tool that can adapt to the characteristics of the sequence of functions without relying on future information. We apply this tool on the sequence of functions $\{f_t\}^T_{t=1}$, where 
\begin{equation}\label{eq:ft}
    f_t(\boldsymbol{x}) = \frac{q_t \hat{Q}_t}{B}\left(\frac{B}{\hat{Q}_t} \boldsymbol{\beta} - \textbf{LCB}_{c,t}(A_t)\right)^{\top} \boldsymbol{x},
\end{equation}
where $\boldsymbol{\beta}=(\beta_i)_{i \in [d+1]} \in \mathbb{R}^{d+1}$ satisfies $\beta_i = 1$ for $i \in [d]$ and $\beta_{d+1}=0$. The expression of $f_t$ serves to incorporate not only the prediction $\hat{Q}_t$ in order to account for the estimated opportunity cost, as seen in equation (\ref{eq:trade-off}), but also the actual demand $q_t$ to accurately capture the amount of resources consumed. In fact, the choice of $f_t$ is largely driven by our analysis. Let us provide further intuitions: To account for the opportunity cost for each resource, we need to consider the information of the total demand $Q$ and total budget $B$. The inner part, $\left(\frac{B}{\hat{Q}_t} \boldsymbol{\beta} - \textbf{LCB}_{c,t}(A_t)\right)$, can be understood as the remaining resource after choosing $A_t$, assuming that the total resource is scaled down from $B$ to $\frac{B}{\hat{Q}_t}$. Next, the ratio $\frac{\hat{Q}_t}{B}$ reflects the estimated amount of demand volume that 1 unit of a resource can serve, where this estimation is rooted in the fact that we replace the latent $Q$ with the prediction $\hat{Q}_t$. Thus, the product of the two terms represents the estimated demand volume that the remaining resource can serve. For a resource $i$, higher consumption from time 1 to $t$ results in a lower value of demand that this resource can serve in the remaining period, thereby increasing its opportunity cost and improving the willingness of the DM to choose actions with lower consumption on this resource in the future. Lastly, the effect of each round is scaled by $q_t$ to accurately reflect the demand volume at time $t$. Collectively, these considerations result in the expression for $f_t$.

\begin{algorithm}[htb]
	\caption{Online-advice-UCB (OA-UCB)}
	\begin{algorithmic}[1]
	    \State \textbf{Intialize} $\boldsymbol{\mu}_1 = \frac{1}{d+1} \boldsymbol{1} = \left(\frac{1}{d+1},\cdots,\frac{1}{d+1}\right) \in \mathbb{R}^{d+1} $, $\eta_1 = 0$, $\boldsymbol{\theta}_1 = \boldsymbol{0} \in \mathbb{R}^{d+1}$, $\kappa = \sqrt{\ln (d+1)} $, $\boldsymbol{B} = (B_i)_{i \in [d]} = (B,...,B)$, $\text{Reward} = 0$.
        \For{$t=1,2,...,T$}
        \State Receive $\hat{Q}_t = {\cal F}_t(q_1, \ldots, q_{t-1})$.
        \State Compute $\text{UCB}_{r,t}(a)$, $\text{LCB}_{c,t}(a,i)$ for all $a \in \mathcal{A}$, $i \in [d]$ by (\ref{eq-ucb-r}), (\ref{eq-lcb-c}), $\textbf{LCB}_{c,t}(a) = (\text{LCB}_{c,t}(a,i))_{i \in [d]}$.
        \State Select \label{alg:opt}
        \begin{equation*}
            A_t \in \underset{a \in \mathcal{A}}{\text{argmax}} \left\{ \text{UCB}_{r,t}(a) - \frac{\hat{Q}_t}{B} \cdot  \boldsymbol{\mu}_t^\top \textbf{LCB}_{c,t}(a) \right\}.
        \end{equation*}
        \State Observe $q_t$, receive reward $q_t R_t$, and consume $q_t C_{t,i}$ for each resource $i\in [d]$.
        \If{$\exists j \in [d]$ such that $\sum_{s=1}^t q_s C_{t,j}> B$}
        \State Break, and pull the null action $a_0$ all the way.
        \EndIf
        \State Update
        \begin{equation*}
            \text{Reward} = \text{Reward} + q_t R_t,\quad B_i = B_i - q_t C_{t,i},\ i \in [d]
        \end{equation*}
        \State Set
        \begin{align}
        & \boldsymbol{g}_t = \frac{q_t \hat{Q}_t}{B}\left(\frac{B}{\hat{Q}_t} \boldsymbol{\beta} - \textbf{LCB}_{c,t}(A_t)\right), \label{eq-oau-gt} \\
        & \rho_t = \begin{cases}
                \boldsymbol{g}_1^{\top} \boldsymbol{\mu}_1  - \min_{j=1,...,d+1} g_{1,j} & t = 1\\ 
                \eta_t \ln \left(\sum_{j=1}^d \mu_{t,j} \exp\left(\frac{-g_{t,j}}{\eta_t} \right) \right) + \boldsymbol{g}_t^{\top} \boldsymbol{\eta}_t & \text{otherwise}.
            \end{cases} \nonumber
        \end{align}
        \State Update
        \begin{equation}
            \boldsymbol{\theta}_{t+1} = \boldsymbol{\theta}_{t} - \boldsymbol{g}_t,\quad \eta_{t+1} = \eta_t + \frac{1}{\kappa^2} \rho_t,\quad \mu_{t+1,j} = \frac{\exp \left(\frac{\theta_{t+1,j}}{\eta_{t+1}} \right)}{\sum_{i=1}^{d+1} \exp \left(\frac{\theta_{t+1,i}}{\eta_{t+1}} \right)}, \ j \in [d+1]. 
            \label{eq-oau-ocoupdate}
        \end{equation}
        \EndFor
	\end{algorithmic}
	\label{alg-OAU}
\end{algorithm}

\subsection{Performance Guarantees of OA-UCB}\label{sec:main_thm}
The following theorem provides a high-probability regret upper bound for Algorithm \ref{alg-OAU}:
\begin{theorem}
\label{thm-upperbound-1}
Consider the OA-UCB algorithm, that is provided with predictions that satisfy $|\hat{Q}_t - Q| \le \epsilon_t$ for all $t\in [T]$. With probability $\ge 1 - 3KTd \delta$, 
\begin{align}
    \text{OPT}_{\text{LP}} - \sum_{t=1}^{\tau - 1} q_t R_t  & \le   O \left( \left( \text{OPT}_{\text{LP}} \sqrt{\frac{\overline{q} K}{B}} + 
    \sqrt{\overline{q} K \text{OPT}_{\text{LP}} } \right) \log\left(\frac{1}{\delta} \right) \right. \label{eq:regret_1st}\\
    & +  \left. \left(\frac{1}{Q} + \frac{1}{B} \right) \sum_{t=1}^{\tau-1} q_t \epsilon_t + \left(\overline{q} + \frac{\overline{q}^2}{b} \right) \sqrt{(\tau - 1)\ln(d+1)} \right). \label{eq:regret_23}
\end{align}
\end{theorem}
Theorem \ref{thm-upperbound-1} is proved in the Appendix, and we provide a sketch proof in Section \ref{sec:sketch}. The Theorem holds in the special case when we set $\epsilon_t = |\hat{Q}_t - Q|$, and $\epsilon_t$ represents an upper bound on the prediction error of $\hat{Q}_t$ on $Q$, for example by certain theoretical guarantees. The term (\ref{eq:regret_1st}) represents the regret due to the learning on $r(a), c(a, i)$. The first term in (\ref{eq:regret_23}) represents the regret due to the prediction error of the prediction oracle, and the second term in (\ref{eq:regret_23}) represents the regret due to the application of OCO. 

We provide more intuition on how the one-step advice $\hat{Q}_t$ can improve the algorithm performance. The expression of $\text{OPT}_{\text{LP}}$ demonstrates the critical role that the total demand value $Q$ plays in an optimal policy. In the ideal case when the DM knows $Q$, $\boldsymbol{r}$ and $\boldsymbol{c}$, s/he could in fact obtain a $O(\bar{q}\sqrt{T})$ regret by solving the linear program (LP) $\text{OPT}_{\text{LP}}$ and following a static randomized rule based on the LP's optimal solution. However, in our setting, all of $Q$, $r$ and $c$ are not known. We follow the traditional techniques in learning $\boldsymbol{r}, \boldsymbol{c}$ while optimizing in the online process. Our crucial contribution lies in our novel way of utilizing the predictions $\{\hat{Q}_t\}^T_{t=1}$ to achieve a regret that depends on these predictions's errors $\{\epsilon_t\}^T_{t=1}$, despites not having knowledge on these errors, meaning that the DM does not know how reliable each prediction is when making online decisions. It is important to note that while we use the LP to build intuition, it does not lead to a succcessful algorithm, and we rely on a OCO procedure that incoporates  $\{\hat{Q}_t\}^T_{t=1}$, $\{q_t\}^T_{t=1}$ in a novel manner.

\textbf{Comparison between regret lower and upper bounds.} 
The regret term (\ref{eq:regret_1st}) matches the lower bound term $\Lambda$ in Theorem \ref{thm-lowerbound-estimation} within a logarithmic factor. Next, we compare the regret upper bound term $(\frac{1}{Q} + \frac{1}{B} ) \sum_{t=1}^{\tau-1} q_t \epsilon_t$ and the lower bound term $\frac{1}{Q}\sum^{T_0}_{t=1}q_t\epsilon_{T_0+1}$ in  Theorem \ref{thm-lowerbound-estimation}. We first assure that the lower and upper bound results are consistent, in the sense that our regret upper bound is indeed in $\Omega(\frac{1}{Q}\sum^{T_0}_{t=1}q_t\epsilon_{T_0+1})$ on the lower bounding instances constructed for the proof of Theorem \ref{thm-lowerbound-estimation}. In those instances, $T_0$ is set in a way that the resource is not fully exhausted at time $T_0$ under any policy, thus the stopping time  $\tau$ of OA-UCB  satisfies $\tau> T_0$ with certainty. More details are provided in Appendix \ref{sec:consistent}. 

Next, we highlight that the regret upper and lower bounds are nearly matching (modulo multiplicative factors of $\log(1/\delta)$ as well as the additive $O\left(\left(\overline{q} + \frac{\overline{q}^2}{b} \right) \sqrt{(\tau - 1)\ln(d+1)} \right)$ term), under the high capacity condition $B = \Theta(Q)$ and the accurate prediction condition (\ref{eq:alpha}) for each $T_0 \in [T]$. The first condition is similar to the large capacity assumption in the literature \cite{besbes2009dynamic,besbes2012blind,liu2022non}, while the second condition is a natural conditon that signfies a non-trivial estimation by the prediciton oracle, as discussed in Section \ref{sec:lower_bound}. On one hand, by setting $T_0 = \Theta(T)$ for the highest possible lower bound in Corollary \ref{cor:benign}, we yield the regret lower bound $\Omega(\max\{\underline{q}T_0^{1-\alpha},\Lambda\})=\Omega(\max\{\underline{q}T^{1-\alpha},\Lambda\})$. On the other hand, the second term in (\ref{eq:regret_23}) is upper bounded as
\begin{align*}
    \left(\frac{1}{Q} + \frac{1}{B} \right) \sum_{t=1}^{\tau-1} q_t \epsilon_t = O\left(\sum_{t=1}^{\tau-1} q_t \frac{\epsilon_t}{Q}\right) = O\left(\sum_{t=1}^{\tau-1} q_t t^{-\alpha}\right) = O(\overline{q} T^{1-\alpha}).
\end{align*}
Altogether, our claim on the nearly matching bounds is established.

\subsection{Proof Sketch of Theorem \ref{thm-upperbound-1}}  \label{sec:sketch} 
We provide an overview on the proof of Theorem \ref{thm-upperbound-1}, which is fully proved in Appendix \ref{sec:pf-thm-upperbound-1}. 
We first provide bounds on the regret induced by the estimation errors of the UCBs and LCBs. Now, with probability $\ge 1 - 3KTd \delta$, the inequalities 
\begin{equation}
    \left | \sum_{t=1}^{\tau - 1} q_t \text{UCB}_{r,t}(A_t) - \sum_{t=1}^{\tau - 1} q_t R_t \right | \le O  \left( \log\left(\frac{1}{\delta}\right) \left ( \sqrt{\overline{q} K \sum_{t=1}^{\tau - 1} q_t R_t} +  \overline{q} K \log\left (\frac{T}{K} \right ) \right) \right),
    \label{eq-spf-upper-ucb}
\end{equation}
\begin{equation}
    \left | \sum_{t=1}^{\tau - 1} q_t \text{LCB}_{c,t}(A_t,i) - \sum_{t=1}^{\tau - 1} q_t C_{t,i} \right | \le  O  \left( \log\left(\frac{1}{\delta}\right) \left ( \sqrt{\overline{q}K B} +  \overline{q} K \log\left (\frac{T}{K} \right ) \right) \right), \ \forall i \in [d].
    \label{eq-spf-upper-lcb}
\end{equation}
hold. Inequalities (\ref{eq-spf-upper-ucb}, \ref{eq-spf-upper-lcb}) are proved in Appendix  \ref{sec:pf_upper_lower}. Next, by the optimality of $A_t$ in Line \ref{alg:opt} in Algorithm \ref{alg-OAU}, the inequality
\begin{equation}
    \text{UCB}_{r,t}(A_t) - \frac{\hat{Q}_t}{B} \cdot  \boldsymbol{\mu}_t^\top \textbf{LCB}_{c,t}(A_t) \ge \textbf{UCB}_{r,t}^{\top} \boldsymbol{u}^* - \frac{\hat{Q}_t}{B} \cdot  \boldsymbol{\mu}_t^\top \textbf{LCB}_{c,t} \boldsymbol{u}^*,
    \label{eq-spf-upper-stepoptimality-lp}
\end{equation}
holds, which is also proved in Appendix  \ref{sec:pf_upper_lower}. (\ref{eq-spf-upper-stepoptimality-lp}) is equivalent to the following
\begin{equation}
    \begin{aligned}
    \textbf{UCB}_{r,t}^\top \boldsymbol{u}^* - \text{UCB}_{r,t}(A_t) + \frac{\hat{Q}_t}{B} \boldsymbol{\mu}_t^{\top} \left(\frac{B}{\hat{Q}_t} \boldsymbol{\beta} - \textbf{LCB}_{c,t}^{\top} \boldsymbol{u}^* \right) \le \frac{\hat{Q}_t}{B} \cdot  \boldsymbol{\mu}_t^\top \left(\frac{B}{\hat{Q}_t} \boldsymbol{\beta}- \textbf{LCB}_{c,t}(A_t) \right).
    \end{aligned}
    \label{eq-spf-upper-stepoptimality}
\end{equation}
Follow the line of arguments in \cite{orabona2019modern,orabona2015scale,orabona2018scale}, the following performance guarantee for our OCO tool holds with $\{f_t\}^T_{t=1}$ defined in (\ref{eq:ft}), that for all $\boldsymbol{\mu} \in \Delta_{d+1}$,
\begin{equation}
    \sum_{t=1}^T f_t(\boldsymbol{\mu}_t) - \sum_{t=1}^T f_t(\boldsymbol{\mu}) \le  O \left( \left(\overline{q} + \frac{\overline{q}^2}{b} \right) \sqrt{T\ln(d+1)} \right)
    \label{eq-spf-ocoregret}
\end{equation}
Inequality (\ref{eq-spf-ocoregret}) is completed proved in Appendix \ref{sec:pf_oco_pg}. Multiply $q_t$ on both side of (\ref{eq-spf-upper-stepoptimality}), and sum over $t$ from 1 to $\tau - 1$. By applying the OCO performance guarantee in (\ref{eq-spf-ocoregret}), we argue that, for all $\boldsymbol{\mu} \in \Delta_{d+1}$,
\begin{equation*}
    \begin{aligned}
    & \sum_{t=1}^{\tau - 1} q_t \textbf{UCB}_{r,t}^{\top} \boldsymbol{u}^* - \sum_{t=1}^{\tau - 1}q_t \text{UCB}_{r,t}(A_t) + \sum_{t=1}^{\tau - 1}q_t \frac{\hat{Q}_t}{B} \cdot  \boldsymbol{\mu}_t^{\top} \left(\frac{B}{\hat{Q}_t} \boldsymbol{\beta} - \textbf{LCB}_{c,t}^{\top} \boldsymbol{u}^* \right) \\
    \le & \sum_{t=1}^{\tau-1} q_t \frac{\hat{Q}_t}{B} \cdot \left(\frac{B}{\hat{Q}_t} \boldsymbol{\beta} - \textbf{LCB}_{c,t}(A_t) \right)^{\top} \boldsymbol{\mu} + O \left( \left(\overline{q} + \frac{\overline{q}^2}{b} \right) \sqrt{(\tau - 1)\ln(d+1)} \right).
    \end{aligned}
\end{equation*}
If $\tau \le T$, then there exists $j_0 \in [d]$ such that $\sum_{t=1}^{\tau} q_t C_{t,j_0} > B$. Take $\boldsymbol{\mu} = \frac{\text{OPT}_{\text{LP}}}{Q} \boldsymbol{e}_{j_0} + \left(1 - \frac{\text{OPT}_{\text{LP}}}{Q} \right) \boldsymbol{e}_{d+1} \in \Delta_{d+1}$ (This is because $\text{OPT}_{\text{LP}} = Q \boldsymbol{r}^{\top} \boldsymbol{u}^* \le Q$). Analysis yields 
\begin{align}
    & \text{OPT}_{\text{LP}} - \sum_{t=1}^{\tau - 1}q_t \text{UCB}_{r,t}(A_t) \le  O \left(\log\left(\frac{1}{\delta} \right) \text{OPT}_{\text{LP}} \sqrt{\frac{\overline{q} K}{B}}   \right.\nonumber\\
    &\left. + \left (\frac{1}{Q} + \frac{1}{B} \right) \sum_{t=1}^{\tau-1} q_t \epsilon_t + \left(\overline{q} + \frac{\overline{q}^2}{b} \right) \sqrt{(\tau - 1)\ln(d+1) } \right).\label{eq-spf-upper-taubound}
\end{align}
If $\tau > T$, it is the case that $\tau - 1 = T$, and no resource is exhausted at the end of the horizon. Take $\boldsymbol{\mu} = \boldsymbol{e}_{d+1}$. Similar analysis to the previous case shows that  
\begin{equation}
    \text{OPT}_{\text{LP}} - \sum_{t=1}^{\tau - 1}q_t \text{UCB}_{r,t}(A_t) \le O\left(\frac{1}{Q} \sum_{t=1}^{T} q_t \epsilon_t + \left(\overline{q} + \frac{\overline{q}^2}{b} \right) \sqrt{T \ln(d+1)} \right).
    \label{eq-spf-upper-Tbound}
\end{equation}
Combine (\ref{eq-spf-upper-ucb}), (\ref{eq-spf-upper-taubound}), (\ref{eq-spf-upper-Tbound}) and the fact that $\text{OPT}_{\text{LP}} \ge \sum_{t=1}^{\tau - 1} q_t R_t$, the theorem holds.



\section{Two Specific Demand Models}

This section presents two specific demand models: the linearly increasing model and the AR(1) model. We provide asymptotically accurate prediction oracles for each of them. By using these advice, we derive explicit expressions of the upper bound for the prediction error, i.e., $\epsilon_t$ in Theorem \ref{thm-upperbound-1}, as well as the corresponding regret upper bound. Our analysis shows that the regret for both models are near optimal except for a $\Tilde{O}\left(\sqrt{\tau - 1} \right)$ term.

\subsection{Linearly Increasing Model}

The linearly increasing model is a commonly used demand model in many operational applications, such as inventory management and revenue management. It assumes that the demand for a product increases linearly over time. This model is appropriate when there is a "expanding market", that is, the demand for a product has an increasing trend due to 
promotions or other external factors. We consider the following linearly increasing model for $\{q_t\}_{t=1}^T$:
\begin{equation}
    q_t = \alpha + \beta t + \xi_t,
    \label{eq-timeseries-linear}
\end{equation}
where $\{\xi_t\}_{t=1}^T$ is a sequence of independent random variables, each of which is bounded and zero-mean, supported in $[-M,M]$, and $\alpha, \beta$ are unknown parameters satisfying $\alpha > M > 0$ and $\beta > 0$. This model can be applied to some expanding markets. We apply the Least Square method and provide the following prediction oracle:
\begin{equation}
    \hat{Q}_t = \sum_{s=1}^{t-1} q_s + \sum_{s=t}^T \left( \hat{\alpha}_t + \hat{\beta}_t s\right),
    \label{eq-timeseries-linear-prediction}
\end{equation}
where $\hat{\alpha}_t$, $\hat{\beta}_t$ are calculated via Least Square method as following:
\begin{equation*}
    \hat{\alpha}_t = \frac{\sum_{s=1}^{t-1} q_s - \hat{\beta}_t \sum_{s=1}^{t-1} s}{t-1}, \quad \hat{\beta}_t = \frac{(t-1)\sum_{s=1}^{t-1} s q_s - \left(\sum_{s=1}^{t-1} s \right) \left(\sum_{s=1}^{t-1} q_s \right)}{(t-1)\sum_{s=1}^{t-1} s^2 - \left(\sum_{s=1}^{t-1} s \right)^2}.
\end{equation*}

The following lemma provides a theoretical guarantee on the accuracy of the prediction oracle:

\begin{lemma}
    Consider the linearly increasing demand model (\ref{eq-timeseries-linear}) with prediction oracle $\hat{Q}_t$ as (\ref{eq-timeseries-linear-prediction}) on $Q = \sum_{t=1}^T q_t$, then with probability $\ge 1 - T \delta$,
    \begin{equation*}
        \left |Q - \hat{Q}_t \right | \le O \left(M T^2 \sqrt{ \log \left( \frac{1}{\delta}\right)}  \left(t-1 \right)^{-\frac{3}{2}}\right):= \epsilon_t.
    \end{equation*}
    \label{lem-timeseries-accuracy-linear}
\end{lemma}

Lemma \ref{lem-timeseries-accuracy-linear} is proved in Appendix F. Combining with the regret upper bound for general prediction upper bound $\epsilon_t$ (Theorem \ref{thm-upperbound-1}), the high-probability regret upper bound under the linearly increasing model (\ref{eq-timeseries-linear}) is provided as follow:

\begin{theorem}
Consider the model under the linearly increasing demand model (\ref{eq-timeseries-linear}), and the DM is provided with predictions $\hat{Q}_t$ as (\ref{eq-timeseries-linear-prediction}) for all $t \in [T]$. Take $\overline{q} = \alpha + \beta T + M$. Assume $b = \Theta(\overline{q})$, then with probability $\ge 1 - T \delta$, the OA-UCB algorithm satisfies
\begin{equation*}
    \left(\frac{1}{Q} + \frac{1}{B} \right) \sum_{t=1}^{\tau - 1} q_t \epsilon_t = O \left(M\left(1 + \frac{\overline{q}}{b}\right) \sqrt{(\tau-1) \log\left( \frac{1}{\delta}\right)} \right) = \Tilde{O}(M\sqrt{\tau - 1})
\end{equation*}
Thus, with probability $\ge 1 - 3KTd\delta - T\delta$, the OA-UCB algorithm induces regret 
\begin{equation*}
        \text{OPT}_{\text{LP}} - \sum_{t=1}^{\tau - 1} q_t R_t \le \Tilde{O} \left(\text{OPT}_{\text{LP}} \sqrt{\frac{\overline{q} K}{B}} + 
    \sqrt{\overline{q} K \text{OPT}_{\text{LP}} } + (M+\overline{q})\sqrt{\tau - 1} \right).
\end{equation*}
\label{thm-upperbound-linear}
\end{theorem}
Theorem \ref{thm-upperbound-linear} is proved in Appendix F. Clearly in this case, the regret lower and upper bound are nearly matched except for a $\Tilde{O} \left( (M+\overline{q})\sqrt{\tau - 1} \right)$ term, which shows the near-optimality of our algorithm.

\subsection{AR(1) Model}

The Autoregressive Moving Average (ARMA) model is a widely used time series model in literature to forecast demand in various applications such as inventory management, production planning, and supply chain management \cite{Aviv03}. The ARMA model assumes that the demand in a time period is a function of the demand in previous time periods and a random error term. The ARMA model captures both the autoregressive and moving average effects of a time series. The autoregressive effect refers to the relationship between the current demand and past demand values, while the moving average effect refers to the relationship between the current demand and past error terms. The ARMA model can be expressed as ARMA($p$, $q$), where $p$ and $q$ are the orders of the autoregressive and moving average parts, respectively. In this section, we focus on a special case of ARMA model, AR(1), where $p = 1$ and $q = 0$. The AR(1) model for $\{q_t\}^T_{t=1}$ is defined as 
\begin{equation}
    q_t = \alpha + \beta q_{t-1} + \xi_t.
    \label{eq-timeseris-ar1}
\end{equation}
where $\{\xi_t\}_{t=1}^T$ is a sequence of independent random variables, each of which follows a zero-mean $\sigma^2-$subgaussian distribution, and $\alpha$, $\beta$, $\sigma$ are unknown parameters satisfying $\alpha>0$, $|\beta| \in [-M,M]$, where $M < 1$.
Denote $\boldsymbol{\gamma} = (\alpha,\beta)$. Motivated by \cite{bacchiocchi2022autoregressive}, we construct the parameter estimation $\hat{\boldsymbol{\gamma}}_t$ for $\boldsymbol{\gamma}$ by solving the following ridge regression problem:
\begin{equation*}
    \hat{\boldsymbol{\gamma}}_t = \left(\hat{\alpha}_t,\hat{\beta}_t \right)=  \mathop{\arg\min}_{\boldsymbol{\gamma}'=(\alpha',\beta') \in \mathbb{R}^2} \left \{\sum_{s=1}^{t-1} \left(q_s - \alpha' - \beta' q_{s-1} \right)^2 + \lambda \|\boldsymbol{\gamma}'\|_2^2 \right \} = \boldsymbol{V}_t^{-1} \boldsymbol{\zeta}_t,
\end{equation*}
where
\begin{equation*}
    \boldsymbol{V}_t = \lambda I_2 + \sum_{s=1}^{t-1} \boldsymbol{z}_{s-1} \boldsymbol{z}_{s-1}^{\top}, \quad \boldsymbol{\zeta}_t = \sum_{s=1}^{t-1} \boldsymbol{z}_{s-1} q_{s}, \quad \boldsymbol{z}_s = (1,q_s),\quad q_0 = 0.
\end{equation*}

Now we provide the following prediction oracle of $Q$:
\begin{equation}
    \hat{Q}_t = \sum_{s=1}^{t-1} q_s + \frac{\hat{\beta}_t - \hat{\beta}_t^{T-t+1}}{1 - \hat{\beta}_t} q_{t-1} + \hat{\phi}_t \left(T-t+1 - \hat{\beta}_t + \hat{\beta}_t^{T-t+2} \right),
    \label{eq-timeseries-ar1-prediction}
\end{equation}
where $\hat{\phi}_t = \frac{\hat{\alpha}_t}{1-\hat{\beta}_t}$. The following lemma provides the theoretical guarantee for the accuracy of the prediction oracle:

\begin{lemma}
Consider the AR(1) model (\ref{eq-timeseris-ar1}) with prediction oracle $\hat{Q}_t$ as (\ref{eq-timeseries-ar1-prediction}) on $Q = \sum_{t=1}^T q_t$. Assume the parameters $\alpha$, $\beta$, $\sigma$, $\delta$, $q_1$ satisfies the following condition:
\begin{equation}
    \min \left \{q_1,\frac{\alpha}{1 - \beta} \right \} > \sigma \sqrt{\frac{2}{1 - \beta^2} \log \left(\frac{2}{\delta} \right)},
    \label{eq-timeseries-ar1-assumption}
\end{equation}
then with probability $\ge 1 - 3 T \delta$, $\underline{q} \le q_t \le \overline{q}$ for all $t$, and
\begin{equation*}
    \left|\hat{Q}_t - Q \right| \le O \left(\frac{T-t+1}{\sqrt{t-1}} \cdot \frac{\alpha A}{\underline{q}(1-M)^2} + \sqrt{T-t+1} \cdot \frac{\sigma}{1 - \beta} \sqrt{\log \left(\frac{1}{\delta} \right)} \right) : = \epsilon_t,
\end{equation*}
where
\begin{equation*}
    \overline{q} = \max \left \{q_1,\frac{\alpha}{1 - \beta} \right \} + \sigma \sqrt{\frac{2}{1 - \beta^2} \log \left(\frac{2}{\delta} \right)}, \quad \underline{q} = \min \left \{q_1,\frac{\alpha}{1 - \beta} \right \} - \sigma \sqrt{\frac{2}{1 - \beta^2} \log \left(\frac{2}{\delta} \right)},
\end{equation*}
and
\begin{equation*}
    A = O \left(\sqrt{\log \left(\frac{1}{\delta} \right)} + \sqrt{\log \left(T \right)} \right).
\end{equation*}
\label{lem-timeseries-accuracy-ar1}
\end{lemma}

Lemma \ref{lem-timeseries-accuracy-ar1} is proved in Appendix F. The assumption (\ref{eq-timeseries-ar1-assumption}) is to assure the positivity of $q_t$ with high probability. Combine Theorem \ref{thm-upperbound-1}, the high-probability regret upper bound under the AR(1) model is provided as follow:

\begin{theorem}
Consider the model under AR(1) demand model (\ref{eq-timeseris-ar1}), that is provided with predictions $\hat{Q}_t$ as (\ref{eq-timeseries-ar1-prediction}) for all $t\in [T]$. Assume the parameters $\alpha$, $\beta$, $\sigma$, $\delta$, $q_1$ satisfies the following condition:
    \begin{equation*}
    \min \left \{q_1,\frac{\alpha}{1 - \beta} \right \} > \sigma \sqrt{\frac{2}{1 - \beta^2} \log \left(\frac{2}{\delta} \right)},
\end{equation*}
then with probability $\ge 1 - 3 T\delta$, $\underline{q} \le q_t \le \overline{q}$ for all $t$, and OA-UCB satisfies
\begin{equation*}
    \left(\frac{1}{Q} + \frac{1}{B} \right) \sum_{t=1}^{\tau - 1} q_t \epsilon_t = O \left(\Gamma \sqrt{\tau - 1} \right),
\end{equation*}
where
\begin{equation*}
    \Gamma = \left(\frac{1}{\underline{q}} + \frac{1}{b} \right) \left(\frac{\alpha}{\underline{q}(1-M)^2} \left( \sqrt{\log \left(\frac{1}{\delta} \right)} + \sqrt{\log \left(T \right)}  \right)+ \frac{\sigma}{1 - \beta} \sqrt{\log \left(\frac{1}{\delta} \right)} \right),
\end{equation*}
and
\begin{equation*}
    \overline{q} = \max \left \{q_1,\frac{\alpha}{1 - \beta} \right \} + \sigma \sqrt{\frac{2}{1 - \beta^2} \log \left(\frac{2}{\delta} \right)}, \quad \underline{q} = \min \left \{q_1,\frac{\alpha}{1 - \beta} \right \} - \sigma \sqrt{\frac{2}{1 - \beta^2} \log \left(\frac{2}{\delta} \right)}.
\end{equation*}
Thus, with probability $\ge 1 - 3 KTd\delta - 3T \delta$, the OA-UCB algorithm induces regret
\begin{equation*}
        \text{OPT}_{\text{LP}} - \sum_{t=1}^{\tau - 1} q_t R_t \le \Tilde{O} \left( \text{OPT}_{\text{LP}} \sqrt{\frac{\overline{q} K}{B}} + 
    \sqrt{\overline{q} K \text{OPT}_{\text{LP}} } + \left(\Gamma + \overline{q} \right) \sqrt{\tau - 1} \right).
    \end{equation*}
\label{thm-upperbound-ar1}
\end{theorem}

Theorem \ref{thm-upperbound-ar1} is proved in Appendix. Clearly in this case, the regret lower and upper bound are nearly matched except for a $\Tilde{O} \left( (\Gamma + \overline{q})\sqrt{\tau - 1} \right)$ term, which shows the near-optimality of our algorithm.

\section{Numerical Experiments}\label{sec:num}

In this section, we provide numerical results, and compare our algorithm with several existing benchmarks in literature. In the first part, the experiments are in the context of general BwK problem with only 1 resource, i.e. $d=1$. In the second part, we illustrate an application of our algorithm in a network revenue management problem with unknown demand distribution, similar as \cite{agrawal2019bandits}. 

\textbf{Demand sequence and Prediction oracle: }We apply AR(1) demand model to generate the demand sequence $\{q_t\}$ as introduced in the previous section. To achieve time-efficiency, we consider an alternative "power-of-two" policy for updating the prediction $\left\{\hat{Q}_t \right\}$: We only recompute $\hat{Q}_t$ as (\ref{eq-timeseries-ar1-prediction}) when $t = 2^k$ for some $k \in \mathbb{N}^+$, as shown in Algorithm \ref{alg-estimation}. The estimation error of Algorithm \ref{alg-estimation} in terms of additive gap is plotted in Figure \ref{fig-estimationerror}.

\begin{algorithm}[htb]
	\caption{Estimation Generation Policy}
    \hspace*{0.02in}{\bf Input:} Time step $t$, history observation $\{q_s\}_{s=1}^{t-1}$, previous estimation $\hat{Q}_{t-1}$. 
	\begin{algorithmic}[1]
	    \If{$t = 2^k$ for some $k \in \mathbb{N}^+$}
        \State Compute $\hat{Q}_t$ as (\ref{eq-timeseries-ar1-prediction}).
        \Else
        \State Set $\hat{Q}_t = \hat{Q}_{t-1}$.
        \EndIf
        \State \textbf{return} $\hat{Q}_t$.
	\end{algorithmic}
	\label{alg-estimation}
\end{algorithm}

\begin{figure}[htb]
\vskip -0.1in
\begin{center}
\centerline{\includegraphics[width=.5\columnwidth]{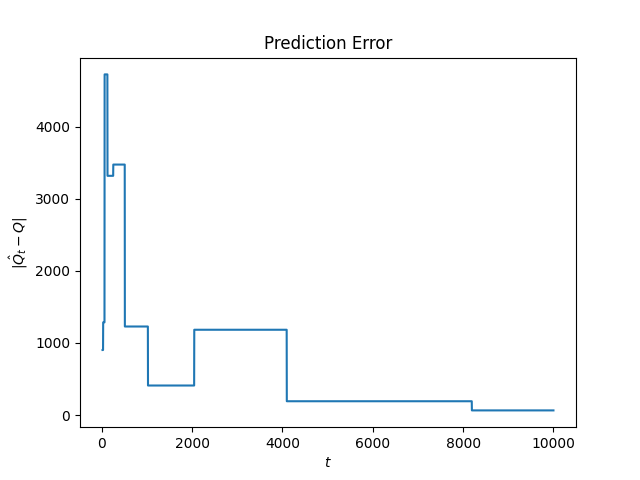}}
\vskip -0.15in
\caption{Estimation error}
\label{fig-estimationerror}
\end{center}
\vskip -0.3in
\end{figure}

\subsection{Bandits with Knapsack Problem with 1 resource}

\textbf{Benchmarks}: Our model can be viewed as a kind of "Non-stationary BwK" problem, if we view our "action" as "arm". Hence, we would like to compare our algorithm with other existing online algorithms design for BwK:

\begin{enumerate}
    \item PDB: The PrimalDualBwK algorithm in \cite{badanidiyuru2013bandits} (Algorithm 2), which is designed to solve stochastic BwK problems and has been shown to achieve optimality.
    \item AD-UCB: The UCB algorithm presented in  \cite{agrawal2014bandits} (Algorithm 2), which is designed to solve stochastic BwK problems and has been shown to achieve optimality.
    \item SW-UCB: The Sliding-Window UCB in \cite{liu2022non}, designed to solve a non-stationary BwK problems, where the amount of non-stationarity is limited by a variational budget. In implementing SW-UCB, we set the sliding window size according to the suggestion in \cite{liu2022non}, and we input the required non-stationarity measures by computing them from the ground truth $\{q_t\}^{T}_{t=1}$.
    \item EXP3++: The EXP3++.BwK in \cite{rangi2018unifying}, working in both stochastic and adversarial BwK setting in 1 resource case. 
\end{enumerate}
In the experiment, we simulate our algorithm and benchmarks on a family of instances, with $K=4$, $d=1$, $\alpha = 12$, $\beta = 0.5$, $\sigma=2$, $\boldsymbol{r} = (1,0.8,0.5,0.3)$, $\boldsymbol{c} = (0.95,0.7,0.4,0.2)$, and varies $b$ and $T$. Each arms's per-unit-demand outcome $(R_t,\{C_{t,i}\}^d_{i=1})$ follows the standard Gaussian distribution truncated in $[0,1]^{d+1}$, which has mean denoted as $(\boldsymbol{r},\boldsymbol{c})$. We perform two groups of the experiment: In the first group, we fix time horizon $T = 10000$ and varies inventory level $b = 10, 15, 20$; In the second group, we fix inventory level $b = 15$ and varies horizon length $T$ from 5000 to 20000. Then, for each fixed $(b,T)$, we simulate each algorithm 100 times with demand volume sequence $\{q_t\}_{t=1}^T$, and compute the regret and competitve ratio (CR) based on the sample average and standard error, where the competitive ratio is computed as the ratio of the total cumulative reward $\sum_{t=1}^{\tau-1} q_t r_t$ obtained by the algorithm to the offline benchmark $\text{OPT}_{\text{LP}}$.

Table \ref{table-regretcr-com-b} compares the regret/competitive ratio of each algorithm on different inventory level, where the average $\pm$ standard error is reported and the best performed algorithm in each setting has been bold; Figure \ref{fig-regret-com} and \ref{fig-cr-com} plot the regret/competitive ratio of each algorithm on different horizon lengths, respectively. The superiority in numerical performance for OA-UCB does not mean that our algorithm is strictly superior to the baselines. Indeed, our algorithm OA-UCB receives online advice, while the benchmarks do not. The numerical results instead indicate the benefit of prediciting the underlying non-stationary demand sequence, and showcase how a suitably designed algoirhtm such as OA-UCB could reap the benefit of predictions. In addition, we remark that the sliding window UCB algorithm proposed by \cite{liu2022non} is designed to handle arbitrarily changing mean outcome distributions, subject to constraints on the amount of temporal variations. On the one hand, the sliding window UCB algorithm has been shown to outperform the stationary BwK benchmarks in piece-wise stationary models where the mean outcome distributions change abruptly \citep{liu2022non}. On the other hand, in our non-stationary scaling setting,  the per-demand-unit mean outcomes $r(a), c(a, 1), \ldots, c(a, d)$ are time stationary for each arm $a$. Hence, historical data are useful for estimating these mean outcomes, which explain why the stationary benchmark could appear to out-perform sliding window UCB.

\begin{table}[htb]
\centering
\begin{tabular}{|cccccc|}
\hline
       & OA-UCB                    & PDB                & AD-UCB         & SW-UCB             & EXP3++            \\ \hline
\multicolumn{6}{|l|}{$b = 10$}                                                                                         \\
Regret & \textbf{4807 $\pm$ 313} & 42874 $\pm$ 382  & 20298 $\pm$ 151 & 51992 $\pm$ 206  & 8161 $\pm$ 327 \\
CR     & \textbf{0.961 $\pm$ 0.003}    & 0.654 $\pm$ 0.003      & 0.836 $\pm$ 0.003     & 0.581 $\pm$ 0.002      & 0.934 $\pm$ 0.003     \\ \hline
\multicolumn{6}{|l|}{$b = 15$}                                                                                         \\
Regret & \textbf{7034 $\pm$ 375} & 92944 $\pm$ 431  & 18725 $\pm$ 183 & 101916 $\pm$ 196 & 29368 $\pm$ 526 \\
CR     & \textbf{0.960 $\pm$ 0.002}    & 0.466 $\pm$ 0.002      & 0.892 $\pm$ 0.001     & 0.414 $\pm$ 0.001      & 0.831 $\pm$ 0.003     \\ \hline
\multicolumn{6}{|l|}{$b = 20$}                                                                                         \\
Regret & \textbf{9314 $\pm$ 182} & 136547 $\pm$ 502 & 10658 $\pm$ 215 & 145539 $\pm$ 165 & 72962 $\pm$ 460 \\
CR     & \textbf{0.957 $\pm$ 0.001}    & 0.372 $\pm$ 0.002      & 0.951 $\pm$ 0.001     & 0.331 $\pm$ 0.001      & 0.665 $\pm$ 0.002     \\ \hline
\end{tabular}
\caption{Regret/Competitive Ratio Comparison among algorithms varying inventory level $b$ under BwK setting.}
\label{table-regretcr-com-b}
\end{table}

\begin{figure}[htb]
    \centering
    \begin{subfigure}{.48\textwidth}
		\centering
	    \includegraphics[width=1\textwidth]{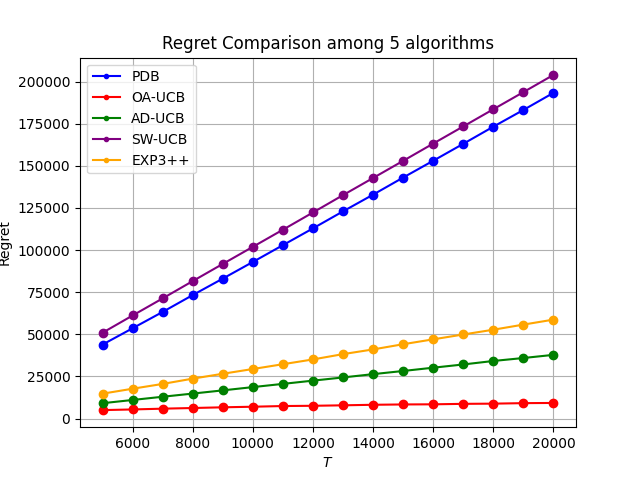}
	    \caption{Regret}
	    \label{fig-regret-com}
	\end{subfigure}
	\begin{subfigure}{.48\textwidth}
		\centering
		\includegraphics[width=1\textwidth]{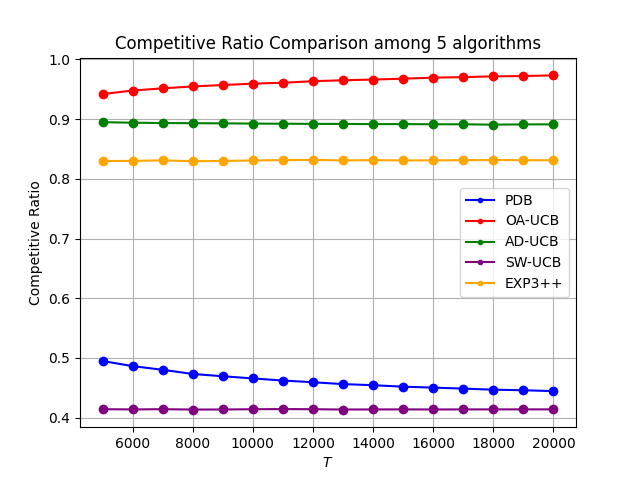}
		\caption{Competitive Ratio}
		\label{fig-cr-com}
	\end{subfigure}
    \caption{Regret/Competitive Ratio Comparison among algorithms varying time horizon $T$}
\end{figure}

In order to better explain the benefit of receiving a "good" prediction further, we conduct two additional groups of experiment. Table \ref{table-prediction-com-incre} and Table \ref{table-prediction-com-decre} compare the regret and competitive ratio achieved by our algorithm OA-UCB under different choices of prediction oracles, as shown in the first row in Tables \ref{table-prediction-com-incre}, \ref{table-prediction-com-decre}. The ALG \ref{alg-estimation} is the prediction we provide in Algorithm \ref{alg-estimation}, which we have shown that the accuracy is "Good", in the sense that $|\hat{Q}_t - Q|$ tends to 0 as $t$ increases. "$xT$" is the static prediction oracle which returns $\hat{Q}_t = Q + x T$ for all $t\in \{1, \ldots, T\}$. We vary $x$ over the values in $\{5,10,15,20\}$ in Table \ref{table-prediction-com-incre}, and over the values in $\{-5,-10,-15,-20\}$ in Table \ref{table-prediction-com-decre}. In "$xT$" where $\hat{Q}_t = Q + x T$ for all $t$, we have $\epsilon_t = |x| T$ for all $t$. The instance and experiment setting are the same as Table \ref{table-regretcr-com-b}. 

Table \ref{table-prediction-com-incre} compares the ALG \ref{alg-estimation} and 4 "aggressive" prediction oracles, i.e. $x > 0$ and $\hat{Q}_t > Q$, showing that when $|x|$ becomes larger, i.e. the prediction becomes worse, the performance of OA-UCB becomes worse, especially in the small inventory case of $b=10$. However, note that the ALG \ref{alg-estimation} may not be the best in small inventory case. This is because in such cases, the algorithm stops too early to allow the ALG \ref{alg-estimation} to produces a good prediction. Indeed, Figure \ref{fig-estimationerror} shows that in early stage, the proposed ALG \ref{alg-estimation} might not have produced a good prediction yet. On the contrary, Table \ref{table-prediction-com-decre} compares the ALG \ref{alg-estimation} with 4 "conservative" prediction oracles, i.e. $x < 0$ and $\hat{Q}_t < Q$, demonstrating the same conclusion as Table \ref{table-prediction-com-incre} on how the performance of OA-UCB worsens when $|x|$ grows. We believe both Table \ref{table-prediction-com-incre} and Table \ref{table-prediction-com-decre} verify the value of prediction: A better prediction oracle can help OA-UCB improve its performance.

\begin{table}[htb]
\centering
\begin{tabular}{|cccccc|}
\hline
       & ALG \ref{alg-estimation}        &  "$+5T$"            &  "$+10T$"   & "$+15T$"   &  "$+20T$"   \\ \hline
\multicolumn{6}{|l|}{$b = 10$}                                                                                               \\
Regret & $4917 \pm 337$          & \textbf{1985 $\pm$ 383} & $14769 \pm 396$ & 24847 $\pm$ 355 & 32860 $\pm$ 382 \\
CR     & 0.960 $\pm$ 0.003             & \textbf{0.984 $\pm$ 0.003}    & 0.881 $\pm$ 0.003     & 0.800 $\pm$ 0.003     & 0.735 $\pm$ 0.003     \\ \hline
\multicolumn{6}{|l|}{$b = 15$}                                                                                               \\
Regret & \textbf{6969 $\pm$ 377} & 9897 $\pm$ 365          & 28085 $\pm$ 393 & 41924 $\pm$ 350 & 53055 $\pm$ 422 \\
CR     & \textbf{0.960 $\pm$ 0.002}    & 0.943 $\pm$ 0.002             & 0.839 $\pm$ 0.002     & 0.759 $\pm$ 0.002     & 0.695 $\pm$ 0.002     \\ \hline
\multicolumn{6}{|l|}{$b = 20$}                                                                                               \\
Regret & \textbf{9330 $\pm$ 239} & 17327 $\pm$ 372         & 37584 $\pm$ 351 & 54635 $\pm$ 383 & 69521 $\pm$ 383 \\
CR     & \textbf{0.957 $\pm$ 0.001}    & 0.920 $\pm$ 0.002             & 0.827 $\pm$ 0.002     & 0.749 $\pm$ 0.002     & 0.685 $\pm$ 0.002     \\ \hline
\end{tabular}
\caption{Comparison of Regret/Competitive Ratios achieved by OA-UCB with 4 static "aggressive" prediction oracles with $x \in \{5,10,15,20\}$.}
\label{table-prediction-com-incre}
\end{table}

\begin{table}[htb]
\centering
\begin{tabular}{|cccccc|}
\hline
       & ALG \ref{alg-estimation}       & "$-5T$"    & "$-10T$"   & "$-15T$"   & "$-20T$"   \\ \hline
\multicolumn{6}{|l|}{$b = 10$}                                                                                       \\
Regret & \textbf{4898 $\pm$ 339} & 9537 $\pm$ 324  & 17238 $\pm$ 178 & 20317 $\pm$ 167 & 20339 $\pm$ 167 \\
CR     & \textbf{0.960 $\pm$ 0.003}    & 0.923 $\pm$ 0.003     & 0.861 $\pm$ 0.001     & 0.836 $\pm$ 0.001     & 0.836 $\pm$ 0.001     \\ \hline
\multicolumn{6}{|l|}{$b = 15$}                                                                                       \\
Regret & \textbf{6930 $\pm$ 395} & 16454 $\pm$ 202 & 18735 $\pm$ 192 & 18773 $\pm$ 193 & 18700 $\pm$ 194 \\
CR     & \textbf{0.960 $\pm$ 0.002}    & 0.905 $\pm$ 0.001     & 0.892 $\pm$ 0.001     & 0.892 $\pm$ 0.001     & 0.893 $\pm$ 0.001     \\ \hline
\multicolumn{6}{|l|}{$b = 20$}                                                                                       \\
Regret & \textbf{9305 $\pm$ 206} & 10752 $\pm$ 211 & 10709 $\pm$ 225 & 10640 $\pm$ 224 & 10749 $\pm$ 241 \\
CR     & \textbf{0.957 $\pm$ 0.001}    & 0.951 $\pm$ 0.001     & 0.951 $\pm$ 0.001     & 0.951 $\pm$ 0.001     & 0.951 $\pm$ 0.001     \\ \hline
\end{tabular}
\caption{Comparison of Regret/Competitive Ratios achieved by OA-UCB with 4 static "conservative" prediction oracles with $x \in \{-5,-10,-15,-20\}$.}
\label{table-prediction-com-decre}
\end{table}

\subsection{A Dynamic Pricing Example: Network Revenue Management Model}

We consider a network revenue management model that is similar to the models in \cite{agrawal2019bandits}, \cite{ferreira2018online}, \cite{besbes2012blind}. In this model, a firm needs to determine the prices of $J$ products produced by $d$ resources within a finite time horizon of length $T$. A resource consumption matrix $A$ is provided to the firm, where the $(j,i)^{\text{th}}$ entry represents the amount of resource $i$ needed to produce one unit of product $j$. The inventory level of each resource is $B_i=B=bT$, which cannot be replenished during the horizon. The firm is given a set $\mathcal{A}$ of $K$ price vectors, where the $j^{\text{th}}$ component of a vector specifies the price for the $j^{\text{th}}$ product. At each time $t$, the firm needs to choose one out of these $K$ price vectors. After a price $\boldsymbol{p}_t$ is chosen, the demand for each product $j$, $D_{t,j}$, is generated independently as $\sum_{\ell=1}^{q_t} X_{\ell,j}$. Here $q_t$ can be viewed as the "market size" at time $t$, or more specifically, $q_t$ is the number of customers entering the system at this timestamp. Each $X_{\ell,j}$ is an independent Bernoulli random variable, referring to the binary choice of buying/not buying product $j$ for the $\ell^{\text{th}}$ customer. The mean of $X_{\ell,j}$ is $\boldsymbol{\lambda}(\boldsymbol{p}_t)_j$, where $\boldsymbol{\lambda}(\cdot):\mathbb{R}^J \rightarrow [0,1]^J$ is an unknown vector value function, representing the individual choice model. The revenue at time $t$ is the sum of prices of products sold, i.e., $\sum_{j=1}^J p_{t,j} D_{t,j}$, and the resource consumption for resource $i$ at time $t$ is the sum of the amount of resources consumed to produce these products, that is, $\sum_{j=1}^J A_{j,i} D_{t,j}$. If some resource is depleted, then the selling horizon is ended, and the stopping time is denoted as $\tau$. The objective for the firm is to maximize the total revenue while satisfying the resource constraints, i.e. $\sum_{t=1}^{\tau - 1}\sum_{j=1}^J p_{t,j} D_{t,j}$.

This problem can be formulated using our model: Each available price vector refers to one action. The $R_t$, $C_{t,i}$, $r(\boldsymbol{p}_t)$, $c(\boldsymbol{p}_t,i)$ can be formulated as following (Suppose at time $t$ $\boldsymbol{p}_t$ is chosen) 
\begin{equation*}
    \begin{aligned}
        & R_t = \frac{1}{q_t} \sum_{j=1}^J p_{j} D_{t,j},\quad C_{t,i} = \frac{1}{q_t} \sum_{j=1}^J A_{j,i} D_{t,j},\\
        & r(\boldsymbol{p}_t) = \mathbb{E}[R_t] = \boldsymbol{p}_t^{\top} \boldsymbol{\lambda}(\boldsymbol{p}_t),\quad c(\boldsymbol{p}_t,i) = \mathbb{E}[C_{t,i}] = \left(A \cdot \boldsymbol{\lambda}(\boldsymbol{p}_t) \right)_i.
    \end{aligned}
\end{equation*}
At each time $t$, after receiving prediction $\hat{Q}_t$, the firm chooses price $\boldsymbol{p}_t \in \mathcal{A}$, observing the feedback of (i) demand volume $q_t$, (ii) reward $q_t R_t = \sum_{j=1}^J p_{t,j} D_{t,j}$ and (iii) resources $i$ consumed $q_t C_{t,i} = \sum_{j=1}^J A_{j,i} D_{t,j}$. If some resource is depleted, i.e. $\exists i_0 \in [d]$ such that $\sum_{s=1}^t q_s C_{s,j} = \sum_{s=1}^t  \sum_{j=1}^J A_{j,i_0} D_{s,j} > B_{i_0}$, then an null price vector $\boldsymbol{p}_{\infty} = (\infty,...,\infty)$ (Clearly in this case the demand will be 0 with certainty) is to be choosed in the remaining horizon $t+1, \ldots, T$. The firm aims to maximize the total reward $\sum_{t=1}^{\tau - 1} q_t R_t = \sum_{t=1}^{\tau - 1}\sum_{j=1}^J p_{t,j} D_{t,j}$, subject to the resource constraints and model uncertainty.

Now we describe how we adapt our OA-UCB to this application. 
The $(R_t;C_{t,1},\ldots,C_{t,d})$ can be viewed as the random outcomes sampled independently from some unknown distribution $P_{\boldsymbol{p}_t,t}$. The latent distribution $P_{\boldsymbol{p},t}$ for price vector $\boldsymbol{p}$ varies over time and shares a common support and mean. However, the support may not be $[0,1]^{d+1}$ anymore, depending on the resource consumption matrix $A$ and price vector set $\mathcal{A}$. Nonetheless, the random variables following $P_{\boldsymbol{p},t}$ are still bounded and non-negative, which implies their sub-gaussianity and makes the confidence bound technique applicable.
Additionally, note that the model uncertainty arises from the underlying demand function $\boldsymbol{\lambda}(\boldsymbol{p})$ for all available $\boldsymbol{p}$, rather than the direct latent reward and resource consumption distribution. Therefore, instead of constructing UCB/LCB separately for $r$ and $c$, we directly construct UCB/LCB for the unknown demand function per unit of customer, i.e. using $\left \{\frac{D_{s,j}}{q_s} \boldsymbol{1}(\boldsymbol{p}_s = \boldsymbol{p}) \right \}_{s=1}^{t-1}$ to estimate $\boldsymbol{\lambda}(\boldsymbol{p})_j$ for price $\boldsymbol{p}$ and product $j$, as shown in Algorithm \ref{alg-OAUdp}.
Here, the value $\frac{D_{t,j}}{q_t}$ for all $j$ is a sequence of random variables with the same mean $\boldsymbol{\lambda}(\boldsymbol{p})_j$ and support $[0,1]$ throughout the horizon, satisfying the condition for using our confidence radius.

\begin{algorithm}
	\caption{OA-UCB for Dynamic Pricing (OA-UCB-DP)}
	\begin{algorithmic}[1]
	    \State \textbf{Intialize} $\boldsymbol{\mu}_1 = \frac{1}{d+1} \boldsymbol{1} = \left(\frac{1}{d+1},\cdots,\frac{1}{d+1}\right) \in \mathbb{R}^{d+1} $, $\eta_1 = 0$, $\boldsymbol{\theta}_1 = \boldsymbol{0} \in \mathbb{R}^{d+1}$, $\kappa = \sqrt{\ln (d+1)} $, $\boldsymbol{B} = (B_i)_{i \in [d]} = (B,...,B)$, $\text{Reward} = 0$.
        \For{$t=1,2,...,T$}
        \State Receive $\hat{Q}_t = {\cal F}_t(q_1, \ldots, q_{t-1})$.
        \State Compute $\text{UCB}_{t,j}(\boldsymbol{p})$, $\textbf{LCB}_{t,j}(\boldsymbol{p})$ for all $\boldsymbol{p} \in \mathcal{A}$ as following:
        \begin{equation*}
            \begin{aligned}
                & \hat{D}_{t,j}(\boldsymbol{p}) = \frac{1}{N_{t-1}(\boldsymbol{p})} \sum_{s=1}^{t-1} \frac{D_{s,j} \boldsymbol{1}(\boldsymbol{p}_s = \boldsymbol{p})}{q_s}, \\
                & \text{UCB}_{t,j}(\boldsymbol{p}) = \hat{D}_{t,j}(\boldsymbol{p}) + \text{rad}(\hat{D}_{t,j}(\boldsymbol{p}),N_{t-1}(\boldsymbol{p}),\delta), \\
                & \text{LCB}_{t,j}(\boldsymbol{p}) = \hat{D}_{t,j}(\boldsymbol{p}) - \text{rad}(\hat{D}_{t,j}(\boldsymbol{p}),N_{t-1}(\boldsymbol{p}),\delta).
            \end{aligned}
        \end{equation*}
        \State Denote $\textbf{UCB}_{t}(\boldsymbol{p}) = (\textbf{UCB}_{t,j}(\boldsymbol{p}))_{j \in [J]}$, $\textbf{LCB}_{t}(\boldsymbol{p}) = (\textbf{LCB}_{t,j}(\boldsymbol{p}))_{j \in [J]}$. Select
        \begin{equation*}
            \boldsymbol{p}_t \in \underset{\boldsymbol{p} \in \mathcal{A}}{\text{argmax}} \left\{ \textbf{UCB}_{t}(\boldsymbol{p})^{\top} \boldsymbol{p} - \frac{\hat{Q}_t}{B} \cdot  \boldsymbol{\mu}_t^\top \cdot A \cdot \textbf{LCB}_{t}(\boldsymbol{p}) \right\}.
        \end{equation*}
        \State Observe market size $q_t$ and demand $\{D_{t,j}\}_{j \in [J]}$.
        \If{$\exists i \in [d]$ such that $\sum_{j=1}^J A_{j,i} D_{t,j} > B_i$,}
        \State Break, and select the null price $\boldsymbol{p}_{\infty}$ all the way.
        \EndIf
        \State Update
        \begin{equation*}
            \text{Reward} = \text{Reward} + \sum_{j=1}^J p_{t,j} D_{t,j},\quad B_i = B_i - \sum_{j=1}^J A_{j,i} D_{t,j}, \ i \in [d].
        \end{equation*}
        \State Set
        \begin{align}
        & \boldsymbol{g}_t = \frac{q_t \hat{Q}_t}{B}\left(\frac{B}{\hat{Q}_t} \boldsymbol{\beta} - (A \cdot \textbf{LCB}_{t}(\boldsymbol{p}))_i\right), \nonumber \\
        & \rho_t = \begin{cases}
                \boldsymbol{g}_1^{\top} \boldsymbol{\mu}_1  - \min_{j=1,...,d+1} g_{1,j} & t = 1\\ 
                \eta_t \ln \left(\sum_{j=1}^d \mu_{t,j} \exp\left(\frac{-g_{t,j}}{\eta_t} \right) \right) + \boldsymbol{g}_t^{\top} \boldsymbol{\eta}_t & \text{otherwise}.
            \end{cases} \nonumber
        \end{align}
        \State Update
        \begin{align}
            & \boldsymbol{\theta}_{t+1} = \boldsymbol{\theta}_{t} - \boldsymbol{g}_t, \nonumber \\
            & \eta_{t+1} = \eta_t + \frac{1}{\kappa^2} \rho_t, \nonumber \\
            & \mu_{t+1,j} = \frac{\exp \left(\frac{\theta_{t+1,j}}{\eta_{t+1}} \right)}{\sum_{i=1}^{d+1} \exp \left(\frac{\theta_{t+1,i}}{\eta_{t+1}} \right)}, \ j \in [d+1]. \nonumber
        \end{align}
        \EndFor
	\end{algorithmic}
	\label{alg-OAUdp}
\end{algorithm}

\textbf{Benchmarks}: We compare the performance of our OA-UCB-DP with the following four dynamic pricing algorithms:
\begin{enumerate}
    \item AD-UCB-DP: The AD-UCB algorithm applied to dynamic pricing setting presented in  \cite{agrawal2019bandits} (Algorithm 6).
    \item TS-fixed: The algorithm presented in \cite{ferreira2018online} (Algorithm 1), which applies Thompson Sampling technique to estimate the latent demand and uses the estimation to get the pricing strategy by solving an LP in each time step. 
    \item TS-update: The algorithm presented in \cite{ferreira2018online} (Algorithm 2), which is same as TS-fixed except the adaptive inventory level in its LP. 
    \item BZ12: The algorithm proposed in \cite{besbes2012blind} (Algorithm 1), which first explores all prices and then exploits the best pricing strategy by solving a linear program once. In implementating this algorithm, we divide the exploration and exploitation phases at period $\tau = T^{2/3}$, as suggested in their paper.
\end{enumerate}

\subsubsection{Single Product, Single Resource Example}

We first consider a special case. The firm only sells a single product ($J=1$) in a finite time horizon. Without loss of generality, we assume that the product is itself the resource ($d=1$) which has limited inventory. There are $K=6$ possible prices: $\mathcal{A} = \{\$10,\$11,\$13,\$15,\$17,\$19\}$, and the mean demand function is given by $\lambda(\$ 10) = 1$, $\lambda(\$ 11) = 0.9$, $\lambda(\$ 13) = 0.7$, $\lambda(\$ 15) = 0.5$, $\lambda(\$ 17) = 0.3$, $\lambda(\$ 19) = 0.1$. The $\alpha$, $\beta$, $\sigma$ is the same as before. We fix $T = 10000$ and vary $b = 10,15,20$. For each fixed $b$, we run each algorithm 100 times and report sample average and standard error of regret and competitive ratio. The result is shown in Table \ref{table-regretcr-com-b-dp-single}. The result clearly demonstrates the benefit of our framework in such single-product, single-resource problem.

\begin{table}[htb]
\centering
\begin{tabular}{|cccccc|}
\hline
       & OA-UCB-DP                      & AD-UCB-DP             & TS-fixed           & TS-update          & BZ12                \\ \hline
\multicolumn{6}{|l|}{$b = 10$}                                                                                              \\
Regret & \textbf{394034 $\pm$ 2820} & 429123 $\pm$ 1627 & 551735 $\pm$ 1903 & 523617 $\pm$ 206 & 528413 $\pm$ 34319 \\
CR     & \textbf{0.747 $\pm$ 0.002}      & 0.725 $\pm$ 0.001      & 0.646 $\pm$ 0.001      & 0.664 $\pm$ 0.002      & 0.661 $\pm$ 0.022       \\ \hline
\multicolumn{6}{|l|}{$b = 15$}                                                                                              \\
Regret & \textbf{331294 $\pm$ 2710} & 355020 $\pm$ 2128 & 530342 $\pm$ 1957 & 512541 $\pm$ 1976 & 511933 $\pm$ 36609 \\
CR     & \textbf{0.837 $\pm$ 0.001}      & 0.826 $\pm$ 0.001      & 0.740 $\pm$ 0.001      & 0.748 $\pm$ 0.001      & 0.749 $\pm$ 0.018       \\ \hline
\multicolumn{6}{|l|}{$b = 20$}                                                                                              \\
Regret & \textbf{73263 $\pm$ 1969}  & 73276 $\pm$ 1931 & 299816 $\pm$ 1788 & 292063 $\pm$ 2213 & 266875 $\pm$ 65152 \\
CR     & \textbf{0.968 $\pm$ 0.001}      & 0.968 $\pm$ 0.001      & 0.870 $\pm$ 0.001      & 0.873 $\pm$ 0.001      & 0.665 $\pm$ 0.028       \\ \hline
\end{tabular}
\caption{Regret/Competitive Ratio Comparison among algorithms varying inventory level $b$ under Single Product, Single Resource Dynamic Pricing setting.}
\label{table-regretcr-com-b-dp-single}
\end{table}


\subsubsection{Multi-Product, Multi-Resource Example}

Now we consider a more general setting with multi-product and multi-resource. Similar as in \cite{besbes2012blind,ferreira2018online,agrawal2019bandits}, there are two products and three resources ($d = 3, J=2$). There are $K=6$ possible prices: 
\begin{equation*}
    \mathcal{A} = \{(\$5,\$10),(\$6,\$11),(\$6,\$13),(\$7,\$15),(\$8,\$17),(\$9,\$19)\},
\end{equation*}
where $\boldsymbol{p} = (p_1,p_2) \in \mathcal{A}$ means that the price for the first product is $p_1$ and the second product is $p_2$. The matrix $A$ equals to
\begin{equation*}
    A = \begin{pmatrix}
        1 & 1\\
        3 & 1 \\
        1 & 4
    \end{pmatrix}.
\end{equation*}
We consider three individual choice model:
\begin{itemize}
    \item Linear model: $\lambda(\boldsymbol{p}) = \left(1 - 0.1 p_1, 1 - 0.05 p_2 \right)$.
    \item Exponential model: $\lambda(\boldsymbol{p}) = \left(\exp(-0.2 p_1), \exp(-0.1 p_2) \right)$.
    \item Logit model: $\lambda(\boldsymbol{p}) = \left(\frac{4 \exp(-0.4 p_1)}{1 + \exp (-0.4 p_1) + \exp(-0.2 p_2)}, \frac{4 \exp(-0.2 p_2)}{1 + \exp(-0.4 p_1) + \exp (-0.2 p_2)}\right)$.
\end{itemize}
The $\alpha,\beta,\sigma$ is the same as before. We fix $T = 10000$, and vary $b = 20,30$. The result is shown in Table \ref{table-regretcr-com-b20-dp-multi},\ref{table-regretcr-com-b30-dp-multi}. The result clearly demonstrates that our algorithm can outperform other benchmarks in this setting in most cases.

\begin{table}[htb]
\centering
\begin{tabular}{|cccccc|}
\hline
       & OA-UCB-DP              & AD-UCB-DP     & TS-fixed       & TS-update      & BZ12           \\ \hline
\multicolumn{6}{|l|}{Linear}                                                                       \\
Regret & \textbf{304652 $\pm$ 1985} & 358159 $\pm$ 2365 & 414709 $\pm$ 7525  & 388136 $\pm$ 9057  & 395570 $\pm$ 27243 \\
CR     & \textbf{0.702 $\pm$ 0.002} & 0.650 $\pm$ 0.002 & 0.595 $\pm$ 0.007  & 0.621 $\pm$ 0.009  & 0.614 $\pm$ 0.027  \\ \hline
\multicolumn{6}{|l|}{Expo}                                                                         \\
Regret & \textbf{280236 $\pm$ 2453} & 349069 $\pm$ 4790 & 449530 $\pm$ 19703 & 434440 $\pm$ 17479 & 417928 $\pm$ 59593 \\
CR     & \textbf{0.740 $\pm$ 0.002} & 0.676 $\pm$ 0.004 & 0.583 $\pm$ 0.018  & 0.597 $\pm$ 0.016  & 0.613 $\pm$ 0.055  \\ \hline
\multicolumn{6}{|l|}{Logit}                                                                        \\
Regret & \textbf{251288 $\pm$ 1877} & 303256 $\pm$ 1577 & 333680 $\pm$ 600   & 315206 $\pm$ 897   & 326623 $\pm$ 463   \\
CR     & \textbf{0.732 $\pm$ 0.002} & 0.676 $\pm$ 0.002 & 0.643 $\pm$ 0.001  & 0.664 $\pm$ 0.001  & 0.651 $\pm$ 0.0005 \\ \hline
\end{tabular}
\caption{Regret/Competitive Ratio Comparison among algorithms when $b=20$ under Multi-Product, Multi-Resource setting}
\label{table-regretcr-com-b20-dp-multi}
\end{table}

\begin{table}[htb]
\centering
\begin{tabular}{|cccccc|}
\hline
       & OA-UCB-DP              & AD-UCB-DP              & TS-fixed       & TS-update      & BZ12           \\ \hline
\multicolumn{6}{|l|}{Linear}                                                                                \\
Regret & \textbf{321384 $\pm$ 3440} & 379419 $\pm$ 3546          & 452417 $\pm$ 11519 & 435005 $\pm$ 6713  & 429984 $\pm$ 39892 \\
CR     & \textbf{0.764 $\pm$ 0.002} & 0.722 $\pm$ 0.003          & 0.668 $\pm$ 0.008  & 0.681 $\pm$ 0.005  & 0.685 $\pm$ 0.029  \\ \hline
\multicolumn{6}{|l|}{Expo}                                                                                  \\
Regret & 175335 $\pm$ 6221 & \textbf{174260 $\pm$ 6196} & 314341 $\pm$ 26772 & 302451 $\pm$ 29435 & 265473 $\pm$ 82943 \\
CR     & 0.860 $\pm$ 0.005 & \textbf{0.860 $\pm$ 0.005} & 0.748 $\pm$ 0.021  & 0.758 $\pm$ 0.024  & 0.787 $\pm$ 0.066  \\ \hline
\multicolumn{6}{|l|}{Logit}                                                                                 \\
Regret & \textbf{218885 $\pm$ 2464} & 260040 $\pm$ 1707          & 293944 $\pm$ 601   & 281591 $\pm$ 724   & 286263 $\pm$ 8657  \\
CR     & \textbf{0.817 $\pm$ 0.002} & 0.783 $\pm$ 0.001          & 0.755 $\pm$ 0.001  & 0.765 $\pm$ 0.001  & 0.761 $\pm$ 0.007  \\ \hline
\end{tabular}
\caption{Regret/Competitive Ratio Comparison among algorithms when $b=30$ under Multi-Product, Multi-Resource setting}
\label{table-regretcr-com-b30-dp-multi}
\end{table}

\section{Conclusion}

We study an online resource allocation problem with bandit feedback and time varying demands, focusing on the benefits of online advice in policy design and algorithm performance. Our main contributions are twofold. Firstly, we propose impossibility results that (i) any policy without any advice performs poorly in terms of regret, and (ii) a regret lower bound that depends on the accuracy of advice. With informative advice, a strictly smaller regret lower bound is possible. Second, we design a robust online algorithm that incorporates the online advice in the form of prediction on total demand volume $Q$, which shows outstanding performance when the advice is informative (prediction is correct). Our approach is novel comparing to existing models in online learning/optimization with advice (such as \cite{lykouris2021competitive}), in that ours returns a (possibly refined) prediction on $Q$ in each time step. Our results demonstrate the significance of online advice and the potential for improved performance in resource allocation problems.

There are many interesting future directions, such as investigating the models  \cite{BamasAS20,lykouris2021competitive,PurohitSK18,Mitzenmacher19} in the presense of sequential prediction oracles similar to ours. It is also interesting to invenstigate other forms of predictions, such as predictions with distributional information \cite{BertsimasSZ19,DiakonikolasVCAN21}. Customizing prediction oracles for our model is also an interesting direction to pursuse \cite{AnandGP20}. 


\bibliography{bandit.bib}

\newpage
\appendix
\onecolumn

\section{Proof for Section 2}

\subsection{Proof for Lemma \ref{lem-lpupperbound}}
Let's first consider
\begin{equation}
    \text{OPT}'_{\text{LP}} = \max_{\boldsymbol{x}_t \in \Delta_{K},\ \forall t \in [T] } \  \sum_{t=1}^T q_t \boldsymbol{r}^\top \boldsymbol{x}_t \qquad \text{s.t.} \quad  \sum_{t=1}^T q_t \boldsymbol{c}^{\top} \boldsymbol{x}_t \le B \boldsymbol{1}_d,
    \label{eq-linear-relaxation1}
\end{equation}

It is evident that $\text{OPT}'_{\text{LP}} \geq \text{OPT}$, since for a fixed policy $\pi$ that achieves $\text{OPT}$, the solution $\bar{\boldsymbol{x}} = \{\bar{x}_{t, a}\}_{t\in [T], a\in \mathcal{A}}$ defined as $$\bar{x}_{t, a} = \mathbb{E}[\mathbf{1}(\text{action $a$ is chosen at $t$ under $\pi$})]$$ is feasible to $\text{OPT}'$, and the objective value of $\bar{\boldsymbol{x}}$ in $\text{OPT}'_{\text{LP}}$ is equal to the expected revenue earned in the online process.

Next, we claim that $\text{OPT}_{\text{LP}} =  \text{OPT}'_{\text{LP}}$. Indeed, for each feasible solution $(\boldsymbol{x}_t)_{t \in [T]}$ to $\text{OPT}'_{\text{LP}}$, the solution
\begin{equation*}
    \boldsymbol{u} = \frac{\sum_{t=1}^T q_t \boldsymbol{x}_t}{\sum_{t=1}^T q_t},
\end{equation*}
is feasible to $\text{LP}_{\text{OPT}} $
and has the same objective value as $(\boldsymbol{x}_t)_{t \in [T]}$. Altogether, the Lemma is proved. 

\section{Proofs for Section 3, and Consistency Remarks}
In this section, we provide proofs to the lower bound results. In both proofs, we consider an arbitrary but fixed deterministic online algorithm, that is, conditioned on the realization of the history in $1, \ldots, t-1$ and $q_t$, $\hat{Q}_t$, the chosen action $A_t$ is deterministic. This is without loss of generality, since the case of random online algorithm can be similarly handled by replace the chosen action $A_t$ with a probability distribution over the actions, but we focus on deterministic case to ease the exposition. Lastly, in Section \ref{sec:consistent} we demonstrate that our regret upper and lower bounds are consistent on the lower bounding instances we constructed in Section \ref{sec:pf-thm-lowerbound-estimation}. 
\subsection{Proof for Lemma \ref{lem-lowerbound-oblivious}}\label{sec:pf-lem-lowerbound-oblivious}
Our lower bound example involve two instances $I^{(1)}, I^{(2)}$ with determinstic rewards and deterministic consumption amounts. Both instances involve two non-dummy actions $1,2$ in addition to the null action $a_0$, and there is $d=1$ resource type. Instances $I^{(1)}, I^{(2)}$ differ in their respective seqeunces of demand volumes $\{q_t^{(1)}\}_{t=1}^T, \{q_t^{(2)}\}_{t=1}^T$, but for other parameters are the same in the two instances. 

In both $I^{(1)}, I^{(2)}$, action 1 is associated with (deterministc) reward $r(1) = 1$ and (deterministic) consumption amount $c(1, 1) = 1$, while action 2 is associated with (deterministc) reward $r(2) = 3/4$ and (deterministic) consumption amount $c(2, 1) = 1/2$. Both instances share the same horizon $T$, a positive even integer, and the same capacity $B = T/2$. The sequences of demand volumes $\{q_t^{(1)}\}_{t=1}^T, \{q_t^{(2)}\}_{t=1}^T$ of instances $I^{(1)}, I^{(2)}$ are respectively defined as 
\begin{align*}
    q_t^{(1)} & = \begin{cases}
    1 & ~ \text{if }t\in \{1, \ldots, T/2\},  \\
    1/16 & ~ \text{if }t\in \{T / 2 + 1, \ldots, T\},
    \end{cases}\\
    q_t^{(2)} &= 1,\quad \text{for all }t\in \{1, \ldots, T\}.
\end{align*}
Then the optimal reward for $I^{(1)}$ is at least $\frac{T}{2}$ (always select the action 1 until the resource is fully consumed), and the optimal reward for $I^{(2)}$ is $\frac{3T}{4}$ (always select action 2 until the resource is fully consumed). 

Consider the first $T/2$ rounds, and consider an arbitrary online algorithm that knows $\{P_a\}_{a\in {\cal A}}$, the sequence $\{(q_s, q_s R_s, q_sC_{s, 1},\ldots q_s C_{s, d})\}^{t-1}_{s=1}$ and the time $t$ demand $q_t$ when the action 
$A_t$ is to be chosen at each time $t$. Under this setting, the DM recieves the same set of observations in the first $T/2$ time steps in each of instances $I^{(1)}, I^{(2)}$. Consequently, the sequence of action pulls in the first $T/2$ time steps are the same. Now, we denote $N_a = \sum^{T/2}_{t=1} \mathbf{1}(A_t = a)$ for $a\in \{1, 2\}$. By the previous remark, $N_a$ is the number of times action $a$ is pulled during time steps $1,\ldots, T/2$ in each of the two instances.  Observe that $N_1 + N_2 \le \frac{T}{2}$, which implies $N_1 \le \frac{T}{4}$ or $N_2 \le \frac{T}{4}$. We denote $\text{Reward}_T(I^{(i)}), \text{Regret}_T(I^{(i)})$ as the expected reward and the expected regret of the policy in instance $I^{(i)}$. In what follows, we demonstrate that 
\begin{equation}\label{eq:lb_no_oracle}
\max_{i\in\{1, 2\}} \text{Regret}_T(I^{(i)}) \geq \frac{T}{32},
\end{equation}
which proves the Lemma.
\\ \hspace*{\fill} \\
\textbf{Case 1: $N_1 \le \frac{T}{4}$.} We consider the algorithm on $I^{(1)}$, which earns 
\begin{equation*}
    \text{Reward}_T(I^{(1)}) \le \frac{T}{4} \cdot 1 + \frac{T}{4} \cdot \frac{3}{4} + \frac{T}{2}\frac{1}{16}= \frac{15}{32} T.
\end{equation*}
Hence,
\begin{equation*}
    \text{Regret}_T(I^{(1)}) \ge \frac{T}{2} - \text{Reward}_T(I^{(1)}) \ge \frac{1}{32} T.
\end{equation*}
\textbf{Case 2: $N_2 \le \frac{T}{4}$.} We consider the algorithm on $I^{(2)}$, which earns 
\begin{equation*}
    \text{Reward}_T(I^{(2)}) = \left(\frac{T}{2} - N_2 \right) \cdot 1 + N_2 \cdot \frac{3}{4} + \left(\frac{T}{2} -  \left( \frac{T}{2} - N_2 \right) \cdot 1 - N_2 \cdot \frac{1}{2} \right) \cdot \frac{\frac{3}{4}}{\frac{1}{2}} = \frac{T}{2} + \frac{N_2}{4} \le \frac{9}{16} T.
\end{equation*}
Hence,
\begin{equation*}
    \text{Regret}_T(I^{(2)}) \ge \frac{3}{4} T - \text{Reward}_T(I^{(2)}) \ge \frac{3}{16} T.
\end{equation*}
Altogether, the inequality (\ref{eq:lb_no_oracle}) is shown.

\subsection{Proof for Theorem \ref{thm-lowerbound-estimation}}\label{sec:pf-thm-lowerbound-estimation}
By the Theorem's assumption that $\epsilon_{T_0+1}>0$ is \emph{$(T_0+1, \{q_t\}_{t=1}^{T_0})$-well response} by $\mathcal{F} = \{\mathcal{F}_t\}$, we know that 
\begin{equation}\label{eq:epsilon_sat}
0<\epsilon_{T_0+1} \leq \min\left\{\hat{Q}_{T_0+1} - \sum^{T_0}_{t=1} q_t - \underline{q}(T-T_0), \overline{q} (T-T_0)- \hat{Q}_{T_0+1} - \sum^{T_0}_{t=1}q_t, \frac{\hat{Q}_{T_0+1}}{2}\right\},
\end{equation}
where $\hat{Q}_{T_0+1} = {\cal F}_{T_0+1}(q_1, \ldots, q_{T_0})$.
In what follows, we first construct two deterministic instances $I^{(1)}, I^{(2)}$ which only differ in their respective seqeunces of demand volumes $\{q_t^{(1)}\}_{t=T_0+1}^T, \{q_t^{(2)}\}_{t=T_0+1}^T$, but the two instances are the same on other parameters, and that $q^{(1)}_t = q^{(2)}_t= q_t$ for $t\in \{1, \ldots, T_0\}$. Both  $I^{(1)}, I^{(2)}$ only involve one resource constraint. We estbalish the Theorem by showing three claims:
\begin{enumerate}
\item Both $I^{(1)}, I^{(2)}$ are $(T_0+1, \epsilon_{T_0+1})$-well-estimated by ${\cal F}$, and the underlying online algorithm and prediction oracle (which are assumed to be fixed but arbitrary in the Theorem statement) suffer
\begin{equation}\label{eq:lb_oracle}
 \text{Regret}_T(I^{(i)}) \geq \frac{ \sum^{T_0}_{t=1} q_t\epsilon_{T_0+1}}{6Q^{(i)}} \; \text{ for some }i\in \{1, 2\}.
\end{equation}
In (\ref{eq:lb_oracle}), we define $\text{Regret}_T(I^{(i)})$ as the regret of the algorithm on instance $I^{(i)}$, and $Q^{(i)} = \sum^{T}_{t=1}q^{(i)}_t$. 
\item Among the set of instances $\{J^{(i)}_c\}_{i\in [K]}$ (see \textbf{Instances $\{J^{(i)}_c\}_{i\in [K]}$}), the online algorithm suffers
\begin{equation}\label{eq:lb_c}
 \text{Regret}_T(J^{(i)}_c) \geq \frac{1}{128}\min\left\{1, \sqrt{\frac{K\bar{q}}{B}}\right\}\text{opt}(J^{(i)}_c) \; \text{ for some }i\in [K],
\end{equation}
where $\text{opt}(I)$ denote the optimum of instance $I$, even when the DM has complete knowledge on $q_1, \ldots, q_T$, and $\hat{Q}_t$ is equal to the ground truth $Q$ in each of the instances in $\{J^{(i)}_c\}_{i\in [K]}$.
\item Among the set of instances $\{J^{(i)}_r\}_{i\in [K]}$ (see \textbf{Instances $\{J^{(i)}_r\}_{i\in [K]}$}),  the online algorithm suffers
\begin{equation}\label{eq:lb_r}
 \text{Regret}_T(J^{(i)}_r) \geq \frac{1}{20}\sqrt{\bar{q}K\text{opt}(J^{(i)}_r) } \; \text{ for some }i\in [K],
\end{equation}
even when the DM has complete knowledge on $q_1, \ldots, q_T$, and $\hat{Q}_t$ is equal to the ground truth $Q$ in each of the instances in $\{J^{(i)}_r\}_{i\in [K]}$.
\end{enumerate}
Once we establish inequalities (\ref{eq:lb_oracle}, \ref{eq:lb_c}, \ref{eq:lb_r}), the Theorem is shown. We remark that (\ref{eq:lb_c}, \ref{eq:lb_r}) are direct consequences of \cite{badanidiyuru2013bandits}. We first extract the instances $\{J^{(i)}_c\}_{i\in [K]}$, $\{J^{(i)}_r\}_{i\in [K]}$ that are constructed in \cite{badanidiyuru2013bandits}, then we construct the instances $I^{(1)}, I^{(2)}$. After that, we prove (\ref{eq:lb_oracle}), which establish the Theorem. 

\textbf{Instances $\{J^{(i)}_c\}_{i\in [K]}$.} These instances are single resource instances, with determinsitic rewards but stochastic consumption. According to \cite{badanidiyuru2013bandits}, we first set parameters 
$$
\eta = \frac{1}{32}\min\left\{1, \sqrt{\frac{K}{B}}\right\},\quad T = \frac{16B}{\eta(1/2 - \eta)},
$$
and set $q_t = \overline{q}$ for all $t\in [T]$. The instances $J^{(1)}_c, \ldots, J^{(K)}_c$ share the same $B, T, \{q_t\}^T_{t=1}$, and the instances share the same (deterministic) reward function:
\[ R(a) = r(a) = 
  \begin{cases}
    1      & \quad \text{if } a\in [K] \setminus \{a_0\}\\
    0  & \quad \text{if } a = a_0
  \end{cases}. 
\]
In contrast, instances $J^{(1)}_c, \ldots, J^{(K)}_c$ differ in the resource consumption model. We denote $C^{(i)}(a)$ as the random consumption of action $a$ in instance $J^{(i)}_c$. The probability distribution of $C^{(i)}(a)$ for each $a, i\in [K]$ is defined as follow:
\[ C^{(i)}(a) \sim 
  \begin{cases}
    \text{Bern}(1/2)     & \quad \text{if } a\in [K] \setminus \{a_0, i\}\\
     \text{Bern}(1/2-\eta)     & \quad \text{if } a=i\\
    \text{Bern}(0) & \quad \text{if } a = a_0
  \end{cases},
\]
where $\text{Bern}(p)$ denotes the Bernoulli distribution with mean $d$. The regret lower bound (\ref{eq:lb_c}) is a direct consequence of Lemma 6.10 in \cite{badanidiyuru2013bandits}, by incorporating the scaling factor $\bar{q}$ into the rewards earned by the DM and the optimal reward.

\textbf{Instances $\{J^{(i)}_r\}_{i\in [K]}$.} These instances are single resource instances, with random rewards but deterministic consumption. These instances share the same $B, T > K$ (set arbitrarily), the same demand volume seqeunce, which is $q_t = \overline{q}$ for all $t\in [T]$, and the same resource consumption model, in that $c(a) = 0$ for all $a\in {\cal A}$. These instances only differ in the reward distributions. We denote $R^{(i)}(a)$ as the random reward of action $a$ in instance $J^{(i)}_r$. The probability distribution of $R^{(i)}(a)$ for each $a, i\in [K]$ is defined as follow:
\[ R^{(i)}(a) \sim 
  \begin{cases}
    \text{Bern}\left(\frac{1}{2} - \frac{1}{4}\sqrt{\frac{K}{T}}   \right)  & \quad \text{if } a\in [K] \setminus \{a_0, i\}\\
     \text{Bern}(1/2)     & \quad \text{if } a=i\\
    \text{Bern}(0) & \quad \text{if } a = a_0
  \end{cases}.
\]
The regret lower bound (\ref{eq:lb_r}) is a direct consequence of Claim 6.2a in \cite{badanidiyuru2013bandits}, by incorporating the scaling factor $\bar{q}$ into the rewards earned by the DM and the optimal reward.
 
\textbf{Construct $I^{(1)}, I^{(2)}$. }
We first describe $\{q_t^{(1)}\}_{t=1}^T, \{q_t^{(2)}\}_{t=1}^T$. As previously mentioned, for $t\in \{1, \ldots, T_0\}$, we have $q_t^{(1)}=q_t^{(2)} = q_t$. To define $q^{(1)}_t, q^{(2)}_t$ for $t\in \{T_0+1, \ldots, T\}$, first recall that $|\hat{Q}_{T_0+1}-Q| \ge \epsilon_{T_0+1}$, where $\epsilon_{T_0+1}$ satisfies (\ref{eq:epsilon_sat}). By (\ref{eq:epsilon_sat}), we know that
\begin{equation*}
    \underline{q}(T - T_0) \le \hat{Q}_{T_0+1} - \sum^{T_0}_{t=1}q_t - \epsilon_{T_0+1} < \hat{Q}_{T_0+1} - \sum^{T_0}_{t=1}q_t + \epsilon_{T_0+1} \le \overline{q}(T - T_0)
\end{equation*}
We set $q^{(1)}_{T_0+1} = \ldots = q^{(1)}_T\in [\underline{q}, \overline{q}]$ and  $q^{(2)}_{T_0+1} = \ldots =q^{(2)}_T\in [\underline{q}, \overline{q}]$ such that based on current instance $\{q_t\}_{t=1}^{T_0}$ we have ever received, we construct the following two subsequent instances $I^{(1)} = \{q_t^{(1)}\}_{t=T_0+1}^T$, $I^{(2)} = \{q_t^{(2)}\}_{t=T_0+1}^T$, such that
\begin{equation*}
    Q^{(1)} = \sum_{t=1}^T q_t^{(1)} = \hat{Q}_{T_0+1} - \epsilon_{T_0+1}, \quad Q^{(2)} = \sum_{t=1}^T q_t^{(2)} = \hat{Q}_{T_0+1} + \epsilon_{T_0+1},
\end{equation*}
which is valid by the stated inequalities.  

Next, we define the parameters $\{r(a)\}_{a\in \mathcal{A}}, \{c(a, 1)\}_{a}, B$. (recall $d=1$) Similar to the proof for Lemma \ref{lem-lowerbound-oblivious}, we only consider deterministic instances, so it is sufficient to define the mean rewards and consumption amounts. To facilitate our discussion, we specify ${\cal A} = [K] = \{1, 2, \ldots, K\}$, with $K\geq 3$ and action $K$ being the null action. The parameters $\{r(a)\}_{a\in \mathcal{A}}, \{c(a, 1)\}_{a}, B$ shared between instances $I^{(1)}, I^{(2)}$ are defined as follows:
\[ r(a) =
  \begin{cases}
    1      & \quad \text{if } a =1,\\
    (1+c)/2 & \quad \text{if } a=2,\\
    0 & \quad \text{if } a\in \{3, \ldots, K\},
  \end{cases}
\]
and
\[ c(a, 1) =
  \begin{cases}
    1      & \quad \text{if } a =1,\\
    c & \quad \text{if } a=2,\\
    0 & \quad \text{if } a\in \{3, \ldots, K\},
  \end{cases}
\]
where
\begin{equation*}
    c = \frac{\hat{Q}_{T_0+1} - \epsilon_{T_0+1}}{ \hat{Q}_{T_0+1} + \epsilon_{T_0+1}}.
\end{equation*}
Finally, we set
$$B = \hat{Q}_{T_0+1} - \epsilon_{T_0+1}.$$
Inequality (\ref{eq:epsilon_sat}) ensures that $c, B > 0$. 

\textbf{Proving (\ref{eq:lb_oracle}).} To evaluate the regrets in the two instances, we start with the optimal rewards. The optimal reward in $I^{(1)}$ is $\hat{Q}_{T_0+1} - \epsilon_{T_0+1}$, which is achieved by pulling action 1 until the resource is exhasuted. The optimal reward for $I^{(2)}$ is $\hat{Q}_{T_0+1} $, which is achieved by pulling action 2 until the resource is exhasuted. 

Consider the execution of the fixed but arbitrary online algorithm during time steps $1, \ldots, T_0$ in each of the instances. The prediction oracle returns the same prediction $\hat{Q}_t$ for $t\in \{1, \ldots, T_0\}$ in both instances, since both instances share the same $r, c, B, T$ and $q^{(1)}_t = q^{(2)}_t$ for $t\in \{1, \ldots, T_0\}$. Consequently, the fixed but arbitrary online algorithm has the same sequence of action pulls $A_1, \ldots, A_{T_0}$ during time steps $1, \ldots, T_0$ in both instances $I^{(1)}, I^{(2)}$. Now, for each action $i\in \{1, 2\}$, we define $N_i = \{t\in \{1, \ldots, T_0\}: A_t = i\}$, which has the same realization in instances $I^{(1)}, I^{(2)}$. Since $N_1 \cup N_2 \subseteq [T_0]$, at least one of the cases $\sum_{t \in N_1} q_t \le \frac{1}{2}\sum_{s=1}^{T_0} q_s $ or $\sum_{t \in N_2} q_t \le \frac{1}{2}\sum_{s=1}^{T_0} q_s $ holds. 

We denote $\text{Reward}_T(I^{(i)})$ as the expected reward of the online algorithm in instance $I^{(i)}$. We proceed with the following case consideration:
\\ \hspace*{\fill} \\
\textbf{Case 1: $\sum_{t \in N_1} q_t \le \frac{1}{2}\sum_{s=1}^{T_0} q_s $.} We consider the online algorithm's execution on $I^{(1)}$, which yields
\begin{equation*}
    \begin{aligned}
    \text{Reward}_T(I^{(1)}) &\le \frac{\sum_{s=1}^{T_0} q_s}{2}  \cdot 1 + \frac{ \sum_{s=1}^{T_0} q_s }{2} \cdot \frac{1}{2}(1 + c) + \left(\hat{Q}_{T_0+1} - \epsilon_{T_0+1}  - \sum_{s=1}^{T_0} q_s \right) \cdot 1\nonumber\\
    &=\left(\sum_{s=1}^{T_0} q_s \right)\left(-\frac{1}{4} + \frac{1}{4} c \right) + \hat{Q}_{T_0+1} - \epsilon_{T_0+1}.
    \end{aligned}
\end{equation*}
Hence,
\begin{equation*}
    \text{Regret}_T(I^{(1)}) \ge \left(\sum_{s=1}^{T_0} q_s \right) \cdot \frac{1}{4} (1-c) = \frac{\sum_{s=1}^{T_0} q_s  \epsilon_{T_0+1}}{2 (\hat{Q}_{T_0+1} +  \epsilon_{T_0+1} )} \geq \frac{\sum_{s=1}^{T_0} q_s  \epsilon_{T_0+1}}{6(\hat{Q}_{T_0+1} -  \epsilon_{T_0+1} )} = \frac{\sum_{s=1}^{T_0} q_s  \epsilon_{T_0+1}}{6Q^{(1)}},
\end{equation*}
where the last inequality is by the inequality $\hat{Q}_{T_0+1} \geq 2\epsilon_{T_0 + 1}$, which is implied by the well repsonse assumption of $\epsilon_{T_0 + 1}$. For the last equality, recall that $\hat{Q}_{T_0+1} -  \epsilon_{T_0+1}  = \sum^T_{t=1} q^{(1)}_t$.

\textbf{Case 2: $\sum_{t \in N_2} q_t \le \frac{1}{2}\sum_{s=1}^{T_0} q_s $.} We consider the online algorithm's execution on $I^{(2)}$, which yields
\begin{equation*}
    \begin{aligned}
    \text{Reward}_T(I^{(2)}) &\leq\left(\sum_{s=1}^{T_0} q_s  - \sum_{t \in N_2} q_t \right) \cdot 1 + \sum_{t \in N_2} q_t \cdot \frac{1}{2}(1+c) \\
    & + \left(B - \left(\sum_{s=1}^{T_0} q_s  - \sum_{t \in N_2} q_t\right) - \sum_{t \in N_2} q_t \cdot c \right) \cdot \frac{\frac{1}{2}(1+c)}{c} \\
    &= \sum_{s=1}^{T_0} q_s \left(1 - \frac{1+c}{2c}\right) + \left(\sum_{s\in N_2}q_s\right)\left[-1 + \frac{1+c}{2} + \frac{1+c}{2c} - \frac{1+c}{2}\right] + B\cdot \frac{1+c}{2c}\\
    & = - \frac{\sum_{s=1}^{T_0} q_s \epsilon_{T_0 + 1}}{\hat{Q}_{T_0 + 1} - \epsilon_{T_0 + 1}} + \left(\sum_{s\in N_2}q_s\right)\cdot \frac{\epsilon_{T_0 + 1}}{\hat{Q}_{T_0 + 1} - \epsilon_{T_0 + 1}} + \hat{Q}_{T_0}\\
    & \leq - \frac{\sum_{s=1}^{T_0} q_s \epsilon_{T_0 + 1}}{2(\hat{Q}_{T_0 + 1} - \epsilon_{T_0 + 1})} + \hat{Q}_{T_0+1}.
    \end{aligned}
\end{equation*}
Hence,
\begin{equation*}
    \text{Regret}_T(I^{(2)}) \ge \frac{\sum_{s=1}^{T_0} q_s   \epsilon_{T_0+1}}{2(\hat{Q}_{T_0+1} - \epsilon_{T_0+1})} \geq \frac{\sum_{s=1}^{T_0} q_s   \epsilon_{T_0+1}}{2(\hat{Q}_{T_0+1} + \epsilon_{T_0+1})} = \frac{\sum_{s=1}^{T_0} q_s   \epsilon_{T_0+1}}{2Q^{(2)}}. 
\end{equation*}
Altogether, the Theorem is proved.

\subsection{Consistency Between Regret Upper and Lower Bounds}\label{sec:consistent}
Recall that in the proof of Theorem \ref{thm-lowerbound-estimation}, we constructed two instances $I^{(1)}, I^{(2)}$ such that (see (\ref{eq:lb_oracle}): 
\begin{equation}\label{eq:lb_oracle_copy}
 \text{Regret}_T(I^{(i)}) \geq \frac{ \sum^{T_0}_{t=1} q_t\epsilon_{T_0+1}}{6Q^{(i)}} \; \text{ for some }i\in \{1, 2\},
 \end{equation}
where $\text{Regret}_T(I^{(i)})$ is the regret of an arbitrary but fixed online algorithm on $I^{(i)}$, with its prediction oracle satisfying that
\begin{equation}\label{eq:property_above}
|Q^{(i)} - \hat{Q}_t|\leq \epsilon_{T_0+1}\quad \text{ for each $i\in \{1, 2\}$.}
\end{equation} 
In the lower bound analysis on $I^{(1)}, I^{(2)}$, we establish the regret lower bound (\ref{eq:lb_oracle_copy}) solely hinging on the model uncertainty on $Q^{(1)}, Q^{(2)}$, and the bound (\ref{eq:lb_oracle_copy}) still holds when the DM knows $\{P_a\}_{a\in \mathcal{A}}$. 

In particular, we can set the online algorihtm to be OA-UCB, with an oracle that satisfies the property (\ref{eq:property_above}) above.  Now, also recall in our construction that $q^{(1)}_t = q^{(2)}_t = q_t$ for all $t\in [T_0]$, thus the predictions $\hat{Q}_t$ for  $t\in [T_0]$ are the same in the two instances, whereas $Q^{(1)} = \hat{Q}_{T_0+1} - \epsilon_{T_0+1}$ but $Q^{(2)} = \hat{Q}_{T_0+1} + \epsilon_{T_0+1}$, while we still have $Q^{(2)}\leq 3 Q^{(1)}$, so that $Q^{(1)} = \Theta(Q^{(2)})$. Therefore, (\ref{eq:lb_oracle_copy}) is equivalent to
\begin{equation}\label{eq:lb_oracle_short}
 \max_{i\in \{1, 2\}}\{\text{Regret}_T(I^{(i)})\} \geq \Omega\left(\frac{ \sum^{T_0}_{t=1} q_t\epsilon_{T_0+1}}{Q^{(1)}}\right).
 \end{equation}
To demonstrate the consistency, it suffices to show
\begin{equation}\label{eq:consistency_required}
    \max_{i\in \{1, 2\}}\left\{\frac{1}{Q^{(1)}} \sum_{t=1}^{\tau-1} q_t \epsilon^{(i)}_t \right\} = \Omega\left(\frac{ \sum^{T_0}_{t=1} q_t\epsilon_{T_0+1}}{Q^{(1)}}\right).
\end{equation}
where $\epsilon^{(i)}_t = |\hat{Q}_t - Q^{(i)}|$ is the prediction error on instance $I^{(i)}$ at time $t$. 
Indeed, to be consistent, we should have Theorem \ref{thm-upperbound-1} holds true for both instances, while (\ref{eq:lb_oracle_short}) still holds true. We establish (\ref{eq:consistency_required}) as follows:
\begin{align}
    \max_{i\in \{1, 2\}}\left\{ \sum_{t=1}^{\tau-1} q_t \epsilon^{(i)}_t \right\} &\geq \sum_{t=1}^{\tau-1} q_t \frac{\epsilon^{(1)}_t+ \epsilon^{(2)}_t}{2}\nonumber\\
    & = \sum_{t=1}^{\tau-1} q_t \frac{|\hat{Q}_t - \hat{Q}_{T_0+1} + \epsilon_{T_0+1}|+ |\hat{Q}_t - \hat{Q}_{T_0+1} - \epsilon_{T_0+1}|}{2} \nonumber\\
    & \geq \sum_{t=1}^{\tau-1} q_t \frac{2 \epsilon_{T_0+1}}{2} \label{eq:by_triangle}\\
    & \geq \sum_{t=1}^{T_0} q_t  \epsilon_{T_0+1}.\label{eq:by_stopping_time}
\end{align}
Step (\ref{eq:by_triangle}) is by the triangle inequality, and step (\ref{eq:by_stopping_time}) is by the fact that for any algorithm that fully exhausts the resource, its stopping time $\tau>T_0$ (In the case when OA-UCB does not fully consume all the resource at the end of time $T$, by definition we have $\tau = T+1 > T_0$). By construction, the common budget $B$ in both instances is strictly larger than $\sum^{T_0}_{t=1}q_t$, thus the resource is always not exhasuted at $T_0$, since at time $t\in [T_0]$ the DM consumes at most $q_t$ units of resource. Altogether, (\ref{eq:consistency_required}) is shown and consistency is verified.

\section{Proof for Theorem \ref{thm-upperbound-1}}\label{sec:pf-thm-upperbound-1}





Denote $\textbf{UCB}_{r,t} = (\text{UCB}_{r,t}(a))_{a \in \mathcal{A}}$, $\textbf{LCB}_{c,t} = (\textbf{LCB}_{c,t}(a,i))_{a \in \mathcal{A},i \in [d]}$.
We first claim that, at a time step $t\leq \tau - 1$, 
\begin{equation}
    \boldsymbol{e}_{A_t} \in \mathop{\arg\max}_{\boldsymbol{u} \in \Delta_{K}} \ \textbf{UCB}_{c,t}^\top \boldsymbol{u} - \frac{\hat{Q}_t}{B} \cdot  \boldsymbol{\mu}_t^\top \textbf{LCB}_{c,t}^{\top} \boldsymbol{u}.
    \label{eq-pf-upper-claim}
\end{equation}
In fact, the following linear optimization problem 
\begin{equation*}
    \begin{aligned}
    \max \quad & \textbf{UCB}_{r,t}^\top \boldsymbol{u} - \frac{\hat{Q}_t}{B} \cdot  \boldsymbol{\mu}_t^\top \textbf{LCB}_{c,t}^{\top} \boldsymbol{u} \\
    \text{s.t.} \quad & \boldsymbol{u} \in \Delta_{K}
    \end{aligned}
\end{equation*}
has an extreme point solution such that $\boldsymbol{u}^* = \boldsymbol{e}_{a}$ for some $a \in \mathcal{A}$. According to the definition of $A_t$, we know that $\boldsymbol{u}^* = \boldsymbol{e}_{A_t}$. Then the claim holds. Suppose $\boldsymbol{u}^*$ is an optimal solution of  (\ref{eq-linear-relaxation2}), then we have $\text{OPT}_{\text{LP}} = Q \boldsymbol{r}^{\top} \boldsymbol{u}^*$, $Q \boldsymbol{c}^{\top} \boldsymbol{u}^* \le B \boldsymbol{1}$ and $\boldsymbol{u}^* \in \Delta_{K}$. By the optimality of (\ref{eq-pf-upper-claim}), we have
\begin{equation*}
    \text{UCB}_{r,t}(A_t) - \frac{\hat{Q}_t}{B} \cdot  \boldsymbol{\mu}_t^\top \textbf{LCB}_{c,t}(A_t) = \textbf{UCB}_{r,t}^{\top} \boldsymbol{e}_{A_t} - \frac{\hat{Q}_t}{B} \cdot  \boldsymbol{\mu}_t^\top \textbf{LCB}_{c,t}^{\top} \boldsymbol{e}_{A_t} \ge \textbf{UCB}_{r,t}^{\top} \boldsymbol{u}^* - \frac{\hat{Q}_t}{B} \cdot  \boldsymbol{\mu}_t^\top \textbf{LCB}_{c,t} \boldsymbol{u}^*,
\end{equation*}
which is equivalent to
\begin{equation}\label{eq:optimism_step}
    \textbf{UCB}_{r,t}^\top \boldsymbol{u}^* - \text{UCB}_{r,t}(A_t) + \frac{\hat{Q}_t}{B} \cdot  \boldsymbol{\mu}_t^{\top} \left(\frac{B}{\hat{Q}_t} \boldsymbol{\beta} - \textbf{LCB}_{c,t}^{\top} \boldsymbol{u}^* \right) \le \frac{\hat{Q}_t}{B} \cdot  \boldsymbol{\mu}_t^\top \left(\frac{B}{\hat{Q}_t} \boldsymbol{\beta} - \textbf{LCB}_{c,t}(A_t) \right).
\end{equation}
Multiply $q_t$ on both side of (\ref{eq:optimism_step}), and sum over $t$ from 1 to $\tau - 1$. By  Lemma \ref{lem-ocoalg-regret}, for any $\boldsymbol{\mu} \in \Delta_{d+1}$ it holds that

\begin{equation}
    \begin{aligned}
    &\sum_{t=1}^{\tau - 1} q_t \textbf{UCB}_{r,t}^{\top} \boldsymbol{u}^* - \sum_{t=1}^{\tau - 1}q_t \text{UCB}_{r,t}(A_t) + \sum_{t=1}^{\tau - 1}q_t \frac{\hat{Q}_t}{B} \cdot  \boldsymbol{\mu}_t^{\top} \left(\frac{B}{\hat{Q}_t} \boldsymbol{\beta} - \textbf{LCB}_{c,t}^{\top} \boldsymbol{u}^* \right) \\
    \le & \sum_{t=1}^{\tau-1} q_t \frac{\hat{Q}_t}{B} \cdot  \boldsymbol{\mu}_t^\top \left(\frac{B}{\hat{Q}_t} \boldsymbol{\beta} - \textbf{LCB}_{c,t}(A_t) \right) \\
    \le & \sum_{t=1}^{\tau-1} q_t \frac{\hat{Q}_t}{B} \cdot \left(\frac{B}{\hat{Q}_t} \boldsymbol{\beta} - \textbf{LCB}_{c,t}(A_t) \right)^{\top} \boldsymbol{\mu} + O \left( \left(\overline{q} + \frac{\overline{q}^2}{b} \right) \sqrt{(\tau - 1)\ln(d+1)}  \right).
    \end{aligned}
    \label{eq-pf-upper-sumqt}
\end{equation}
Recall by Lemma \ref{lem-confidencebound} that with probability $\ge 1 - 3 KTd \delta$, we have 
\begin{equation*}
    \textbf{LCB}_{c,t} \le \boldsymbol{c}.
\end{equation*}
Hence, with probability $\ge 1 - 3 KTd \delta$, 
\begin{subequations}
    \begin{align}
    \sum_{t=1}^{\tau - 1}q_t \frac{\hat{Q}_t}{B} \cdot  \boldsymbol{\mu}_t^\top \left(\frac{B}{\hat{Q}_t} \boldsymbol{\beta} - \textbf{LCB}_{c,t}^\top \boldsymbol{u}^* \right) & \ge \sum_{t=1}^{\tau - 1}q_t \frac{\hat{Q}_t}{B} \cdot  \boldsymbol{\mu}_t^\top \left(\frac{B}{\hat{Q}_t} \boldsymbol{\beta} - \boldsymbol{c}^{\top} \boldsymbol{u}^* \right) \label{eq-pf-upper-main1-a}  \\
    & = \sum_{t=1}^{\tau - 1} q_t \frac{\hat{Q}_t}{B} \cdot \boldsymbol{\mu}_t^\top \left(\frac{B}{\hat{Q}_t} \boldsymbol{\beta} - \frac{B}{Q} \boldsymbol{\beta} \right) \nonumber \\ 
    &+ \sum_{t=1}^{\tau - 1}q_t \frac{\hat{Q}_t}{B} \cdot  \boldsymbol{\mu}_t^\top \left(\frac{B}{Q}  \boldsymbol{\beta} - \boldsymbol{c}^{\top} \boldsymbol{u}^* \right) \label{eq-pf-upper-main1-b}\\
    & \ge \sum_{t=1}^{\tau - 1}q_t \frac{\hat{Q}_t}{B} \cdot  \boldsymbol{\mu}_t^\top \left(\frac{B}{\hat{Q}_t} \boldsymbol{\beta} - \frac{B}{Q} \boldsymbol{\beta} \right) \label{eq-pf-upper-main1-c}\\
    & \ge - \sum_{t=1}^{\tau - 1}q_t \frac{\hat{Q}_t}{B}\left | \frac{B}{\hat{Q}_t} - \frac{B}{Q} \right |,
    \end{align}
    \label{eq-pf-upper-main1}
\end{subequations}
where (\ref{eq-pf-upper-main1-a}) comes from Lemma $\ref{lem-confidencebound}$, (\ref{eq-pf-upper-main1-b}) comes from rearranging the sum, and (\ref{eq-pf-upper-main1-c}) comes from the fact the definition of $\boldsymbol{u}^*$. We first consider the case $\tau \le T$, which implies that there exists $j_0 \in [d]$ such that 
\begin{equation}
    \sum_{t=1}^{\tau} q_t C_{t,j_0} > B \quad \Rightarrow \quad \sum_{t=1}^{\tau - 1} q_t C_{t,j_0} > B-\overline{q}.
    \label{eq-pf-upper-taucondition}
\end{equation}
Take $\boldsymbol{\mu} = \lambda \boldsymbol{e}_{j_0} + (1 - \lambda) \boldsymbol{e}_{d+1}$, where $\lambda \in [0,1]$ is a constant that we tune later. In this case, with probability $\ge 1 - 3KT \delta$,

\begin{equation}
    \sum_{t=1}^{\tau - 1} q_t \textbf{UCB}_{r,t}^{\top} \boldsymbol{u}^* \ge \sum_{t=1}^{\tau - 1} q_t \boldsymbol{r}_t^\top \boldsymbol{u}^* = \text{OPT}_{\text{LP}} \frac{Q_{\tau - 1}}{Q},
    \label{eq-pf-upper-taucase1}
\end{equation}
and 
\begin{equation}
    \begin{aligned}
    \sum_{t=1}^{\tau-1} q_t \frac{\hat{Q}_t}{B}  \left(\frac{B}{\hat{Q}_t} \boldsymbol{\beta} - \textbf{LCB}_{c,t}(A_t) \right)^{\top} \left ( \lambda \boldsymbol{e}_{j_0} + (1 - \lambda) \boldsymbol{e}_{d+1} \right ) &=\lambda \sum_{t=1}^{\tau-1} q_t \frac{\hat{Q}_t}{B}  \left(\frac{B}{\hat{Q}_t} - \text{LCB}_{c,t}(A_t,j_0) \right) \\
    & = \lambda \sum_{t=1}^{\tau-1} q_t \frac{\hat{Q}_t}{B}  \left(\frac{B}{\hat{Q}_t} - \frac{B}{Q} \right) \\
    & + \lambda \sum_{t=1}^{\tau-1} q_t \frac{\hat{Q}_t}{B}  \left(\frac{B}{Q} - C_{t,j_0} \right) \\
    &+ \lambda \sum_{t=1}^{\tau-1} q_t \frac{\hat{Q}_t}{B}  \left(C_{t,j_0} - \text{LCB}_{c,t}(A_t,j_0) \right).
    \end{aligned}
    \label{eq-pf-upper-taucase2}
\end{equation}
Then we deal with each term respectively:
\begin{subequations}
    \begin{align}
    \sum_{t=1}^{\tau-1} q_t \frac{\hat{Q}_t}{B}  \left(\frac{B}{Q} - C_{t,j_0} \right) & = \sum_{t=1}^{\tau-1} q_t \frac{Q}{B}  \left(\frac{B}{Q} - C_{t,j_0} \right) + \sum_{t=1}^{\tau-1} q_t \frac{\hat{Q}_t - Q }{B}  \left(\frac{B}{Q} - C_{t,j_0} \right) \label{eq-pf-upper-taucase2-1-a}\\
    & \le Q_{\tau - 1} - \frac{Q}{B}\sum_{t=1}^{\tau - 1} q_t C_{t,j_0} + \frac{1}{B} \sum_{t=1}^{\tau-1} q_t \epsilon_t  \left|\frac{B}{Q} - C_{t,j_0} \right| \label{eq-pf-upper-taucase2-1-b}\\
    & < Q_{\tau - 1} - Q + \frac{Q}{B} \overline{q} + \frac{1}{B} \sum_{t=1}^{\tau-1} q_t \epsilon_t  \frac{B}{Q} + \frac{1}{B} \sum_{t=1}^{\tau-1} q_t \epsilon_t C_{t,j_0} \label{eq-pf-upper-taucase2-1-c}\\
    & \le Q_{\tau - 1} - Q + \frac{Q}{B} \overline{q} + \left (\frac{1}{Q} + \frac{1}{B} \right) \sum_{t=1}^{\tau-1} q_t \epsilon_t, \label{eq-pf-upper-taucase2-1-d} 
    \end{align}
    \label{eq-pf-upper-taucase2-1}
\end{subequations}
where (\ref{eq-pf-upper-taucase2-1-a}) comes from rearranging the sum, (\ref{eq-pf-upper-taucase2-1-c}) comes from the (\ref{eq-pf-upper-taucondition}), and (\ref{eq-pf-upper-taucase2-1-d}) comes from the fact that $C_{t,j_0} $ is supported in $[0,1]$. Similarly, 
\begin{equation}
    \begin{aligned}
    \sum_{t=1}^{\tau-1} q_t \frac{\hat{Q}_t}{B}  \left(C_{t,j_0} - \text{LCB}_{c,t}(A_t,j_0) \right) & =  \sum_{t=1}^{\tau-1} q_t \frac{Q}{B}  \left(C_{t,j_0} - \text{LCB}_{c,t}(A_t,j_0) \right) \\
    & + \sum_{t=1}^{\tau-1} q_t \frac{\hat{Q}_t - Q}{B}  \left(C_{t,j_0} - \text{LCB}_{c,t}(A_t,j_0) \right) \\
    & \le \frac{Q}{B} \left |\sum_{t=1}^{\tau-1} q_t \left(C_{t,j_0} - \text{LCB}_{c,t}(A_t,j_0) \right) \right | + \frac{1}{B} \sum_{t=1}^{\tau-1} q_t \epsilon_t,
    \end{aligned}
    \label{eq-pf-upper-taucase2-2}
\end{equation}
where the inequality comes from the fact that $|\hat{Q}_t - Q| \le \epsilon_t$, $0 \le \text{LCB}_{c,t}(A_t,j_0), C_{t,j_0} \le 1$. Combine (\ref{eq-pf-upper-taucase2}), (\ref{eq-pf-upper-taucase2-1}) and (\ref{eq-pf-upper-taucase2-2}), we obtain
\begin{equation}
    \begin{aligned}
    & \sum_{t=1}^{\tau-1} q_t \frac{\hat{Q}_t}{B}  \left(\frac{B}{\hat{Q}_t} \boldsymbol{\beta}- \textbf{LCB}_{c,t}(A_t) \right)^{\top} \left ( \lambda \boldsymbol{e}_{j_0} + (1 - \lambda) \boldsymbol{e}_{d+1} \right ) \\
    \le & \lambda \sum_{t=1}^{\tau-1} q_t \frac{\hat{Q}_t}{B}  \left(\frac{B}{\hat{Q}_t} - \frac{B}{Q} \right) + \lambda \left(Q_{\tau - 1} - Q + \frac{Q}{B} \overline{q} + \left (\frac{1}{Q} + \frac{1}{B} \right) \sum_{t=1}^{\tau-1} q_t \epsilon_t \right) \\
    + & \lambda \left( \frac{Q}{B} \left |\sum_{t=1}^{\tau-1} q_t \left(C_{t,j_0} - \text{LCB}_{c,t}(A_t,j_0) \right) \right | + \frac{1}{B} \sum_{t=1}^{\tau-1} q_t \epsilon_t \right) \\
    \le & \lambda \left(Q_{\tau - 1} - Q + \frac{Q}{B} \overline{q} +  \frac{Q}{B} \left |\sum_{t=1}^{\tau-1} q_t \left(C_{t,j_0} - \text{LCB}_{c,t}(A_t,j_0) \right) \right | \right) + \sum_{t=1}^{\tau-1} q_t \frac{\hat{Q}_t}{B}  \left(\frac{B}{\hat{Q}_t} - \frac{B}{Q} \right) \\
    +& O \left( \left (\frac{1}{Q} + \frac{1}{B} \right) \sum_{t=1}^{\tau-1} q_t \epsilon_t \right),
    \end{aligned}
    \label{eq-pf-upper-taucase3}
\end{equation}
Finally, combine (\ref{eq-pf-upper-sumqt}), (\ref{eq-pf-upper-main1}), (\ref{eq-pf-upper-taucase1}), (\ref{eq-pf-upper-taucase3}), we obtain
\begin{equation*}
    \begin{aligned}
    & \text{OPT}_{\text{LP}} \frac{Q_{\tau - 1}}{Q} - \sum_{t=1}^{\tau - 1}q_t \text{UCB}_{r,t}(A_t) - \sum_{t=1}^{\tau - 1}q_t \frac{\hat{Q}_t}{B}\left | \frac{B}{\hat{Q}_t} - \frac{B}{Q} \right | \\
    \le & \lambda \left(Q_{\tau - 1} - Q + \frac{Q}{B} \overline{q} +  \frac{Q}{B} \left |\sum_{t=1}^{\tau-1} q_t \left(C_{t,j_0} - \text{LCB}_{c,t}(A_t,j_0) \right) \right | \right) + \sum_{t=1}^{\tau-1} q_t \frac{\hat{Q}_t}{B}  \left(\frac{B}{\hat{Q}_t} - \frac{B}{Q} \right) \\
    +& O \left( \left (\frac{1}{Q} + \frac{1}{B} \right) \sum_{t=1}^{\tau-1} q_t \epsilon_t \right),
    \end{aligned}
\end{equation*}
which is equivalent to
\begin{equation*}
    \begin{aligned}
    \text{OPT}_{\text{LP}} - \sum_{t=1}^{\tau - 1}q_t \text{UCB}_{r,t}(A_t) & \le \text{OPT}_{\text{LP}} \left(1 - \frac{Q_{\tau - 1}}{Q} \right) + \lambda \left(Q_{\tau - 1} - Q + \frac{Q}{B} \overline{q} \right . \\
    & \left .+  \frac{Q}{B} \left |\sum_{t=1}^{\tau-1} q_t \left(C_{t,j_0} - \text{LCB}_{c,t}(A_t,j_0) \right) \right | \right) + 2 \sum_{t=1}^{\tau-1} q_t \frac{\hat{Q}_t}{B}  \left | \frac{B}{\hat{Q}_t} - \frac{B}{Q} \right | \\
    &+ O \left( \left (\frac{1}{Q} + \frac{1}{B} \right) \sum_{t=1}^{\tau-1} q_t \epsilon_t \right) + O \left( \left(\overline{q} + \frac{\overline{q}^2}{b} \right) \sqrt{(\tau - 1)\ln(d+1) } \right).
    \end{aligned}
\end{equation*}

Let $\lambda = \frac{\text{OPT}_{\text{LP}}}{Q} \le 1$ (This is because $\text{OPT}_{\text{LP}} = Q \boldsymbol{r}^{\top} \boldsymbol{u}^* \le Q$), then we can further derive with probability $\ge 1 - 3KTd \delta$,
\begin{equation}
    \begin{aligned}
    \text{OPT}_{\text{LP}} - \sum_{t=1}^{\tau - 1}q_t \text{UCB}_{r,t}(A_t) & \le \text{OPT}_{\text{LP}} \left(1 - \frac{Q_{\tau-1}}{Q} \right) + \frac{\text{OPT}_{\text{LP}}}{Q} \left(Q_{\tau - 1} - Q + \frac{Q}{B} \overline{q} \right. \\
    & \left.+  \frac{Q}{B} \left |\sum_{t=1}^{\tau-1} q_t \left(C_{t,j_0} - \text{LCB}_{c,t}(A_t,j_0) \right) \right | \right) + 2 \sum_{t=1}^{\tau-1} q_t \frac{\hat{Q}_t}{B}  \left | \frac{B}{\hat{Q}_t} - \frac{B}{Q} \right | \\
    & + O \left( \left (\frac{1}{Q} + \frac{1}{B} \right) \sum_{t=1}^{\tau-1} q_t \epsilon_t \right) + O \left( \left(\overline{q} + \frac{\overline{q}^2}{b} \right) \sqrt{(\tau - 1)\ln(d+1) }\right) \\
    & = \frac{\text{OPT}_{\text{LP}}}{B} \overline{q} + \frac{\text{OPT}_{\text{LP}}}{B}  \left |\sum_{t=1}^{\tau-1} q_t \left(C_{t,j_0} - \text{LCB}_{c,t}(A_t,j_0) \right) \right | + 2 \sum_{t=1}^{\tau-1} q_t \frac{\hat{Q}_t}{B}  \left | \frac{B}{\hat{Q}_t} - \frac{B}{Q} \right | \\
    &  + O \left( \left (\frac{1}{Q} + \frac{1}{B} \right) \sum_{t=1}^{\tau-1} q_t \epsilon_t \right) + O \left( \left(\overline{q} + \frac{\overline{q}^2}{b} \right) \sqrt{(\tau - 1)\ln(d+1)} \right) \\
    & \le O \left(\log\left(\frac{1}{\delta} \right) \text{OPT}_{\text{LP}} \sqrt{\frac{\overline{q} K}{B}}    + \left (\frac{1}{Q} + \frac{1}{B} \right) \sum_{t=1}^{\tau-1} q_t \epsilon_t \right .\\
    &\left. + \left(\overline{q} + \frac{\overline{q}^2}{b} \right) \sqrt{(\tau - 1)\ln(d+1)}\right),
    \end{aligned}
    \label{eq-pf-upper-taubound}
\end{equation}
where the second inequality comes from Lemma \ref{lem-lcbc} and the following 
\begin{equation*}
    \sum_{t=1}^{\tau} q_t \frac{\hat{Q}_t}{B} \left | \frac{B}{Q} - \frac{B}{\hat{Q}_t} \right | = \frac{1}{Q} \sum_{t=1}^{\tau} q_t \left |\hat{Q}_t - Q\right | \le \frac{1}{Q} \sum_{t=1}^{\tau} q_t \epsilon_t.
\end{equation*}
The above concludes our arguments for the case $\tau\leq T$. In complement, we then consider the case $\tau > T$, which means that $\tau = T+1$, and no resource is fully exhausted during the horizon. With probability $\ge 1 -3 KT \delta$, we have
\begin{equation}
    \sum_{t=1}^{T} q_t \textbf{UCB}_{r,t}^{\top} \boldsymbol{u}^* \ge \sum_{t=1}^{T} q_t \boldsymbol{r}_t^{\top} \boldsymbol{u}^* = \text{OPT}_{\text{LP}}.
    \label{eq-pf-upper-Tcase1}
\end{equation}
Take $\boldsymbol{\mu} = \boldsymbol{e}_{d+1}$ and combine (\ref{eq-pf-upper-sumqt}), (\ref{eq-pf-upper-main1}), (\ref{eq-pf-upper-Tcase1}), with probablity $\ge 1 - 3 KT \delta$, we have
\begin{equation}
    \begin{aligned}
    \text{OPT}_{\text{LP}} - \sum_{t=1}^{\tau - 1}q_t \text{UCB}_{r,t}(A_t) & \le \sum_{t=1}^{\tau - 1}q_t \frac{\hat{Q}_t}{B}\left | \frac{B}{\hat{Q}_t} - \frac{B}{Q} \right | + O \left( \left(\overline{q} + \frac{\overline{q}^2}{b} \right) \sqrt{T \ln(d+1)} \right) \\
    & \le O \left(\frac{1}{Q} \sum_{t=1}^{\tau-1} q_t \epsilon_t + \left(\overline{q} + \frac{\overline{q}^2}{b} \right) \sqrt{T\ln(d+1) } \right).
    \end{aligned}
    \label{eq-pf-upper-Tbound}
\end{equation}
Combine (\ref{eq-pf-upper-Tbound}) and (\ref{eq-pf-upper-taubound}), for any stopping time $\tau$, with probability $\ge 1 -3KTd \delta$, we have
\begin{equation*}
    \begin{aligned}
    \text{OPT}_{\text{LP}} - \sum_{t=1}^{\tau - 1}q_t \text{UCB}_{r,t}(A_t) & \le O \left(\log\left(\frac{1}{\delta} \right) \text{OPT}_{\text{LP}} \sqrt{\frac{\overline{q} K}{B}}    + \left (\frac{1}{Q} + \frac{1}{B} \right) \sum_{t=1}^{\tau-1} q_t \epsilon_t \right. \\
    & \left. + \left(\overline{q} + \frac{\overline{q}^2}{b} \right) \sqrt{(\tau - 1) \ln(d+1)} \right).
    \end{aligned}
\end{equation*}
By Lemma \ref{lem-ucbr}, we can further derive it to the high probability bound, that with probability $\ge 1 -3KTd \delta$,
\begin{equation*}
    \begin{aligned}
    \text{OPT}_{\text{LP}} - \sum_{t=1}^{\tau - 1}q_t R_t & \le O \left(\log\left(\frac{1}{\delta} \right) \left(  \text{OPT}_{\text{LP}} \sqrt{\frac{\overline{q} K}{B}} + \sqrt{\overline{q} K \sum_{t=1}^{\tau - 1} q_t R_t} +   \overline{q} K \log\left (\frac{T}{K} \right ) \right) \right. \\
    & \left . + \left (\frac{1}{Q} + \frac{1}{B} \right) \sum_{t=1}^{\tau-1} q_t \epsilon_t + \left(\overline{q} + \frac{\overline{q}^2}{b} \right) \sqrt{(\tau - 1)\ln(d+1)} \right) \\
    & \le O \left(\log\left(\frac{1}{\delta} \right) \left(  \text{OPT}_{\text{LP}} \sqrt{\frac{\overline{q} K}{B}} + \sqrt{\overline{q} K \text{OPT}_{\text{LP}} }  \right)  + \left (\frac{1}{Q} + \frac{1}{B} \right) \sum_{t=1}^{\tau-1} q_t \epsilon_t \right.\\ 
    & \left.+ \left(\overline{q} + \frac{\overline{q}^2}{b} \right) \sqrt{(\tau - 1)\ln(d+1)} \right),
    \end{aligned}
\end{equation*}
where the second inequality comes from the fact that $\text{OPT}_{\text{LP}} \ge \sum_{t=1}^{\tau - 1} q_t R_t$. Now we finish the proof of Theorem \ref{thm-upperbound-1}.

\section{Proofs for Confidence Radii}\label{sec:pf_conf_radii}
This section contains proofs for the confidence radius results, which largely follow the literature, but we provide complete proofs since we are in a non-stationary setting. Section \ref{sec:pf_lem-radius} provides the proof for Lemma \ref{lem-radius}, which allows us to extract the implicit constants in existing proofs in \cite{babaioff2015dynamic,agrawal2014bandits}. Section \ref{sec:pf_lem-confidencebound} provides the proof for Lemma \ref{lem-confidencebound}. Finally, section \ref{sec:pf_upper_lower}, we prove inequalities (\ref{eq-spf-upper-ucb}, \ref{eq-spf-upper-lcb}).
\subsection{Proof for Lemma \ref{lem-radius}, due to \cite{babaioff2015dynamic,agrawal2014bandits} }\label{sec:pf_lem-radius}
In this subsection, we prove Lemma \ref{lem-radius} by following the line of arguments in \cite{babaioff2015dynamic}. We emphasize that a version of the Lemma has been proved in \cite{babaioff2015dynamic}. We dervie the Lemma for the purpose of extracting the values of the constant coefficients.
We first extract some relevant concentration inequalities in the following two Lemmas. 
\begin{lemma}[Theorem 8 in \cite{chung2006concentration}] Suppose $\{U_i\}^n_{i=1}$ are independent random variables satisfying $U_i \le M$, for $1 \le i \le n$ almost surely. Let $U = \sum_{i=1}^n U_i$, $\|U\|^2 = \sum_{i=1}^n \mathbb{E}[U_i^2]$. With probability $\ge 1 - e^{-x}$, we have
\begin{equation*}
    U - \mathbb{E}[U] \le \sqrt{2 \|U\|^2 x} + \frac{2 x}{3} \max\{M,0\}.
\end{equation*}
\label{lem-concen-squaresum}
\end{lemma}
\begin{lemma}[Theorem 6 in \cite{chung2006concentration}] Suppose $U_i$ are independent random variables satisfying $U_i - \mathbb{E}[U_i] \le M$, $M>0$, for $1 \le i \le n$. Let $U = \sum_{i=1}^n U_i$, $\text{Var}(U) = \sum_{i=1}^n \text{Var}(U_i)$, then with probability $\ge 1 - e^{-x}$, we have
\begin{equation*}
    U - \mathbb{E}[U] \le \sqrt{2 \text{Var}(U) x} + \frac{2 M x}{3}.
\end{equation*}
\label{lem-concen-var}
\end{lemma} 
Using Lemma \ref{lem-concen-var}, we first derive the following Lemma that bounds the empirical mean:
\begin{lemma}
Let $\{X_i\}^n_{i=1}$ be independent random variables supported in $[0,1]$. Let $X = \sum_{i=1}^n X_i$ and $\text{Var}(X) = \sum_{i=1}^n \text{Var}(X_i)$. For any fixed $x>0$, With probability $\ge 1 - 2e^{-x}$, we have
\begin{equation*}
    \left |X - \mathbb{E}[X] \right | \le \sqrt{2 \text{Var}(X) x} + \frac{2x}{3}.
\end{equation*}
\label{lem-concen-absvar}
\end{lemma}

\begin{proof}[Proof of Lemma \ref{lem-concen-absvar}]
Apply Lemma \ref{lem-concen-var} with $U_i = X_i$, $U_i = - X_i$, respectively, and $M=1$, then with probability $\ge 1 - 2 e^{-x}$, we have
\begin{equation*}
    \left |X - \mathbb{E}[X] \right | \le \sqrt{2 \text{Var}(X) x} + \frac{2x}{3}.
\end{equation*}
\end{proof}
Next, we bound the difference between the ground truth variance and its empirical counterpart using Lemma \ref{lem-concen-squaresum}:
\begin{lemma}
Suppose $X_i$ are independent random variables supported in $[0,1]$. Let $X = \sum_{i=1}^n X_i$, $\text{Var}(X) = \sum_{i=1}^n \text{Var}(X_i)$, $V_n = \sum_{i=1}^n \left(X_i - \mathbb{E}[X_i] \right)^2$then with probability $\ge 1 - 3e^{-x}$, we have
\begin{equation*}
    \sqrt{\text{Var}(X)} \le \sqrt{V_n} + 2 \sqrt{x}.
\end{equation*}
\label{lem-concen-varest}
\end{lemma}

\begin{proof}[Proof of Lemma \ref{lem-concen-varest}]
The proof follows the line of argument in \cite{audibert2009exploration}. First, we apply Lemma \ref{lem-concen-squaresum} with $U_i = -(X_i - \mathbb{E}[X_i])^2$ and $M = 0$. With probability $\ge 1 - e^{-x}$, we have
\begin{equation}
    \begin{aligned}
    \text{Var}(X) & \le V_n + \sqrt{2 \left(\sum_{i=1}^n \mathbb{E} \left[ \left(X_i - \mathbb{E}[X_i] \right)^4 \right] \right) x} \\
    & \le V_n + \sqrt{2 \left(\sum_{i=1}^n \mathbb{E} \left[ \left(X_i - \mathbb{E}[X_i] \right)^2 \right] \right) x} \\
    & = V_n + \sqrt{2 \text{Var}(X) x}.
    \end{aligned}
    \label{eq-concen-varest}
\end{equation}
Since $X_i\in [0,1]$ almost surely for all $i\in [n]$, we have
\begin{equation*}
    \text{Var}(X_i) = \mathbb{E}[X_i^2] - \mathbb{E}[X_i]^2 \le \mathbb{E}[X_i] - \mathbb{E}[X_i]^2 \le \frac{1}{4}.
\end{equation*}
Now, observe that
\begin{equation*}
    \text{Var}(X) = \sum_{i=1} \text{Var}(X_i) \le \sum_{i=1}^n \frac{1}{4} = \frac{n}{4} \quad \Rightarrow \sqrt{\text{Var}(X)} \le \frac{\sqrt{n}}{2}.
\end{equation*}
If $2 \sqrt{x} \ge \frac{\sqrt{n}}{2}$, then the Lemma evidently holds. Otherwise, we assume $2 \sqrt{x} \le \frac{\sqrt{n}}{2}$, which is equivalent to $x \le \frac{n}{16}$. Combining Lemma \ref{lem-concen-absvar} and (\ref{eq-concen-varest}), with probability $\ge 1 - 3 e^{-x}$, we have,
\begin{equation*}
    \begin{aligned}
    \text{Var}(X) & \le V_n + \sqrt{2 \text{Var}(X)x} + \frac{\left(X - \mathbb{E}[X] \right)^2}{n} \\
    & \le V_n + \sqrt{2 \text{Var}(X)x} + \frac{1}{n} \left(2 \text{Var}(X)x + \frac{4}{3}x \sqrt{2 \text{Var}(X)x} + \frac{4x^2}{9} \right) \\
    & \le  V_n + \sqrt{2 \text{Var}(X)x} + \frac{1}{n} \left(2 \sqrt{\text{Var}(X)x} \cdot \frac{\sqrt{n}}{2} \frac{\sqrt{n}}{4} + \frac{4}{3}\sqrt{2 \text{Var}(X)x} \cdot \frac{n}{16}+ \frac{4x}{9}\cdot \frac{n}{16} \right) \\
    & = \sqrt{\text{Var}(X)x} \left(\frac{13}{12}\sqrt{2} + \frac{1}{4} \right) + \left(V_n + \frac{x}{36} \right).
    \end{aligned}
\end{equation*}
Consequently, we can derive an upper bound for $\sqrt{\text{Var}(X)}$:
\begin{equation*}
    \sqrt{\text{Var}(X)} \le \frac{\sqrt{x}}{2} \left(\frac{13}{12}\sqrt{2} + \frac{1}{4} \right) + \frac{1}{2}\sqrt{x\left(\frac{13}{12}\sqrt{2} + \frac{1}{4} \right)^2 + 4\left(V_n + \frac{x}{36}\right)} \le \sqrt{V_n} + 2 \sqrt{x},
\end{equation*}
which proves the Lemma.
\end{proof}

\begin{lemma}
Suppose $X_i$ are independent random variables supported in $[0,1]$. Let $X = \sum_{i=1}^n X_i$,  then with probability $\ge 1 - 3 e^{-x} $, we have
\begin{equation*}
    \left |X - \mathbb{E}[X] \right | \le \sqrt{2 X x} + 4 x.
\end{equation*}
\label{lem-concen-abssum}
\end{lemma}

\begin{proof}[Proof of Lemma \ref{lem-concen-abssum}]
Apply Lemma \ref{lem-concen-absvar} and Lemma \ref{lem-concen-varest}, we directly derive that with probability $\ge 1 - 3 e^{-x} $
\begin{equation*}
    \left |X - \mathbb{E}[X] \right | \le \sqrt{2 \text{Var}(X) x} + \frac{2x}{3} \le \sqrt{2 V_n x} + \left(2 \sqrt{2} + \frac{2}{3} \right) x < \sqrt{2 V_n x} + 4 x \le \sqrt{2 X x} + 4 x,
\end{equation*}
where the last inequality comes from the fact that for random variable whose support is $[0,1]$, then its variance is always smaller than its mean.
\end{proof}

\begin{lemma}[Theorem 1.1 in \cite{dubhashi2009concentration}]
Suppose $X_i$ are independent random variables supported in $[0,1]$. Let $X = \sum_{i=1}^n X_i$, then for any $R > 2 e \mathbb{E}[X]$, we have
\begin{equation*}
    \mathbb{P}(X > R) \le 2^{-R}.
\end{equation*}
\end{lemma}

Now we turn back to the proof of Lemma \ref{lem-radius}. Denote $\delta = e^{-x}$. Apply Lemma \ref{lem-concen-abssum} then with probability $\ge 1 - 3 \delta$, we have,
\begin{equation*}
    N \left |\hat{V} - \mathbb{E}\left[\hat{V} \right] \right | \le \sqrt{2 N \hat{V} \log \left(\frac{1}{\delta}\right)} + 4 \log \left(\frac{1}{\delta}\right),
\end{equation*}
which is equivalent to , 
\begin{equation}
    \left |\hat{V} - \mathbb{E}\left[\hat{V} \right] \right | \le \text{rad}\left(\hat{V},N,\delta \right).
    \label{eq-Radpr1}
\end{equation}
Besides,
\begin{equation}
    \begin{aligned}
     \mathbb{P}\left(\text{rad}\left(\hat{V},N,\delta \right) > 3 \text{rad}\left(\mathbb{E}\left[\hat{V} \right],N,\delta \right) \right) &\le \mathbb{P} \left(\hat{V} > 9 \mathbb{E}\left[\hat{V} \right] + 32 \log \left(\frac{1}{\delta} \right) \right) \\
    & \le 2^{-9 \mathbb{E}\left[\hat{V} \right] - 32 \log \left(\frac{1}{\delta} \right)} \\
    & \le \delta.
    \end{aligned}
    \label{eq-Radpr2}
\end{equation}
Therefore, combining (\ref{eq-Radpr1}) and (\ref{eq-Radpr2}), the lemma holds.

\subsection{Proof for Lemma \ref{lem-confidencebound}}\label{sec:pf_lem-confidencebound}
By Lemma \ref{lem-radius}, with probability $\ge 1 - 3 KT \delta$, we have
\begin{equation*}
    |r(a) - \hat{R}_t(a)| \le \text{rad}(\hat{R}_{t}(a), N_{t-1}^+(a),\delta).
\end{equation*}
Hence with probability $\ge 1 - 3 KT \delta$,
\begin{equation*}
    \begin{cases}
    r(a) \le \hat{R}_t(a) + \text{rad}(\hat{R}_{t}(a), N_{t-1}^+(a),\delta) \\
    r(a) \le 1
    \end{cases}
    \qquad \Rightarrow \qquad 
    \begin{aligned}
    r(a) &\le  \min\left \{\hat{R}_{t}(a) + \text{rad}(\hat{R}_{t}(a), N_{t-1}^+(a),\delta),1 \right \} \\
    & = \text{UCB}_{r,t}(a).
    \end{aligned}
\end{equation*}
Similarly, with probability $\ge 1 - 3KTd \delta$, 
\begin{equation*}
    \textbf{LCB}_{c,t}(a) \preceq \boldsymbol{c}(a).
\end{equation*}

\subsection{Proof for Inequalities (\ref{eq-spf-upper-ucb}, \ref{eq-spf-upper-lcb})}\label{sec:pf_upper_lower}
We first provide the two lemmas:
\begin{lemma}[Theorem 1.6 in \cite{freedman1975tail}] Suppose $\{U_i\}_{i=1}^n$ is a martingale difference sequence supported in $[0,1]$ with respect to the filtration $\{{\cal F}_i\}^n_{i=1}$. Let $U = \sum_{i=1}^n U_i$, and $V = \sum_{i=1}^n \text{Var}(U_i|{\cal F}_{i-1})$. Then for any $a>0$, $b > 0$, we have
\begin{equation*}
    \mathbb{P} \left(|U| \ge a, V \le b \right) \le 2 e^{-\frac{a^2}{2(a+b)}}.
\end{equation*}
\label{lem-mar-tail}
\end{lemma}

\begin{lemma} Suppose $\{X_i\}_{i=1}^n$ are random variables supported in $[0,1]$, where $X_i$ is ${\cal F}_i$-measurable and $\{{\cal F}_i\}^n_{i=1}$ is a filtration. Let $M_i = \mathbb{E}[X_i|{\cal F}_{i-1}]$ for each $i\in \{1, \ldots, n\}$, and $M = \sum_{i=1}^n M_i$. Then with probability $\ge 1 - 2 n \delta$, we have
\begin{equation*}
    \left | \sum_{i=1}^n (X_i - M_i)\right | \le O \left(\sqrt{M \log \left(\frac{1}{\delta} \right)} + \log \left(\frac{1}{\delta} \right) \right ).
\end{equation*}
\label{lem-mar-sum-tail}
\end{lemma}

\begin{proof}[Proof of Lemma \ref{lem-mar-sum-tail}]
The proof follows the line of Theorem 4.10 in \cite{babaioff2015dynamic}. 
Let $U_i = X_i - M_i$ for each $i\in \{1, \ldots, n\}$. Clearly, $\{U_i\}_{i=1}^n$ is a martingale difference sequence with respect to the filtration $\{{\cal F}_i\}^n_{i=1}$. Since
\begin{equation*}
    \text{Var}(U_i|\mathcal{F}_{i-1}) = \text{Var}(X_i|\mathcal{F}_{i-1}) = \mathbb{E}[X_i^2 | \mathcal{F}_{i-1}] - \mathbb{E}[X_i | \mathcal{F}_{i-1}]^2 \le \mathbb{E}[X_i | \mathcal{F}_{i-1}] = M_i \text{ almost surely},
\end{equation*}
we have $V = \sum_{i=1}^n \text{Var}(U_i|\mathcal{F}_{i-1}) \leq \sum_{i=1}^n M_i = M$ almost surely. Apply Lemma \ref{lem-mar-tail} with $a = \sqrt{2 b \log \left(\frac{1}{\delta} \right)} + 2 \log \left(\frac{1}{\delta} \right)$ for any $b \ge 1$, it follows that with probability $\le 2 \delta$,


\begin{equation*}
    |U| = \left |\sum_{i=1}^n U_i \right | \ge O \left(\sqrt{b \log \left(\frac{1}{\delta} \right)} + \log \left(\frac{1}{\delta} \right) \right ) \quad \& \quad V \le b,
\end{equation*}
Take the union bound over all integer $b$ from $1$ to $n$, noting that $V \le M$ and $b-1\leq M\leq b$ for some $b\in \{1, \ldots, n\}$ almost surely, with probability $\ge 1 - 2 n \delta$ we have
\begin{equation*}
    \left | \sum_{i=1}^n (X_i - M_i)\right | \le O \left(\sqrt{M \log \left(\frac{1}{\delta} \right)} + \log \left(\frac{1}{\delta} \right) \right ).
\end{equation*}
Altogether, the lemma holds.

\end{proof}
Now, we paraphrase inequalities (\ref{eq-spf-upper-ucb}, \ref{eq-spf-upper-lcb}) as Lemmas \ref{lem-ucbr}, \ref{lem-lcbc}, and provide their proofs. 
\begin{lemma}
With probability $\ge 1 - 3 KT \delta $, we have
\begin{equation*}
    \left | \sum_{t=1}^{\tau - 1} q_t \text{UCB}_{r,t}(A_t) - \sum_{t=1}^{\tau - 1} q_t R_t \right | \le O  \left( \log\left(\frac{1}{\delta}\right) \left ( \sqrt{\overline{q} K \sum_{t=1}^{\tau - 1} q_t R_t} +  \overline{q} K \log\left (\frac{T}{K} \right ) \right) \right).
\end{equation*}
\label{lem-ucbr}
\end{lemma}

\begin{lemma}
With probability $\ge 1 - 3 KT d \delta $, we have
\begin{equation*}
    \left | \sum_{t=1}^{\tau - 1} q_t \text{LCB}_{c,t}(A_t,i) - \sum_{t=1}^{\tau - 1} q_t C_{t,i} \right | \le O  \left( \log\left(\frac{1}{\delta}\right) \left ( \sqrt{\overline{q}K B} +  \overline{q} K \log\left (\frac{T}{K} \right ) \right) \right),\quad \forall i \in [d].
\end{equation*}
\label{lem-lcbc}
\end{lemma}

\begin{proof}[Proof of Lemma \ref{lem-ucbr}]
First with probability $\ge 1 - 2 T \delta$, we have
\begin{subequations}
    \begin{align}
        \left| \sum_{t=1}^{\tau - 1} q_t r(A_t) - \sum_{t=1}^{\tau - 1} q_t R_t \right | & = \overline{q} \left| \sum_{t=1}^{\tau - 1} \frac{q_t}{\overline{q}} (r(A_t) - R_t) \right | \label{eq-pf-upper-ucbr-1-a} \\
        & \le O \left(\sqrt{\overline{q} \log\left(\frac{1}{\delta}\right) \sum_{t=1}^{\tau-1} q_t r(A_t)} + \overline{q} \log\left(\frac{1}{\delta}\right) \right) \label{eq-pf-upper-ucbr-1-b} \\
        & \le O \left(\sqrt{\overline{q} \log\left(\frac{1}{\delta}\right) \sum_{t=1}^{\tau-1} q_t \text{UCB}_{r,t}(A_t)} + \overline{q} \log\left(\frac{1}{\delta}\right) \right), \label{eq-pf-upper-ucbr-1-c}
    \end{align}
    \label{eq-pf-upper-ucbr-1}
\end{subequations}
where (\ref{eq-pf-upper-ucbr-1-c}) comes from Lemma \ref{lem-confidencebound}. Inequality (\ref{eq-pf-upper-ucbr-1-b}) comes from Lemma \ref{lem-mar-sum-tail}, where we apply $X_t = \frac{q_t R_t}{\overline{q}}$ and ${\cal F}_{t-1} = \sigma(\{A_s, q_s, R_s, \{C_{s, i}\}^d_{i=1}, \hat{Q}_s\}^{t-1}_{s=1} \cup \{q_t\})$. Then with probability $\ge 1 - 3 KT \delta $, we also have

\begin{subequations}
    \begin{align}
    \left |\sum_{t=1}^{\tau - 1} q_t \text{UCB}_{r,t}(A_t) - \sum_{t=1}^{\tau - 1} q_t r(A_t) \right | & \le 6 \sum_{t=1}^{\tau - 1} q_t \text{rad}(r(A_t),N_{t-1}^+(A_t), \delta) \label{eq-pf-upper-ucbr-2-a}\\
    & \le 6 \sum_{a \in \mathcal{A}: N_{\tau-1}(a) > 0} \sum_{n=1}^{N_{\tau-1}(a)} q_n(a) \text{rad}\left(r(a),n,\delta\right) \label{eq-pf-upper-ucbr-2-b}\\
    & = 6 \overline{q} \sum_{a \in \mathcal{A}: N_{\tau-1}(a) > 0} \sum_{n=1}^{N_{\tau-1}(a)} \frac{q_n(a)}{\overline{q}} \left(\sqrt{\frac{2 r(a) \log\left(\frac{1}{\delta}\right)}{n}} + \frac{4}{n} \log\left(\frac{1}{\delta}\right) \right) \label{eq-pf-upper-ucbr-2-c}\\
    & \le 6 \overline{q} \sum_{a \in \mathcal{A}: N_{\tau-1}(a)>0} \left (2 \sqrt{2 r(a) \frac{Q_{\tau-1}(a)}{\overline{q}} \log\left(\frac{1}{\delta}\right)} \right . \nonumber\\
    & \left . + 4 \left( 1 +   \log(N_{\tau-1}(a)) \right )\log\left(\frac{1}{\delta}\right) \right) \label{eq-pf-upper-ucbr-2-d} \\
    & \le 12  \left(\sqrt{2 \overline{q} K \log\left(\frac{1}{\delta}\right) \sum_{a \in \mathcal{A}} r(a) Q_{\tau-1}(a)} \right . \nonumber\\
    & \left . + 2 \overline{q} K \log\left(\frac{T}{K} \right) \log\left(\frac{1}{\delta}\right) + 2 \overline{q} K \log\left(\frac{1}{\delta}\right)   \right) \label{eq-pf-upper-ucbr-2-e}\\
    & = 12  \left(\sqrt{2 \overline{q} K \log\left(\frac{1}{\delta}\right) \sum_{t=1}^{\tau-1} q_t r(A_t)} \right . \nonumber \\
    & \left .+ 2 \overline{q} K \log\left(\frac{T}{K} \right) \log\left(\frac{1}{\delta}\right) + 2 \overline{q} K \log\left(\frac{1}{\delta}\right) \right)
    \label{eq-pf-upper-ucbr-2-f} \\
    & \le 12  \left(\sqrt{2 \overline{q} K \log\left(\frac{1}{\delta}\right) \sum_{t=1}^{\tau-1} q_t \text{UCB}_{r,t}(A_t)} \right . \nonumber\\ 
    & \left . + 2 \overline{q} K \log\left(\frac{T}{K} \right) \log\left(\frac{1}{\delta}\right) + 2 \overline{q} K \log\left(\frac{1}{\delta}\right) \right), 
    \label{eq-pf-upper-ucbr-2-g}
    \end{align} 
    \label{eq-pf-upper-ucbr-2}
\end{subequations}
where
\begin{itemize}
    \item (\ref{eq-pf-upper-ucbr-2-a}) comes from the following, with probability $\ge 1 - 3 KT \delta$, 
    \begin{equation*}
        \begin{aligned}
        \left |\text{UCB}_{r,t}(A_t) - r(A_t) \right | & \le \left |\hat{R}_{t-1}(A_t) - r(A_t) \right | + \text{rad}(\hat{R}_{t-1}(A_t), N_{t-1}^+(A_t),\delta) \\
        & \le 2 \text{rad}(\hat{R}_{t-1}(A_t), N_{t-1}(A_t),\delta) \\
        & \le 6 \text{rad}(r(A_t), N_{t-1}(A_t),\delta).
        \end{aligned}
    \end{equation*}
    \item (\ref{eq-pf-upper-ucbr-2-b}) comes from rearranging the sum. $q_n(a)$ means the $n$-th adversarial term that the algorithm selects $a$.
    \item (\ref{eq-pf-upper-ucbr-2-c}) comes from the definition of $\text{rad}(\cdot,\cdot,\cdot)$.
    \item (\ref{eq-pf-upper-ucbr-2-d}) comes from the following
    \begin{equation*}
    \sum_{i=1}^n \frac{w_i}{\sqrt{i}} = \sum_{i=1}^n \frac{2 w_i}{2 \sqrt{i}} \le \sum_{i=1}^n \frac{2 w_i}{\sqrt{\sum_{j=1}^i w_j} + \sqrt{\sum_{j=1}^{i-1} w_j}} = \sum_{i=1}^n 2\left( \sqrt{\sum_{j=1}^i w_j} - \sqrt{\sum_{j=1}^{i-1} w_j} \right) = 2 \sqrt{\sum_{i=1}^n w_i},
    \end{equation*}
    and
    \begin{equation*}
        \sum_{i=1}^n \frac{w_i}{i} \le  \sum_{i=1}^n \frac{1}{i} \le  (1 + \log(n)).
    \end{equation*}
    where $w_i \in (0,1]$.
    \item In (\ref{eq-pf-upper-ucbr-2-d}) and (\ref{eq-pf-upper-ucbr-2-e}) $Q_t(a) = \sum_{s\in[t],A_s=a} q_s$.
    \item (\ref{eq-pf-upper-ucbr-2-e}) comes from Jansen inequality.
\end{itemize}
Combine (\ref{eq-pf-upper-ucbr-1}) and (\ref{eq-pf-upper-ucbr-2}), we have
\begin{equation*}
    \begin{aligned}
    \sum_{t=1}^{\tau - 1} q_t \text{UCB}_{r,t}(A_t) &\le \sum_{t=1}^{\tau - 1} q_t r_t + O \left(\sqrt{ \overline{q} K \log\left(\frac{1}{\delta}\right) \sum_{t=1}^{\tau-1} q_t \text{UCB}_{r,t}(A_t)} \right .\\ 
    & \left .+  \overline{q} K \log\left(\frac{T}{K} \right) \log\left(\frac{1}{\delta}\right) +  \overline{q} K \log\left(\frac{1}{\delta}\right) \right),
    \end{aligned}
\end{equation*}
which is equivalent to 
\begin{equation*}
    \left(\sqrt{\sum_{t=1}^{\tau - 1} q_t \text{UCB}_{r,t}(A_t)} - O \left(\sqrt{\overline{q} K \log\left(\frac{1}{\delta}\right)} \right) \right)^2 \le \sum_{t=1}^{\tau - 1} q_t r_t + O \left(\overline{q} K \log\left(\frac{T}{K} \right) \log\left(\frac{1}{\delta}\right) +  \overline{q} K \log\left(\frac{1}{\delta}\right) \right),
\end{equation*}
Hence,
\begin{equation}
    \begin{aligned}
        \sqrt{\sum_{t=1}^{\tau - 1} q_t \text{UCB}_{r,t}(A_t)} & \le O \left(\sqrt{\overline{q} K \log\left(\frac{1}{\delta}\right)} \right) \\ 
        & + \sqrt{\sum_{t=1}^{\tau - 1} q_t r_t + O \left(\overline{q} K \log\left(\frac{T}{K} \right) \log\left(\frac{1}{\delta}\right) +  \overline{q} K \log\left(\frac{1}{\delta}\right) \right)} \\
        & \le \sqrt{\sum_{t=1}^{\tau - 1} q_t r_t} + O \left(\sqrt{\overline{q} K \log\left(\frac{T}{K} \right) \log\left(\frac{1}{\delta}\right)} + \sqrt{\overline{q} K \log\left(\frac{1}{\delta}\right)} \right).
    \end{aligned}
    \label{eq-pf-upper-ucbr-3}
\end{equation}
Combine (\ref{eq-pf-upper-ucbr-1}) and (\ref{eq-pf-upper-ucbr-2}), (\ref{eq-pf-upper-ucbr-3}), we finish the proof.
\end{proof}

\begin{proof}[Proof of Lemma \ref{lem-lcbc}]
The proof is quite similar to Lemma \ref{lem-ucbr}, so we omit the descriptive details. Similarly, with probability $\ge 1 - 2 T d \delta$, we have
\begin{equation}
    \begin{aligned}
        \left| \sum_{t=1}^{\tau - 1} q_t c(A_t,i) - \sum_{t=1}^{\tau - 1} q_t C_{t,i} \right | & = \overline{q} \left| \sum_{t=1}^{\tau - 1} \frac{q_t}{\overline{q}} (c(A_t) - C_{t,i}) \right |  \\
        & \le O \left(\sqrt{\overline{q} \log\left(\frac{1}{\delta}\right) \sum_{t=1}^{\tau-1} q_t c(A_t,i)} + \overline{q} \log\left(\frac{1}{\delta}\right) \right)  \\
        & \le O \left(\sqrt{\overline{q} \log\left(\frac{1}{\delta}\right) \sum_{t=1}^{\tau-1} q_t \text{UCB}_{c,t}(A_t,i)} + \overline{q} \log\left(\frac{1}{\delta}\right) \right), 
    \end{aligned}
    \label{eq-pf-upper-lcbc-1}
\end{equation}
where
\begin{equation*}
    \text{UCB}_{c,t}(a,i) = \min\left \{\hat{C}_{t}(a,i) + \text{rad}(\hat{C}_{t}(a,i), N_{t-1}^+(a), \delta),1 \right \}.
\end{equation*}
Then with probability $\ge 1 - 3 KT d \delta $, we also have

\begin{equation}
    \begin{aligned}
    \left |\sum_{t=1}^{\tau - 1} q_t \text{LCB}_{c,t}(A_t,i) - \sum_{t=1}^{\tau - 1} q_t c(A_t,i) \right | & \le 6 \sum_{t=1}^{\tau - 1} q_t \text{rad}(c(A_t,i),N_{t-1}^+(A_t), \delta) \\
    & \le 6 \sum_{a \in \mathcal{A}: N_{\tau-1}(a) > 0} \sum_{n=1}^{N_{\tau-1}(a)} q_n(a) \text{rad}\left(c(a,i),n,\delta\right) \\
    & = 6 \overline{q} \sum_{a \in \mathcal{A}: N_{\tau-1}(a) > 0} \sum_{n=1}^{N_{\tau-1}(a)} \frac{q_n(a)}{\overline{q}} \left(\sqrt{\frac{2 c(a,i) \log\left(\frac{1}{\delta}\right)}{n}} + \frac{4}{n} \log\left(\frac{1}{\delta}\right) \right) \\
    & \le 6 \overline{q} \sum_{a \in \mathcal{A}: N_{\tau-1}(a)>0} \left (2 \sqrt{2 c(a,i) \frac{Q_{\tau-1}(a)}{\overline{q}} \log\left(\frac{1}{\delta}\right)} \right .\\
    & \left .+ 4 \left( 1 +   \log(N_{\tau-1}(a)) \right )\log\left(\frac{1}{\delta}\right) \right) \\
    & \le 12  \left(\sqrt{2 \overline{q} K \log\left(\frac{1}{\delta}\right) \sum_{a \in \mathcal{A}} c(a,i) Q_{\tau-1}(a)} \right . \\ 
    & \left .+ 2 \overline{q} K \log\left(\frac{T}{K} \right) \log\left(\frac{1}{\delta}\right) + 2 \overline{q} K \log\left(\frac{1}{\delta}\right)   \right) \\
    & \le 12 \left(\sqrt{2 \overline{q} K \log\left(\frac{1}{\delta}\right) \sum_{t=1}^{\tau-1} q_t \text{UCB}_{c,t}(A_t,i) } \right . \\
    & \left .+ 2 \overline{q} K \log\left (\frac{T}{K} \right ) \log\left(\frac{1}{\delta}\right)  + 2 \overline{q} K \log\left(\frac{1}{\delta}\right)  \right). \\
    \end{aligned}
    \label{eq-pf-upper-lcbc-2}
\end{equation}
Similarly,
\begin{equation}
    \begin{aligned}
        \left |\sum_{t=1}^{\tau - 1} q_t \text{UCB}_{c,t}(A_t,i) - \sum_{t=1}^{\tau - 1} q_t c(A_t,i) \right | & \le 6 \sum_{t=1}^{\tau - 1} q_t \text{rad}(c(A_t,i),N_{t-1}^+(A_t), \delta) \\
        & \le O \left(\sqrt{ \overline{q} K \log\left(\frac{1}{\delta}\right) \sum_{t=1}^{\tau-1} q_t \text{UCB}_{c,t}(A_t,i) } \right. \\
        & \left . + \overline{q} K \log\left (\frac{T}{K} \right ) \log\left(\frac{1}{\delta}\right)  + \overline{q} K \log\left(\frac{1}{\delta}\right) \right).
    \end{aligned}
    \label{eq-pf-upper-lcbc-3}
\end{equation}
Combine (\ref{eq-pf-upper-lcbc-1}) and (\ref{eq-pf-upper-lcbc-3}), we have
\begin{equation*}
    \begin{aligned}
    \sum_{t=1}^{\tau - 1} q_t \text{UCB}_{c,t}(A_t,i) & \le \sum_{t=1}^{\tau - 1} q_t C_{t,i} + O \left(\sqrt{ \overline{q} K \log\left(\frac{1}{\delta}\right) \sum_{t=1}^{\tau-1} q_t \text{UCB}_{c,t}(A_t,i)} \right .\\
    &\left.  +  \overline{q} K \log\left(\frac{T}{K} \right) \log\left(\frac{1}{\delta}\right) +  \overline{q} K \log\left(\frac{1}{\delta}\right) \right),
    \end{aligned}
\end{equation*}
which is equivalent to 
\begin{equation*}
    \begin{aligned}
    \left(\sqrt{\sum_{t=1}^{\tau - 1} q_t \text{UCB}_{c,t}(A_t,i)} - O \left(\sqrt{\overline{q} K \log\left(\frac{1}{\delta}\right)} \right) \right)^2 & \le \sum_{t=1}^{\tau - 1} q_t C_{t,i} \\
    & + O \left(\overline{q} K \log\left(\frac{T}{K} \right) \log\left(\frac{1}{\delta}\right) +  \overline{q} K \log\left(\frac{1}{\delta}\right) \right),
    \end{aligned}
\end{equation*}
Hence,
\begin{equation}
    \begin{aligned}
        \sqrt{\sum_{t=1}^{\tau - 1} q_t \text{UCB}_{c,t}(A_t,i)} & \le O \left(\sqrt{\overline{q} K \log\left(\frac{1}{\delta}\right)} \right) \\
        & + \sqrt{\sum_{t=1}^{\tau - 1} q_t C_{t,i} + O \left(\overline{q} K \log\left(\frac{T}{K} \right) \log\left(\frac{1}{\delta}\right) +  \overline{q} K \log\left(\frac{1}{\delta}\right) \right)} \\
        & \le \sqrt{\sum_{t=1}^{\tau - 1} q_t C_{t,i}} + O \left(\sqrt{\overline{q} K \log\left(\frac{T}{K} \right) \log\left(\frac{1}{\delta}\right)} + \sqrt{\overline{q} K \log\left(\frac{1}{\delta}\right)} \right) \\
        & \le \sqrt{B} + O \left(\sqrt{\overline{q} K \log\left(\frac{T}{K} \right) \log\left(\frac{1}{\delta}\right)} + \sqrt{\overline{q} K \log\left(\frac{1}{\delta}\right)} \right),
    \end{aligned}
    \label{eq-pf-upper-lcbc-4}
\end{equation}
where the last inequality comes from the definition of the stopping time $\tau$. Combine (\ref{eq-pf-upper-lcbc-1}) and (\ref{eq-pf-upper-lcbc-2}), (\ref{eq-pf-upper-lcbc-4}), we finish the proof.
\end{proof}

\section{Proof for OCO Performance Guarantee}
\label{sec:pf_oco_pg}

This section contains the proof for the performance guarantee of the OCO tool AdaHedge we apply in the Algorithm \ref{alg-OAU}, which largely follows the line in \cite{orabona2019modern,de2014follow,orabona2015scale,orabona2018scale}. Here we still provide the proof for completeness. We first provide several lemmas, then we show the Inequality (\ref{eq-spf-ocoregret}) holds.

\begin{lemma}[\cite{karimi2016linear}]
    For a continuously differentiable and $\mu-$strongly convex function $f$ with respect to norm $\| \cdot \|$, suppose $\boldsymbol{x}^* \in \mathop{\arg\min}_{\boldsymbol{x}} f(\boldsymbol{x})$, then for all $\boldsymbol{x} \in \text{dom}f$, 
    the following Polyak-Lojasiewicz Inequality holds:
    \begin{equation*}
        f(\boldsymbol{x}) - f(\boldsymbol{x}^*) \le \frac{\|\nabla f(\boldsymbol{x}) \|_*^2}{2 \mu} 
    \end{equation*}
    \label{lem-stconvex-polyine}
\end{lemma}

\begin{proof}
    The $\mu-$strong convexity of $f$ implies $\forall \boldsymbol{x},\boldsymbol{y}$, we have
    \begin{equation*}
        f(\boldsymbol{y}) \ge f(\boldsymbol{x}) + \nabla f(\boldsymbol{x})^{\top} (\boldsymbol{y} - \boldsymbol{x}) + \frac{\mu}{2} \|\boldsymbol{y} - \boldsymbol{x}\|^2.
    \end{equation*}
    Take minimization respect to $\boldsymbol{y}$ on both sides, we obtain
    \begin{equation*}
        f(\boldsymbol{x}^*) \ge f(\boldsymbol{x}) - \frac{1}{2\mu} \|\nabla f(\boldsymbol{x}) \|_*^2
    \end{equation*}
    Rearranging it then we have the Polyak-Lojasiewicz Inequality holds.
\end{proof}

\begin{lemma}[A generalized version of Lemma 7 in \cite{orabona2018scale}] Let $\{a_t\}$ be any sequence of non-negative real numbers. Suppose $\{\delta_t\}$ is a sequence of non-negative real numbers satisfying
\begin{equation*}
    \begin{cases}
        \delta_0 = 0\\
        \delta_t \le \delta_{t-1} + \min \left\{ b a_t, c \frac{a_t^2}{\delta_{t-1}}\right\} & t \ge 1
    \end{cases}
\end{equation*}
then for any $T > 0$, 
\begin{equation*}
    \delta_T \le \sqrt{(b^2 + 2c)\sum_{t=1}^T a_t^2}
\end{equation*}
\label{lem-oco-recur}
\end{lemma}
\begin{proof}
    Observe that
    \begin{equation*}
       \delta_T^2 = \sum_{t=1}^T (\delta_t^2 - \delta_{t-1}^2) = \sum_{t=1}^T \left( (\delta_{t} - \delta_{t-1})^2 + 2 (\delta_{t} - \delta_{t-1})\delta_{t-1} \right).
    \end{equation*}
    From the definition of $\delta_{t}$, we have
    \begin{equation*}
        (\delta_{t} - \delta_{t-1})^2 \le b^2 a_t^2 \qquad \text{and} \qquad 2(\delta_{t} - \delta_{t-1})\delta_{t-1} \le 2c a_t^2
    \end{equation*}
    sum over $t$ from 1 to $T$, then we finish the proof.
\end{proof}

\begin{lemma} Suppose $f_t(\boldsymbol{x}) = \boldsymbol{g}_t^{\top} \boldsymbol{x}$ is a sequence of linear functions, $\boldsymbol{g}_t \in \mathbb{R}^d$, $\boldsymbol{x} \in \Delta_d$, then Algorithm \ref{alg-adahedge} applied on $\{f_t\}$ with $\kappa = \sqrt{\ln d}$ guarantees the following for all $T \ge 1$, $\boldsymbol{u} \in \Delta_{d}$:
\begin{equation*}
    \sum_{t=1}^T f_t(\boldsymbol{x}_t) - \sum_{t=1}^T f_t(\boldsymbol{u}) \le 2 \sqrt{(4 + \ln d) \sum_{t=1}^T \|\boldsymbol{g}_t\|_{\infty}^2}.
\end{equation*}
\label{lem-adahedge-regret}

\begin{algorithm}[htb]
	\caption{AdaHedge}
	\begin{algorithmic}[1]
	    \State Initialize $\boldsymbol{x}_1 = \frac{1}{d} \boldsymbol{1} = \left(\frac{1}{d},\cdots,\frac{1}{d}\right) \in \mathbb{R}^{d} $, $\eta_1 = 0$, $\boldsymbol{\theta}_1 = \boldsymbol{0} \in \mathbb{R}^{d}$, $\kappa > 0$..
        \For{$t=1,2,...,T$}
        \State Play $\boldsymbol{x}_t$ and observe cost $f_t(\boldsymbol{x}_t) = \boldsymbol{g}_t^{\top} \boldsymbol{x}_t$.
        \State Set
        \begin{align}
        \rho_t = \begin{cases}
                \boldsymbol{g}_1^{\top} \boldsymbol{x}_1  - \min_{j=1,...,d} g_{t,j} & t = 1\\ 
                \eta_t \ln \left(\sum_{j=1}^d x_{t,j} \exp\left(\frac{-g_{t,j}}{\eta_t} \right) \right) + \boldsymbol{g}_t^{\top} \boldsymbol{x}_t & \text{otherwise}.
            \end{cases} \nonumber
        \end{align}
        \State Update
        \begin{align}
            & \boldsymbol{\theta}_{t+1} = \boldsymbol{\theta}_{t} - \boldsymbol{g}_t, \nonumber \\
            & \eta_{t+1} = \eta_t + \frac{1}{\kappa^2} \rho_t, \nonumber \\
            & x_{t+1,j} = \frac{\exp \left(\frac{\theta_{t+1,j}}{\eta_{t+1}} \right)}{\sum_{i=1}^{d} \exp \left(\frac{\theta_{t+1,i}}{\eta_{t+1}} \right)}, \ j \in [d]. \nonumber
        \end{align}
        \EndFor
	\end{algorithmic}
	\label{alg-adahedge}
\end{algorithm}

\end{lemma}

\begin{proof} Observe that $\boldsymbol{\theta}_t = \sum_{i=1}^{t-1} \boldsymbol{g_i}$, $\eta_t = \frac{1}{\kappa^2} \sum_{i=1}^{t-1} \rho_i$, and for $t > 1$
\begin{equation*}
    \boldsymbol{x}_{t} = \left(\frac{\exp \left(\frac{\theta_{t,j}}{\eta_{t}} \right)}{\sum_{i=1}^{d} \exp \left(\frac{\theta_{t,i}}{\eta_{t}} \right)} \right)_{j \in [d]}=\mathop{\arg\min}_{\boldsymbol{x} \in \Delta_d} \ \eta_{t} \psi(\boldsymbol{x}) + \sum_{i=1}^{t-1} f_i(\boldsymbol{x}) 
\end{equation*}
where $\psi(\boldsymbol{x}) = \sum_{i=1}^d x_i \ln x_i + \ln d$. Denote $\psi_t(\boldsymbol{x}) = \eta_t \psi(\boldsymbol{x})$ for $t > 1$ and $\psi_1(\boldsymbol{x}) = \psi(\boldsymbol{x})$, $F_t(\boldsymbol{x})= \psi_t(\boldsymbol{x}) + \sum_{i=1}^{t-1} f_i(\boldsymbol{x})$. Then we can view the update rule as
\begin{equation*}
    \boldsymbol{x}_{t} = \mathop{\arg\min}_{\boldsymbol{x} \in \Delta_d} \ \psi_t(\boldsymbol{x}) + \sum_{i=1}^{t-1} f_i(\boldsymbol{x}) = \mathop{\arg\min}_{\boldsymbol{x} \in \Delta_d} \ F_t(\boldsymbol{x}),\ t \in [T]
\end{equation*}

We claim that for any $\boldsymbol{u} \in \Delta_d$, we have
\begin{equation}
    \sum_{t=1}^T f_t(\boldsymbol{x}_t) - \sum_{t=1}^T f_t(\boldsymbol{u}) \le \psi_{T+1}(\boldsymbol{u}) + \sum_{t=1}^T \left( F_t(\boldsymbol{x}_t) - F_{t+1}(\boldsymbol{x}_{t+1}) + f_t(\boldsymbol{x}_t) \right )
    \label{eq-pf-adahedge-claim}
\end{equation}

In fact, it is equivalent to verify that
\begin{equation*}
    - \sum_{t=1}^T f_t(\boldsymbol{u}) \le \psi_{T}(\boldsymbol{u}) + \sum_{t=1}^T \left( F_t(\boldsymbol{x}_t) - F_{t+1}(\boldsymbol{x}_{t+1}) \right )
\end{equation*}
which can be shown as following:
\begin{equation*}
    \begin{aligned}
    - \sum_{t=1}^T f_t(\boldsymbol{u}) &= \psi_{T+1}(\boldsymbol{u}) - F_{T+1}(\boldsymbol{u}) \\
    &= \psi_{T+1}(\boldsymbol{u}) -  F_{1}(\boldsymbol{x}_1) + F_{1}(\boldsymbol{x}_1) - F_{T+1}(\boldsymbol{x}_{T+1})+ F_{T+1}(\boldsymbol{x}_{T+1}) -  F_{T+1}(\boldsymbol{u}) \\
    &= \psi_{T+1}(\boldsymbol{u}) -  F_{1}(\boldsymbol{x}_1) + \sum_{t=1}^T \left( F_t(\boldsymbol{x}_t) - F_{t+1}(\boldsymbol{x}_{t+1}) \right ) + F_{T+1}(\boldsymbol{x}_{T+1}) -  F_{T+1}(\boldsymbol{u}) \\
    & \le \psi_{T+1}(\boldsymbol{u}) + \sum_{t=1}^T \left( F_t(\boldsymbol{x}_t) - F_{t+1}(\boldsymbol{x}_{t+1}) \right ) 
    \end{aligned}
\end{equation*}
where the first equality uses the definition of $F_{T+1}(\cdot)$, and the inequality comes from the fact that $F_1(\boldsymbol{x}_1) = \min_{\boldsymbol{x} \in \Delta_d} \psi(\boldsymbol{x}) = 0$ and $F_{T+1}(\boldsymbol{x}_{T+1}) = \min_{\boldsymbol{x} \in \Delta_d} F_{T+1}(\boldsymbol{x}) \le F_{T+1}(\boldsymbol{u})$ for all $\boldsymbol{u} \in \Delta_d$. Then the claim holds. Now we upper bound the term $ F_t(\boldsymbol{x}_t) - F_{t+1}(\boldsymbol{x}_{t+1}) + f_t(\boldsymbol{x}_t)$:
\begin{equation}
    \begin{aligned}
    F_t(\boldsymbol{x}_t) - F_{t+1}(\boldsymbol{x}_{t+1}) + f_t(\boldsymbol{x}_t) & = F_t(\boldsymbol{x}_t) - \psi_{t+1}(\boldsymbol{x}_{t+1}) - \sum_{i=1}^t f_i(\boldsymbol{x}_{t+1}) + f_t (\boldsymbol{x}_t) \\
    & = F_t(\boldsymbol{x}_t) - \eta_{t+1} \psi(\boldsymbol{x}_{t+1}) - \sum_{i=1}^t f_i(\boldsymbol{x}_{t+1}) + f_t (\boldsymbol{x}_t) \\
    & \le F_t(\boldsymbol{x}_t) - \eta_{t} \psi(\boldsymbol{x}_{t+1}) - \sum_{i=1}^t f_i(\boldsymbol{x}_{t+1}) + f_t (\boldsymbol{x}_t) \\
    & \le F_t(\boldsymbol{x}_t) - \min_{\boldsymbol{x \in \Delta_d}} \left \{ \eta_{t} \psi(\boldsymbol{x}) + \sum_{i=1}^t f_i(\boldsymbol{x}) \right \} + f_t (\boldsymbol{x}_t) \\
    & = \min_{\boldsymbol{x \in \Delta_d}} \left \{ \eta_{t} \psi(\boldsymbol{x}) + \sum_{i=1}^{t-1} f_i(\boldsymbol{x}) \right \} - \min_{\boldsymbol{x \in \Delta_d}} \left \{ \eta_{t} \psi(\boldsymbol{x}) + \sum_{i=1}^t f_i(\boldsymbol{x}) \right \} + f_t (\boldsymbol{x}_t) \\
    & = \begin{cases}
        \boldsymbol{g}_1^{\top} \boldsymbol{x}_1  - \min_{j=1,...,d} g_{t,j} & t = 1\\ 
        \eta_t \ln \left( \frac{\sum_{j=1}^d \exp\left(\frac{\theta_{t+1,j}}{\eta_t} \right)}{\sum_{j=1}^d \exp\left(\frac{\theta_{t,j}}{\eta_t} \right)} \right) + \boldsymbol{g}_t^{\top} \boldsymbol{x}_t & \text{otherwise}
    \end{cases} \\
    &= \begin{cases}
    \boldsymbol{g}_1^{\top} \boldsymbol{x}_1  - \min_{j=1,...,d} g_{t,j} & t = 1\\ 
    \eta_t \ln \left(\sum_{j=1}^d x_{t,j} \exp\left(\frac{-g_{t,j}}{\eta_t} \right) \right) + \boldsymbol{g}_t^{\top} \boldsymbol{x}_t & \text{otherwise} 
    \end{cases} \\
    &= \rho_t
    \end{aligned}
    \label{eq-pf-adahedge-stepbound}
\end{equation}
where the first inequality comes from the claim that $\rho_t \ge 0$ ($\rho_1 \ge 0$ is obvious, for $t > 1$, use the concavity of $\ln(\cdot)$), and thus $\eta_{t+1} \ge \eta_t$. Combine (\ref{eq-pf-adahedge-claim}), (\ref{eq-pf-adahedge-stepbound}), we have
\begin{equation*}
    \sum_{t=1}^T f_t(\boldsymbol{x}_t) - \sum_{t=1}^T f_t(\boldsymbol{u}) \le \psi_{T+1}(\boldsymbol{u}) + \sum_{t=1}^T \delta_T = \eta_{T+1} \left(\psi(\boldsymbol{u}) + \kappa^2 \right) 
\end{equation*}
Now we only need to upper bound $\eta_{T+1}$. Since $\eta_{t+1} = \eta_t + \frac{1}{\kappa^2} \rho_t = \frac{1}{\kappa^2}\sum_{i=1}^t \rho_i$, we bound $\rho_t$ first. For one hand, denote $\tilde{\boldsymbol{x}}_{t} = \mathop{\arg\min}_{\boldsymbol{x} \in \Delta_d} \left \{ \eta_{t} \psi(\boldsymbol{x}) + \sum_{i=1}^t f_i(\boldsymbol{x}) \right \} = \mathop{\arg\min}_{\boldsymbol{x} \in \Delta_d} \left \{ F_t(\boldsymbol{x}) + f_t (\boldsymbol{x}) \right \}$, then we have
\begin{equation}
    \begin{aligned}
    \rho_t & = F_t(\boldsymbol{x}_t) - \min_{\boldsymbol{x \in \Delta_d}} \left \{ \eta_{t} \psi(\boldsymbol{x}) + \sum_{i=1}^t f_i(\boldsymbol{x}) \right \} + f_t (\boldsymbol{x}_t) \\
    & = F_t(\boldsymbol{x}_t) - \eta_{t} \psi(\tilde{\boldsymbol{x}}_{t}) + \sum_{i=1}^t f_i(\tilde{\boldsymbol{x}}_{t})  + f_t (\boldsymbol{x}_t) \\
    & \le F_t(\tilde{\boldsymbol{x}}_{t}) - \eta_{t} \psi(\tilde{\boldsymbol{x}}_{t}) + \sum_{i=1}^t f_i(\tilde{\boldsymbol{x}}_{t})  + f_t (\boldsymbol{x}_t) \\
    & = - f_t(\tilde{\boldsymbol{x}}_{t}) +f_t(\boldsymbol{x}_t) \\
    & \le 2 \|\boldsymbol{g}_t\|_{\infty},
    \end{aligned}
    \label{eq-pf-adahedge-rhobound1}
\end{equation}
where the last inequality comes from Cauchy-Schwarz inequality. For another hand, 
\begin{equation}
    \begin{aligned}
        \rho_t & = F_t(\boldsymbol{x}_t) - \min_{\boldsymbol{x \in \Delta_d}} \left \{ \eta_{t} \psi(\boldsymbol{x}) + \sum_{i=1}^t f_i(\boldsymbol{x}) \right \} + f_t (\boldsymbol{x}_t) \\
        & = F_t(\boldsymbol{x}_t) - \eta_{t} \psi(\tilde{\boldsymbol{x}}_{t}) + \sum_{i=1}^t f_i(\tilde{\boldsymbol{x}}_{t})  + f_t (\boldsymbol{x}_t) \\ 
        & = F_t(\boldsymbol{x}_t) + f_t (\boldsymbol{x}_t) - \left(F_t(\tilde{\boldsymbol{x}}_{t}) + f_t (\tilde{\boldsymbol{x}}_{t}) \right) \\
        & \le \frac{\|\boldsymbol{g}_t\|_{\infty}^2}{2 \eta_t}
    \end{aligned}
    \label{eq-pf-adahedge-rhobound2}
\end{equation}
where the last inequality uses Lemma \ref{lem-stconvex-polyine} with $F_t + f_t$, which is $\eta_t$-strongly convex and the fact that $\nabla(F_t + f_t)(\boldsymbol{x}_t) = \nabla F_t(\boldsymbol{x}_t) + \nabla f_t(\boldsymbol{x}_t) = \boldsymbol{0} + \boldsymbol{g}_t = \boldsymbol{g}_t$. Combine (\ref{eq-pf-adahedge-rhobound1}),(\ref{eq-pf-adahedge-rhobound2}), we have
\begin{equation*}
    \rho_t \le \min \left \{2 \|\boldsymbol{g}_t\|_{\infty}, \frac{\|\boldsymbol{g}_t\|_{\infty}^2}{2 \eta_t} \right \},
\end{equation*}
and thus $\eta_t$ satisfies
\begin{equation*}
    \begin{cases}
        \eta_1 = 0 \\
        \eta_t \le \eta_{t-1} + \frac{1}{\kappa^2} \min \left \{2 \|\boldsymbol{g}_t\|_{\infty}, \frac{\|\boldsymbol{g}_t\|_{\infty}^2}{2 \eta_t} \right \} & t \ge 2
    \end{cases}
\end{equation*}
Apply Lemma \ref{lem-oco-recur}, we have
\begin{equation*}
    \eta_{T+1} \le \frac{1}{\kappa^2} \sqrt{(4 + \kappa^2) \sum_{t=1}^T \|\boldsymbol{g}_t\|_{\infty}^2 }
\end{equation*}

Finally,
\begin{equation*}
    \sum_{t=1}^T f_t(\boldsymbol{x}_t) - \sum_{t=1}^T f_t(\boldsymbol{u}) \le  \eta_{T+1} \left(\psi(\boldsymbol{u}) + \kappa^2 \right) \le \frac{\psi(\boldsymbol{u}) + \kappa^2}{\kappa^2} \sqrt{(4 + \kappa^2) \sum_{t=1}^T \|\boldsymbol{g}_t\|_{\infty}^2 } = 2 \sqrt{(4 + \ln d) \sum_{t=1}^T \|\boldsymbol{g}_t\|_{\infty}^2}.
\end{equation*}

The lemma is proved.

\end{proof}

Now we paraphrase inequality (\ref{eq-spf-ocoregret}) as the Lemma \ref{lem-ocoalg-regret}, and provide the proof.

\begin{lemma}
    Suppose $f_t(\boldsymbol{x}) = \frac{q_t \hat{Q}_t}{B}\left(\frac{B}{\hat{Q}_t} \boldsymbol{\beta} - \textbf{LCB}_{c,t}(A_t) \right)^{\top} \boldsymbol{x}$, then the OCO update in Algorithm \ref{alg-OAU} applied on $\{f_t\}$ guarantees the following for all $\boldsymbol{u} \in \Delta_{d}$:
\begin{equation*}
    \sum_{t=1}^T f_t(\boldsymbol{x}_t) - \sum_{t=1}^T f_t(\boldsymbol{u}) \le  O \left( \left(\overline{q} + \frac{\overline{q}^2}{b} \right) \sqrt{\ln(d+1) T} \right).
\end{equation*}
\label{lem-ocoalg-regret}
\end{lemma}

\begin{proof}
    Directly apply Lemma \ref{lem-adahedge-regret}, we have 
    \begin{equation*}
        \begin{aligned}
        \sum_{t=1}^T f_t(\boldsymbol{x}_t) - \sum_{t=1}^T f_t(\boldsymbol{u}) & \le 2 \sqrt{(4 + \ln (d+1)) \sum_{t=1}^T \left \|\frac{q_t \hat{Q}_t}{B}\left(\frac{B}{\hat{Q}_t} \boldsymbol{\beta} - \textbf{LCB}_{c,t}(A_t) \right) \right \|_{\infty}^2} \\
        & \le  O \left( \left(\overline{q} + \frac{\overline{q}^2}{b} \right) \sqrt{\ln(d+1) T} \right).
        \end{aligned}
    \end{equation*}
    where the last inequality comes from the triangle inequality of norm. Then the lemma holds.
\end{proof}

\section{Proofs for Regret Upper Bound for Linearly increasing and AR(1) demand models}

This section provides the proof for the regret upper bound for our algorithm under linearly increasing and AR(1) demand model. 

\subsection{Proof for Lemma \ref{lem-timeseries-accuracy-linear}}

Clearly $\overline{q} = \alpha + \beta T + M = \Theta \left( \beta T \right)$. Observe that for $t > 1$,
\begin{equation*}
   \hat{\beta}_t = \frac{(t-1)\sum_{s=1}^{t-1} s q_s - \left(\sum_{s=1}^{t-1} s \right) \left(\sum_{s=1}^{t-1} q_s \right)}{(t-1)\sum_{s=1}^{t-1} s^2 - \left(\sum_{s=1}^{t-1} s \right)^2} = \beta + \frac{12}{t(t-1)(t-2)} \sum_{s=1}^{t-1} \xi_s \left(s - \frac{t}{2} \right),
\end{equation*}
and similarly,  
\begin{equation*}
    \hat{\alpha}_t = \frac{\sum_{s=1}^{t-1} q_s - \hat{\beta}_t \sum_{s=1}^{t-1} s}{t-1} = \alpha + \frac{2}{(t-1)(t-2)} \sum_{s=1}^{t-1} \xi_s \left(2 t - 3 s -1 \right).
\end{equation*}

By the definition of $\hat{Q}_t$, we have
\begin{equation*}
    \begin{aligned}
        \left|\hat{Q}_t - Q \right| & = \left | \sum_{s=t}^T \left( \hat{\alpha}_t + \hat{\beta}_t s - \alpha - \beta s - \xi_s \right) \right| \\
        & \le (T - t + 1) \left|\hat{\alpha}_t - \alpha \right | + \frac{1}{2} (T - t + 1)(T+t) \left|\hat{\beta}_t - \beta \right | + \left | \sum_{s=t}^T \xi_s \right |.
    \end{aligned}
\end{equation*}

Hence
\begin{equation}
    \begin{aligned}
    \mathbb{P}\left(\left|\hat{Q}_t - Q \right| > \epsilon_t \right) & \le \mathbb{P} \left(\left|\hat{\alpha}_t - \alpha \right | > \frac{\lambda_1}{T-t+1} \epsilon_t  \right) + \mathbb{P} \left(\left|\hat{\beta}_t - \beta \right | > \frac{2\lambda_2}{(T-t+1)(T+t)} \epsilon_t  \right) \\
    & + \mathbb{P}\left(\left | \sum_{s=t}^T \xi_s \right | > \lambda_3 \epsilon_t \right),
    \end{aligned}
    \label{eq-pf-upper-linear-prob}
\end{equation}
where $\lambda_1 = \frac{(t-1)\sqrt{T-t+1}}{A}$, $\lambda_2 = \frac{(T+t)\sqrt{T-t+1}}{A}$, $\lambda_3 = \frac{(t-1)^{\frac{3}{2}}}{A}$, and $A$ is the value that makes $\lambda_1 + \lambda_2 + \lambda_3 = 1$. Now we apply Hoeffding's inequality for each term for $t \ge 4$:
\begin{equation}
    \begin{aligned}
        \mathbb{P} \left(\left|\hat{\alpha}_t - \alpha \right | > \frac{\lambda_1}{T-t+1} \epsilon_t  \right) & \le \mathbb{P} \left(\left| \sum_{s=1}^{t-1} \xi_s (2 t - 3 s - 1) \right | > \frac{(t-1)(t-2)}{2(T-t+1)} \lambda_1\epsilon_t  \right) \\
        & \le 2 \exp \left(- \frac{(t-1)(t-2)}{(T-t+1)^2 \left (t-\frac{1}{2} \right)} \cdot \frac{\lambda_1^2 \epsilon_t^2}{8 M^2}\right) \\
        & \le 2 \exp \left(- \frac{(t-1)^3}{T-t+1} \cdot \frac{ \epsilon_t^2}{16 A^2 M^2}\right),
    \end{aligned}
    \label{eq-pf-upper-linear-probalpha}
\end{equation}
\begin{equation}
    \begin{aligned}
        \mathbb{P} \left(\left|\hat{\beta}_t - \beta \right | > \frac{2\lambda_2}{(T-t+1)(T+t)} \epsilon_t  \right) & \le \mathbb{P} \left(\left| \sum_{s=1}^{t-1} \xi_s \left (s - \frac{t}{2} \right) \right | > \frac{t(t-1)(t-2)}{6(T-t+1)(T+t)} \lambda_2 \epsilon_t  \right) \\
        & \le 2 \exp \left(- \frac{t(t-1)(t-2)}{(T-t+1)^2 (T+t)^2} \cdot \frac{\lambda_2^2 \epsilon_t^2}{6 M^2}\right) \\
        & \le 2 \exp \left(- \frac{(t-1)^3}{T-t+1} \cdot \frac{ \epsilon_t^2}{16 A^2 M^2}\right),
    \end{aligned}
    \label{eq-pf-upper-linear-probbeta}
\end{equation}
\begin{equation}
    \begin{aligned}
        \mathbb{P}\left(\left | \sum_{s=t}^T \xi_s \right | > \lambda_3 \epsilon_t \right) \le 2 \exp \left( - \frac{\lambda_3^2 \epsilon_t^2}{2 M^2 (T-t+1)} \right) \le 2 \exp \left(- \frac{(t-1)^3}{T-t+1} \cdot \frac{ \epsilon_t^2}{16 A^2 M^2}\right).
    \end{aligned}
    \label{eq-pf-upper-linear-probvarep}
\end{equation}

Combine (\ref{eq-pf-upper-linear-prob}), (\ref{eq-pf-upper-linear-probalpha}), (\ref{eq-pf-upper-linear-probbeta}), (\ref{eq-pf-upper-linear-probvarep}), we have
\begin{equation*}
    \mathbb{P}\left(\left|\hat{Q}_t - Q \right| > \epsilon_t \right) \le 6 \exp \left(- \frac{(t-1)^3}{T-t+1} \cdot \frac{ \epsilon_t^2}{16 A^2 M^2}\right).
\end{equation*}
Thus, with probability $\ge 1 - \delta$, 
\begin{equation*}
    \left|\hat{Q}_t - Q \right| \le O \left(AM \sqrt{\log \left( \frac{1}{\delta}\right) \frac{(T-t+1)}{(t-1)^3}} \right) = O \left(M T^2 \sqrt{ \log \left( \frac{1}{\delta}\right)}  \left( t-1 \right)^{-\frac{3}{2}}\right) := \epsilon_t.
\end{equation*}



Take the union bound for all $t$ then we finish the proof.

\subsection{Proof for Theorem \ref{thm-upperbound-linear}}

Based on the expression of $\epsilon_t$ derived in Lemma \ref{lem-timeseries-accuracy-linear}, with probability $\ge 1 - T \delta$, we have
\begin{equation*}
    \begin{aligned}
    \sum_{t=1}^{\tau - 1} q_t \epsilon_t & \le 3 \overline{q} T + \sum_{t=4}^{\tau - 1} q_t \epsilon_t \\
    & \le 3 \overline{q} T +  \sum_{t=4}^{\tau - 1} (\alpha + \beta t + M) \epsilon_t  \\
    & = O \left(M T^2 \sqrt{\log \left( \frac{1}{\delta}\right)} \right) \cdot O \left( \sum_{t=4}^{\tau - 1} \beta t \cdot \left( t-1 \right)^{-\frac{3}{2}} \right) \\
    & = O\left(\overline{q} M T  \sqrt{(\tau - 1)\log \left( \frac{1}{\delta}\right)} \right).
    \end{aligned}
\end{equation*}

Since $Q = \sum_{t=1}^T q_t = \Theta(\overline{q} T)$, $B = b T$, we have
\begin{equation*}
    \left(\frac{1}{Q} + \frac{1}{B} \right) \sum_{t=1}^{\tau-1} q_t \epsilon_t = O \left( M\left( 1 + \frac{\overline{q}}{b}\right) \sqrt{(\tau-1) \log\left( \frac{1}{\delta}\right)} \right).
\end{equation*}

Hence, combine Theorem \ref{thm-upperbound-1}, with probability $\ge 1 - 3KTd \delta - T\delta $, we have
\begin{equation*}
    \begin{aligned}
        \text{OPT}_{\text{LP}} - \sum_{t=1}^{\tau - 1} q_t R_t  & \le   O \left( \left( \text{OPT}_{\text{LP}} \sqrt{\frac{\overline{q} K}{B}} + 
    \sqrt{\overline{q} K \text{OPT}_{\text{LP}} } \right) \log\left(\frac{1}{\delta} \right) \right. \\
    & +  \left. \left(\frac{1}{Q} + \frac{1}{B} \right) \sum_{t=1}^{\tau-1} q_t \epsilon_t + \left(\overline{q} + \frac{\overline{q}^2}{b} \right) \sqrt{(\tau - 1)\ln(d+1)} \right) \\
    & = O \left( \left( \text{OPT}_{\text{LP}} \sqrt{\frac{\overline{q} K}{B}} + 
    \sqrt{\overline{q} K \text{OPT}_{\text{LP}} } \right) \log\left(\frac{1}{\delta} \right) \right. \\
    & +  \left.  M\left( 1 + \frac{\overline{q}}{b}\right) \sqrt{(\tau-1) \log\left( \frac{1}{\delta}\right)} + \left(\overline{q} + \frac{\overline{q}^2}{b} \right) \sqrt{(\tau - 1)\ln(d+1)} \right) \\
    & = \Tilde{O} \left(\text{OPT}_{\text{LP}} \sqrt{\frac{\overline{q} K}{B}} + 
    \sqrt{\overline{q} K \text{OPT}_{\text{LP}} } + (M+\overline{q})\sqrt{\tau - 1} \right).
    \end{aligned}
\end{equation*}
Now the theorem is proved.

\subsection{Proof for Lemma \ref{lem-timeseries-accuracy-ar1}}

We prove this lemma by showing thw following 2 claims:
\begin{enumerate}
    \item With probability $\ge 1 - T \delta$, $\underline{q} \le q_t \le \overline{q}$, for all $t$, where 
    \begin{equation*}
    \overline{q} = \max \left \{q_1,\frac{\alpha}{1 - \beta} \right \} + \sigma \sqrt{\frac{2}{1 - \beta^2} \log \left(\frac{2}{\delta} \right)}, \quad \underline{q} = \min \left \{q_1,\frac{\alpha}{1 - \beta} \right \} - \sigma \sqrt{\frac{2}{1 - \beta^2} \log \left(\frac{2}{\delta} \right)}.
\end{equation*}
\textbf{Proof:} We recursively explore the expression of $q_t$:
\begin{equation*}
    \begin{aligned}
        q_t &= \alpha + \beta q_{t-1} + \xi_t \\
        & = \alpha + \beta \left(\alpha + \beta q_{t-2} + \xi_{t-1} \right) + \xi_t \\
        & = \alpha (1 + \beta) + \beta^2 q_{t-2} + \beta \xi_{t-1} + \xi_t \\
        & \ \  \vdots \\
        & = \alpha \left(1 + \beta + \cdots + \beta^{t-2} \right) + \beta^{t-1} q_1 + \beta^{t-2} \xi_2 + \cdots + \beta \xi_{t-1} + \xi_t \\
        & = \frac{\alpha}{1 - \beta} \left(1 - \beta^{t-1} \right) + \beta^{t-1} q_1 +  \beta^{t-2} \xi_2 + \cdots + \beta \xi_{t-1} + \xi_t. 
    \end{aligned}
\end{equation*}

Since $\{\xi_t\}$ is a sequence of independent random variables with zero-mean $\sigma^2-$subgaussian distribution, we can derive that 
\begin{equation*}
    \sum_{s=2}^t \beta^{t-s}  \xi_s \ \ \text{follows a zero-mean $A'\sigma^2$-subgaussian distribution,} 
\end{equation*}
where
\begin{equation*}
    A' = \sum_{s=2}^t \left(\beta^{t-s} \right)^2 = \frac{1 - \beta^{2(t-1)}}{1 - \beta^2}.
\end{equation*}
Hence with probability $1 - \delta$, 
\begin{equation*}
    \begin{aligned}
    q_t & \in \left[ \frac{\alpha}{1 - \beta} \left(1 - \beta^{t-1} \right) + \beta^{t-1} q_1 - \sigma \sqrt{\frac{2\left(1 - \beta^{2(t-1)}\right)}{1 - \beta^2} \log \left(\frac{2}{\delta} \right)}, \right. \\
    & \left .\frac{\alpha}{1 - \beta} \left(1 - \beta^{t-1} \right) + \beta^{t-1} q_1 + \sigma \sqrt{\frac{2\left(1 - \beta^{2(t-1)}\right)}{1 - \beta^2} \log \left(\frac{2}{\delta} \right)} \right] \\
    & \subseteq \left[\min \left \{q_1,\frac{\alpha}{1 - \beta} \right \} - \sigma \sqrt{\frac{2}{1 - \beta^2} \log \left(\frac{2}{\delta} \right)}, \max \left \{q_1,\frac{\alpha}{1 - \beta} \right \} + \sigma \sqrt{\frac{2}{1 - \beta^2} \log \left(\frac{2}{\delta} \right)} \right] \\
    & = \left[\underline{q},\overline{q} \right]
    \end{aligned}
\end{equation*}
Take the union bound for $t=1,2,...,T$, then with probability $\ge 1- T\delta$, $\underline{q} \le q_t \le \overline{q}$.
 
    \item Suppose $q_t \ge \underline{q}$ holds for all $t$, then with probability $\ge 1 - 2 T \delta$, 
    \begin{equation*}
    \left|\hat{Q}_t - Q \right| \le O \left(\frac{T-t+1}{\sqrt{t-1}} \cdot \frac{\alpha A}{\underline{q}(1-M)^2} + \sqrt{T-t+1} \cdot \frac{\sigma}{1 - \beta} \sqrt{\log \left(\frac{1}{\delta} \right)} \right) = \epsilon_t,
\end{equation*}
where $\underline{q}$, $\overline{q}$ are defined as above and 
\begin{equation*}
    A = O \left(\sqrt{\log \left(\frac{1}{\delta} \right)} + \sqrt{\log \left(T \right)} \right).
\end{equation*}
\textbf{Proof:} Apply Lemma 2 in \cite{bacchiocchi2022autoregressive} with $k = 1$ and $n = 1$, then with probability $\ge 1 - \delta$, we have
\begin{equation*}
    \begin{aligned}
    \left \| \hat{\boldsymbol{\gamma}}_t - \boldsymbol{\gamma} \right \|_{\boldsymbol{V}_t} \le \Delta_t &=  \sqrt{\lambda} \|\boldsymbol{\gamma} \|_2 + \sigma \sqrt{2 \log \left(\frac{1}{\delta} \right) + \log \left(\frac{\det(\boldsymbol{V}_t)}{\lambda^2} \right)} \\ 
    & \le \sqrt{\lambda} \|\boldsymbol{\gamma} \|_2 + \sigma \sqrt{2 \log \left(\frac{1}{\delta} \right) + 2 \log \left(\frac{2 \lambda + \overline{q}^2 (t-1)}{\lambda^2} \right)} \\
    & = O \left(\sqrt{\log \left(\frac{1}{\delta} \right)} + \sqrt{\log \left(T \right)} \right) := A.
    \end{aligned}
\end{equation*}
By the definition of $\boldsymbol{V}_t$, i.e.
\begin{equation*}
     \boldsymbol{V}_t = \lambda I_2 + \sum_{s=1}^{t-1} \boldsymbol{z}_{s-1} \boldsymbol{z}_{s-1}^{\top} = \begin{pmatrix}
     \lambda + t - 1 & \sum_{s=1}^{t-1} q_{s-1} \\
     \sum_{s=1}^{t-1} q_{s-1} & \lambda + \sum_{s=1}^{t-1} q_{s-1}^2
     \end{pmatrix},
\end{equation*}
when $t > 2$ and $\lambda = 1$, with probability $\ge 1 - \delta$,
\begin{equation*}
     (\lambda + t - 1) \left |\hat{\alpha}_t - \alpha \right|^2 \le  \left \| \hat{\boldsymbol{\gamma}}_t - \boldsymbol{\gamma} \right \|_{\boldsymbol{V}_t}^2 \le \Delta_t^2 \quad \Rightarrow \quad \left |\hat{\alpha}_t - \alpha \right| \le \frac{\Delta_t}{\sqrt{\lambda + t-1}} = O \left( \frac{\Delta_t}{\sqrt{t-1}}\right).
\end{equation*}
Similarly, with probability $\ge 1 - \delta$, we have
\begin{equation*}
    \left(\lambda + \sum_{s=1}^{t-1} q_{s-1}^2\right) \left |\hat{\beta}_t - \beta \right|^2 \le  \left \| \hat{\boldsymbol{\gamma}}_t - \boldsymbol{\gamma} \right \|_{\boldsymbol{V}_t}^2 \le \Delta_t^2 \quad \Rightarrow \quad \left |\hat{\beta}_t - \beta \right| \le \frac{\Delta_t}{\sqrt{\lambda + \sum_{s=1}^{t-1}q_s^2 }} = O \left( \frac{ \Delta_t}{\underline{q}\sqrt{t-1}}\right).
\end{equation*}

Denote $\phi = \frac{\alpha}{1 - \beta}$, then we obtain: with probability $\ge 1-\delta$,
\begin{equation*}
    \left |\hat{\phi}_t - \phi \right| = \left |\frac{\hat{\alpha}_t}{1 - \hat{\beta}_t}- \frac{\alpha}{1 - \beta} \right| \le \frac{1}{(1-M)^2} \left|\hat{\alpha}_t (1-\beta) - \alpha (1 - \hat{\beta}_t ) \right| = O \left(\frac{\alpha \Delta_t}{\underline{q}(1-M)^2 \sqrt{t-1}} \right).
\end{equation*}
Now we recursively explore the expression of $Q$:
\begin{equation*}
    \begin{aligned}
        Q &= \sum_{t=1}^T q_t \\
        & = \sum_{t=1}^{T-1} q_t + \alpha + \beta q_{T-1} + \xi_T \\
        & = \sum_{t=1}^{T-2} q_t + \phi(1 - \beta) + (\beta + 1) q_{T-1} + \xi_T \\
        & \ \ \vdots \\
        & = \sum_{s=1}^{t-1} q_s + \frac{\beta - \beta^{T-t+1}}{1 - \beta} q_{t-1} + \phi \left(T-t+1 - \beta + \beta^{T-t+2}\right) + \sum_{s=t}^T \frac{1 - \beta^{T-s+1}}{1 - \beta} \xi_s.
    \end{aligned}
\end{equation*}
By the definition of $\hat{Q}_t$ (\ref{eq-timeseries-ar1-prediction}), we can analyze the prediction error $\epsilon_t$ as following: with probability $1 - \delta$,
\begin{equation}
    \begin{aligned}
    \left|\hat{Q}_t - Q \right| & \le \overline{q} \left |  \frac{\hat{\beta}_t - \hat{\beta}_t^{T-t+1}}{1 - \hat{\beta}_t} - \frac{\beta - \beta^{T-t+1}}{1 - \beta} \right | + (T-t+1)\left|\hat{\phi}_t - \phi \right| \\ &+ \left |\hat{\phi}_t \left(\hat{\beta}_t - \hat{\beta}_t^{T-t+2} \right) - \phi \left(\beta - \beta^{T-t+2} \right) \right | + \left|\sum_{s=t}^T \frac{1 - \beta^{T-s+1}}{1 - \beta} \xi_s \right | \\
    & \le O \left( \frac{\overline{q}}{1 - M }\right) + O \left(\frac{\alpha \Delta_t (T-t+1)}{\underline{q}(1-M)^2 \sqrt{t-1}} \right) \\
    &+ \left |\hat{\phi}_t \left(\hat{\beta}_t - \hat{\beta}_t^{T-t+2} \right) - \phi \left(\beta - \beta^{T-t+2} \right) \right | + \left|\sum_{s=t}^T \frac{1 - \beta^{T-s+1}}{1 - \beta} \xi_s \right |.
    \end{aligned}
    \label{eq-pf-upper-ar1-errortotal}
\end{equation}
Now we study the last two terms respectively. First, with probability $\ge 1 - \delta$, we have 
\begin{equation} 
    \begin{aligned}
    \left |\hat{\phi}_t \left(\hat{\beta}_t - \hat{\beta}_t^{T-t+2} \right) - \phi \left(\beta - \beta^{T-t+2} \right) \right | & \le \left |\left( \hat{\phi}_t - \phi \right )\left(\hat{\beta}_t - \hat{\beta}_t^{T-t+2} \right) \right | \\
    & + \left | \phi \left(\hat{\beta}_t - \hat{\beta}_t^{T-t+2} - \beta + \beta^{T-t+2} \right) \right | \\
    & \le 2 \left | \hat{\phi}_t - \phi  \right | + 4 |\phi| \\
    & = O \left(\frac{\alpha \Delta_t}{\underline{q}(1-M)^2 \sqrt{t-1}} + \phi \right).
    \end{aligned}
    \label{eq-pf-upper-ar1-error1}
\end{equation}
Second, since $\{\xi_t\}$ is a sequence of independent random variables with zero-mean $\sigma^2-$subgaussian distribution, we can derive that 
\begin{equation*}
    \sum_{s=t}^T \frac{1 - \beta^{T-s+1}}{1 - \beta} \xi_s \quad \text{follows a zero-mean $A'\sigma^2$-subgaussian distribution,}
\end{equation*}
where
\begin{equation*}
    A' = \sum_{s=t}^T \left(\frac{1 - \beta^{T-s+1}}{1 - \beta} \right)^2 = \frac{T-t+1}{(1-\beta)^2} - 2 \frac{\beta-\beta^{T-t+2}}{(1-\beta)^3} + \frac{\beta^2 - \beta^{2(T-t+2)}}{(1-\beta)^2(1-\beta^2)} = O \left( \frac{T-t+1}{(1-\beta)^2} \right).
\end{equation*}
Hence with probability $\ge 1 - \delta$, 
\begin{equation}
    \left|\sum_{s=t}^T \frac{1 - \beta^{T-s+1}}{1 - \beta} \xi_s \right | \le O \left(\frac{\sqrt{T-t+1}}{1-\beta} \sigma \sqrt{\log \left(\frac{1}{\delta} \right)}\right).
    \label{eq-pf-upper-ar1-error2}
\end{equation}
Combine (\ref{eq-pf-upper-ar1-errortotal}), (\ref{eq-pf-upper-ar1-error1}), (\ref{eq-pf-upper-ar1-error2}) and take union bound on event 
\begin{equation*}
    \left \{\left|\sum_{s=t}^T \frac{1 - \beta^{T-s+1}}{1 - \beta} \xi_s \right | > \Omega \left(\frac{\sqrt{T-t+1}}{1-\beta} \sigma \sqrt{\log \left(\frac{1}{\delta} \right)}\right) \right\}
\end{equation*}
and $\left\{\left \| \hat{\boldsymbol{\gamma}}_t - \boldsymbol{\gamma} \right \|_{\boldsymbol{V}_t} > \Delta_t \right\}$, we obtain the explicit expression of $\epsilon_t$: with probability $\ge 1 - 2 \delta$
\begin{equation*}
    \begin{aligned}
        \left|\hat{Q}_t - Q \right| & \le O \left(\frac{\alpha \Delta_t (T-t+1)}{\underline{q}(1-M)^2 \sqrt{t-1}} \right) + O \left(\frac{\sqrt{T-t+1}}{1-\beta} \sigma \sqrt{\log \left(\frac{1}{\delta} \right)} \right) \\
        & = O \left(\frac{T-t+1}{\sqrt{t-1}} \cdot \frac{\alpha A}{\underline{q}(1-M)^2} + \sqrt{T-t+1} \cdot \frac{\sigma}{1 - \beta} \sqrt{\log \left(\frac{1}{\delta} \right)} \right) := \epsilon_t.
    \end{aligned}
\end{equation*}
Take the union bound for $t = 1,2,...,T$, the claim is proved.
\end{enumerate}

Altogether, the lemma is proved.

\subsection{Proof for Theorem \ref{thm-upperbound-ar1}}

Based on the expression of $\epsilon_t$ derived in Lemma \ref{lem-timeseries-accuracy-ar1}, with probability $\ge 1 - 3 T \delta$, we have
\begin{equation*}
    \begin{aligned}
    \sum_{t=1}^{\tau- 1} q_t \epsilon_t &= 2 O\left(\overline{q}T \right) +  \overline{q} \sum_{t=3}^{\tau - 1} \epsilon_t  \\
    & \le O \left(2 \overline{q} T + \sum_{t=3}^{\tau - 1}  \left( \frac{T-t+1}{\sqrt{t-1}} \cdot \frac{\alpha A}{\underline{q}(1-M)^2} + \sqrt{T-t+1} \cdot \frac{\sigma}{1 - \beta} \sqrt{\log \left(\frac{1}{\delta} \right)} \right)\right) \\
    & \le O \left(2 \overline{q} T + (T - \tau + 1) \sqrt{\tau - 1} \cdot \frac{\alpha A}{\underline{q}(1-M)^2} + (\tau - 1) \sqrt{T} \cdot \frac{\sigma}{1 - \beta}\sqrt{\log \left(\frac{1}{\delta} \right)}  \right) \\
    & = O \left(\Gamma' T \sqrt{\tau - 1} \right),
    \end{aligned}
\end{equation*}
where
\begin{itemize}
\item 
\begin{equation*}
    \Gamma' = \frac{\alpha A}{\underline{q}(1-M)^2} + \frac{\sigma}{1 - \beta} \sqrt{\log \left(\frac{1}{\delta} \right)} = \frac{\alpha}{\underline{q}(1-M)^2} \left( \sqrt{\log \left(\frac{1}{\delta} \right)} + \sqrt{\log \left(T \right)}  \right)+ \frac{\sigma}{1 - \beta} \sqrt{\log \left(\frac{1}{\delta} \right)}.
\end{equation*}
\item The second inequality comes from the following two facts:
\begin{equation*}
    \sum_{t=3}^{\tau - 1}  \frac{T-t+1}{\sqrt{t-1}} \le \int_2^{\tau - 1}  \frac{T-t+1}{\sqrt{t-1}} dt = \Theta \left( T\sqrt{\tau - 1} - (\tau-1)^{\frac{3}{2}}\right) = \Theta \left((T-\tau+1)\sqrt{\tau - 1} \right),
\end{equation*}
and
\begin{equation*}
    \sum_{t=3}^{\tau - 1} \sqrt{T-t+1} \le \int_2^{t-1} \sqrt{T-t+1} dt = \Theta\left(T^{\frac{3}{2}} - (T - \tau+1)^{\frac{3}{2}} \right) = \Theta \left((\tau - 1)\sqrt{T} \right).
\end{equation*}
\end{itemize}

Since $Q \ge \underline{q}T$, $B = bt$, utimately we have
\begin{equation*}
    \left(\frac{1}{Q} + \frac{1}{B} \right) \sum_{t=1}^{\tau-1} q_t \epsilon_t = O \left(\Gamma \sqrt{\tau - 1} \right),
\end{equation*}
where
\begin{equation*}
    \Gamma = \left(\frac{1}{\underline{q}} + \frac{1}{b} \right) \Gamma' = \left(\frac{1}{\underline{q}} + \frac{1}{b} \right) \left(\frac{\alpha}{\underline{q}(1-M)^2} \left( \sqrt{\log \left(\frac{1}{\delta} \right)} + \sqrt{\log \left(T \right)}  \right)+ \frac{\sigma}{1 - \beta} \sqrt{\log \left(\frac{1}{\delta} \right)} \right).
\end{equation*}

Same as the proof for Theorem \ref{thm-upperbound-linear}, we can obtain that with probability $\ge 1 - 3 KTd\delta - 3T \delta$, 
\begin{equation*}
        \text{OPT}_{\text{LP}} - \sum_{t=1}^{\tau - 1} q_t R_t \le \Tilde{O} \left( \text{OPT}_{\text{LP}} \sqrt{\frac{\overline{q} K}{B}} + 
    \sqrt{\overline{q} K \text{OPT}_{\text{LP}} } + \left(\Gamma + \overline{q} \right) \sqrt{\tau - 1} \right).
\end{equation*}



\end{document}